\documentclass[10pt]{article}

\usepackage[letterpaper,margin=1in]{geometry}
\usepackage[T1]{fontenc}
\usepackage[utf8]{inputenc}
\usepackage{lmodern}
\usepackage{microtype}
\usepackage[authoryear,round]{natbib}
\setcitestyle{authoryear,round,citesep={;},aysep={,},yysep={;}}

\usepackage{amsmath,amsfonts,amssymb,bm,amsthm}
\usepackage{mathtools}
\mathtoolsset{showonlyrefs=false}
\usepackage{graphicx}
\usepackage{booktabs}
\usepackage{enumerate}
\usepackage{tikz}
\usetikzlibrary{arrows.meta,positioning}
\usepackage{fix-cm}

\def\1{\bm{1}}

\DeclareMathAlphabet{\mathsfit}{\encodingdefault}{\sfdefault}{m}{sl}
\SetMathAlphabet{\mathsfit}{bold}{\encodingdefault}{\sfdefault}{bx}{n}

\newenvironment{noheadproof}[1][]{\noindent\textit{}}{}

\usetikzlibrary{arrows.meta,calc,positioning,backgrounds}

\usepackage{xcolor,float,placeins,array,needspace}
\usepackage[font=small]{caption}
\usepackage{hyperref}
\usepackage[nameinlink,capitalise,noabbrev]{cleveref}
\usepackage{url}
\newtheorem{theorem}{Theorem}
\newtheorem{lemma}{Lemma}[section]
\newtheorem{proposition}{Proposition}
\newtheorem{corollary}{Corollary}
\theoremstyle{definition}
\newtheorem{assumption}{Assumption}
\newtheorem{definition}{Definition}
\theoremstyle{remark}
\newtheorem{remark}{Remark}
\theoremstyle{plain}
\crefname{assumption}{Assumption}{Assumptions}
\Crefname{assumption}{Assumption}{Assumptions}
\AddToHook{env/theorem/begin}{\crefalias{theorem}{theorem}}
\AddToHook{env/lemma/begin}{\crefalias{lemma}{lemma}\crefalias{section}{lemma}}
\AddToHook{env/proposition/begin}{\crefalias{proposition}{proposition}}
\AddToHook{env/corollary/begin}{\crefalias{corollary}{corollary}}
\AddToHook{env/assumption/begin}{\crefalias{assumption}{assumption}}
\AddToHook{env/definition/begin}{\crefalias{definition}{definition}}
\AddToHook{env/remark/begin}{\crefalias{remark}{remark}}

\makeatletter
\renewcommand{\section}{\@startsection{section}{1}{\z@}%
  {-9pt plus -2pt minus -1pt}{5pt plus 1pt minus 1pt}%
  {\normalfont\large\bfseries}}
\renewcommand{\subsection}{\@startsection{subsection}{2}{\z@}%
  {-7pt plus -2pt minus -1pt}{4pt plus 1pt minus 1pt}%
  {\normalfont\normalsize\bfseries}}
\renewcommand{\subsubsection}{\@startsection{subsubsection}{3}{\z@}%
  {-6pt plus -2pt minus -1pt}{3pt plus 1pt minus 1pt}%
  {\normalfont\normalsize\bfseries}}
\renewcommand{\paragraph}{\@startsection{paragraph}{4}{\z@}%
  {5pt plus 1pt minus 1pt}{-0.7em}%
  {\normalfont\normalsize\bfseries}}
\makeatother
\AtBeginDocument{%
  \setlength{\abovedisplayskip}{6pt plus 2pt minus 2pt}%
  \setlength{\belowdisplayskip}{6pt plus 2pt minus 2pt}%
  \setlength{\abovedisplayshortskip}{0pt plus 2pt}%
  \setlength{\belowdisplayshortskip}{4pt plus 2pt minus 2pt}%
}

\newcommand{\PaperTitle}{Optimization Geometry of Equivalent Brownian RKHS Representations}
\newcommand{\PaperPDFAuthors}{Mahdi Mohammadigohari and Gustau Camps-Valls}
\hypersetup{
  colorlinks=true,
  linkcolor=blue!50!black,
  citecolor=green!35!black,
  urlcolor=blue!55!black,
  pdftitle={\PaperTitle},
  pdfauthor={\PaperPDFAuthors},
  pdfpagemode=UseNone
}
\title{Optimization Geometry of Equivalent\\Brownian RKHS Representations}
\author{%
\begin{minipage}[t]{0.48\textwidth}\centering\small
Mahdi Mohammadigohari\\
Faculty of Engineering\\
Free University of Bozen-Bolzano\\
Bruno Buozzi 1, Bolzano 39100, Italy\\
\texttt{mahdi.mohammadigohari@gmail.com}
\end{minipage}\hspace{0.02\textwidth}%
\begin{minipage}[t]{0.47\textwidth}\centering\small
Gustau Camps-Valls\\
Image Processing Laboratory (IPL)\\
Universitat de Val\`encia\\
Paterna, Val\`encia 46980, Spain\\
\texttt{gustau.camps@uv.es}
\end{minipage}}
\date{}
\renewenvironment{noheadproof}{\par\noindent\ignorespaces}{}

\begin{document}

\maketitle
\raggedbottom

\begin{abstract}
Equivalent finite parameterizations can represent the same functions and intrinsic norm yet induce
different optimization algorithms. We study this effect in a controlled finite Brownian RKHS with nodal,
increment, and spectral coordinates. Classical finite-element, RKHS-interpolation, Brownian-covariance,
and mixed-boundary DCT identities make the shared hypothesis class, Brownian energy, approximation
operator, and coordinate maps explicit. Our main results concern the optimization geometry of this fixed model. With mapped initialization,
identical scalar steps, and identical minibatches, nodal and spectral GD/SGD have exactly the same mapped
trajectories. Increment GD is an explicit Euler step for the constant Brownian/Sobolev metric, with factor
\(1/h\). For Brownian-regularized least squares,
\(\kappa_2(\mathbf H_{\mathrm{inc}})\le1+A/\rho\), independently of grid resolution \(G\) for fixed
\(A\), \(\rho>0\), and the stated normalization. Under the stated standard-Adam convention, the universal
orthogonal equivariance group is exactly the signed permutations; the block DCT-VIII transform is not one.
Float64 tests over five grids numerically verify the finite identities, mapped one-layer and recursive
trajectories, conditioning predictions, and theorem-matched Adam separation. Thus coordinate effects are
isolated without changing the represented functions, intrinsic regularizer, or approximation space.
\end{abstract}


\section{Introduction}
\label{sec:intro}

Reproducing kernel Hilbert spaces describe a model through the functions it represents and the intrinsic
norm controlling them \citep{aronszajn1950theory,scholkopf2002learning,berlinet2004rkhs,steinwart2008support}.
Splines, Gaussian processes, random features, and neural tangent kernels all exploit this viewpoint
\citep{wahba1990spline,rasmussen2006gaussian,rahimi2007random,jacot2018ntk,lee2019wide}. A complementary
line studies the geometry of optimization: natural gradient, mirror descent, implicit-bias analysis, and
reparameterization theory show that equivalent descriptions of one model can induce different algorithms
\citep{amari1998natural,beck2003mirror,gunasekar2018optimization,amid2020reparameterizing,
kristiadi2023geometry,li2022implicitbias,martens2020new}.
Coordinatewise adaptive methods add a basis dependence that ordinary Euclidean gradient descent does not
have under orthogonal changes \citep{kingma2015adam,ling2022vectoradam,zhang2025rotationadam}.

These viewpoints are difficult to separate because a neural reparameterization often also changes
initialization, implicit regularization, architecture, or the represented class. A controlled comparison
instead needs several coordinate systems for one fixed Hilbert space. Brownian profile layers provide this
setting: BKLs and VBKLs recursively use one-dimensional Brownian RKHS profiles
\citep{anonymous2026bkl,anonymous2026vbkl}, and the VBKL companion already employs anchored
piecewise-linear profiles and their increment energy.

The required ingredients are largely classical: piecewise-linear stiffness energies come from finite
elements \citep{strang1973analysis,brenner2008finite}; nodal interpolation and power functions from spline
and kernel approximation \citep{wahba1990spline,fasshauer2007meshfree,wendland2005scattered}; Brownian
covariance identities from Gaussian processes \citep{rasmussen2006gaussian}; and the mixed-boundary
spectrum from DCT theory \citep{martucci1994symmetric,strang1999dct,masera2017odd}. We assemble them in
one two-sided anchored normalization so that the represented space, norm, approximation error, and
coordinate maps are simultaneously fixed and explicit, rather than claiming these ingredients separately
as new.

In this controlled model, the spectral map is orthogonal whereas
\(\mathbf D_0^{\top}\mathbf D_0=h\mathbf K_0\). We derive the resulting mapped optimizer trajectories,
the constant-metric Brownian interpretation of increment descent, and its blockwise recursive consequence.
We also obtain the fixed-\(A,\rho\) grid-resolution-independent least-squares bound and the maximal universal
orthogonal equivariance group of the stated standard-Adam update.

Our contributions are:
\begin{enumerate}
\item one controlled finite Brownian RKHS with exactly equivalent nodal, increment, and block DCT-VIII
coordinates, preserving the hypothesis class, intrinsic norm, and approximation space;
\item exact mapped GD/SGD identities: nodal and spectral trajectories coincide, while increment descent is
a constant-metric Brownian/Sobolev Euler step with the precise mesh scaling;
\item the blockwise recursive consequence and
\(\kappa_2(\mathbf H_{\mathrm{inc}})\le1+A/\rho\), independent of \(G\) under the stated fixed
quantities; and
\item the signed-permutation maximality theorem for standard Adam, with source-traced float64 verification.
EuroSAT and Salinas remain secondary external-validity studies, not coordinate-equivalence tests.
\end{enumerate}
Proofs and supplementary material are in the appendix; \Cref{tab:main-results} summarizes the results.

\section{Notation and Preliminaries}
\label{sec:notation}
\label{sec:preliminaries}
{
We write $[n]=\{1,\ldots,n\}$, $\langle\cdot,\cdot\rangle$ and $\|\cdot\|_2$ for the Euclidean inner
product and norm, $\mathbf I_n$ for the identity, $\mathbf e_j$ for a canonical basis vector,
$\delta_{ij}$ for the Kronecker delta, and $\operatorname{blkdiag}$ for block-diagonal assembly. Vectors
are bold lowercase, matrices bold uppercase, scalars italic; grid quantities are indexed from $0$ and
spectral quantities from $1$. Three symbols are disambiguated relative to the underlying frameworks:
increment coordinates are $\mathbf w$, reserving $\delta$ for the Kronecker delta and the evaluation
functional $\delta_t$; eigenvalues of the block $\mathbf T_m$ are $\nu_k$, reserving $\mu$ for measures;
and eigenvalues of the anchored stiffness matrix are $\lambda_k$, $k\in[m]$, each of multiplicity two.
Notation used only inside proofs is collected in Appendix~\ref{app:additional-notation}.
}

BKLs construct recursive reproducing-kernel representations from the one-dimensional Brownian kernel
\citep{anonymous2026bkl},
\begin{align}
k_{\mathrm B}(x,x')=\frac{|x|+|x'|-|x-x'|}{2}={\min(|x|,|x'|)\,\mathbb 1\{xx'\ge0\}},
\qquad x,x'\in\mathbb R,
\end{align}
{the covariance of two-sided Brownian motion anchored at the origin. Its RKHS on $I$,
denoted $\mathcal H_{\mathrm B}$, is the Cameron--Martin space
$\mathcal H_{\mathrm B}=\{f\in H^1(I):f(0)=0\}$ with
$\langle f,g\rangle_{\mathcal H_{\mathrm B}}=\int_{I}f'g'$ \citep{berlinet2004rkhs}.}
Given support functions $\mathcal S$ {on an input space $\mathcal X$, each
$u:\mathcal X\rightarrow\mathbb R$,} and a probability measure $\mu$ on $\mathcal S$, the Brownian
integral kernel
$k[\mathcal S,\mu](\mathbf x,\mathbf x')=\int_{\mathcal S}k_{\mathrm B}(u(\mathbf x),u(\mathbf x'))\,\mathrm d\mu(u)$
has an RKHS defining the next layer; iterating from linear projections yields a hierarchy of Brownian
RKHSs. Variation Brownian Kernel Ladders (VBKLs) instead use an atomic representation \citep{anonymous2026vbkl}: composing support functions
with Brownian profiles yields a dictionary $\mathcal U_L
=
\left\{
g\circ u:
u\in \mathcal U_{L-1},
\; g\in\mathcal H_{\mathrm B},
\; \left\|g\right\|_{\mathcal H_{\mathrm B}}\le1
\right\}$ of depth-$L$ atoms, with variation space
$\mathcal V^{(L)}=\{F=\int_{\mathcal U_L}u\,\mathrm d\mu(u):\|\mu\|_{\mathrm{TV}}<\infty\}$ and
complexity
$\widehat{\mathcal C}_{\mathrm{var}}^{(L)}(F)=\inf\{\|\mu\|_{\mathrm{TV}}:F=\int_{\mathcal U_L}u\,\mathrm d\mu(u)\}$.
In both, profile functions lie in infinite-dimensional Brownian RKHSs, so implementations must discretize
them. The VBKL construction already employs anchored piecewise-linear profiles and their increment-energy
constraint; the present paper makes the corresponding finite RKHS and coordinate metrics explicit for a
controlled optimization comparison.

\section{Finite Brownian RKHS}
\label{sec:finite-rkhs}

We collect the classical finite-element and RKHS machinery in the paper's two-sided anchored
normalization \citep{strang1973analysis,brenner2008finite}, ensuring that later optimizer comparisons use
the same function space, norm, and approximation operator.

\begin{definition}[Finite Brownian profile space]
\label{def:finite-brownian-profile-space}

Let $I=\left[-A,A\right]$ for some $A>0$, let {$G\in\mathbb N$, and let
$t_i=-A+ih$ for $i=0,\ldots,G$ be the uniform grid with mesh size $h=\frac{2A}{G}$,
so that $-A=t_0<t_1<\cdots<t_G=A$.}
For each $i=0,\ldots,G$, let {$\phi_i\in C(I)$ be the hat function that is affine on each
subinterval $[t_j,t_{j+1}]$ and satisfies $\phi_i(t_j)=\delta_{ij}$ for every $j=0,\ldots,G$
\citep{brenner2008finite}.} Define the finite-element reconstruction operator
$\mathcal R:\mathbb R^{G+1}\rightarrow C(I)$ by
$\mathcal R(\mathbf v)=\sum_{i=0}^{G}v_i\phi_i$, where
$\mathbf v=(v_0,\ldots,v_G)^{\top}\in\mathbb R^{G+1}$. The \emph{finite Brownian profile space} is
$\mathcal H_h=\mathcal R(\mathbb R^{G+1})
=\operatorname{span}(\phi_0,\ldots,\phi_G)$.
{Since $\phi_i(t_j)=\delta_{ij}$, the family $\{\phi_i\}_{i=0}^{G}$ is linearly independent,
so every $f\in\mathcal H_h$ has a unique coefficient vector $\mathbf v\in\mathbb R^{G+1}$ with
$f=\mathcal R(\mathbf v)$.}

\end{definition}

\begin{definition}[{Brownian energy form}]
\label{def:discrete-brownian-inner-product}
{On $\mathcal H_h$ we define} the symmetric bilinear form
$\langle f,g\rangle_{B,h}=\int_{-A}^{A} f'(t)g'(t)\,dt$, {with derivatives taken
elementwise on each $[t_i,t_{i+1}]$ and hence defined almost everywhere,} and the Brownian seminorm
$\|f\|_{B,h}=\sqrt{\langle f,f\rangle_{B,h}}$.
\end{definition}
\begin{proposition}[Finite Brownian energy representation]
\label{prop:finite-brownian-energy}

Let $\mathbf D\in\mathbb R^{G\times(G+1)}$ be the first-difference matrix,
$(\mathbf D\mathbf v)_i=v_{i+1}-v_i$ for $i=0,\ldots,G-1$, and let
$\mathbf K=(K_{rs})_{r,s=0}^{G}$ with $K_{rs}=\langle\phi_r,\phi_s\rangle_{B,h}$ be the stiffness
matrix. For $f=\mathcal R(\mathbf v)\in\mathcal H_h$,
\begin{align}
\left\|f\right\|_{B,h}^{2}
=
\mathbf{v}^{\top}\mathbf{K}\mathbf{v}
=
\frac{1}{h}
\left\|\mathbf{D}\mathbf{v}\right\|_2^{2}
=
\frac{1}{h}
\sum_{i=0}^{G-1}
\left(v_{i+1}-v_i\right)^2.
\end{align}
{Since both matrices are symmetric, this is equivalent to the matrix identity}
$\mathbf{K}=\frac{1}{h}\mathbf{D}^{\top}\mathbf{D}$.

\end{proposition}

{This recovers the classical stiffness-matrix form of the Dirichlet energy \citep{strang1973analysis,brenner2008finite}; we record it because it identifies the regularizer that finite realizations implement.}
\begin{corollary}[Implemented Brownian regularizer]
\label{cor:implemented-brownian-regularizer}

For $\mathbf v\in\mathbb R^{G+1}$ let $\Omega_h(\mathbf v):=\sum_{i=0}^{G-1}(v_{i+1}-v_i)^2$. Then
$\Omega_h(\mathbf v)=h\|\mathcal R(\mathbf v)\|_{B,h}^{2}$: the implemented regularizer is the
Brownian energy times the fixed mesh factor $h>0$.

\end{corollary}

\subsection{Anchored Finite Brownian RKHS}
\label{subsec:anchored-finite-rkhs}

The Brownian energy is constant-invariant and hence only a seminorm on $\mathcal H_h$. As in the
Cameron--Martin space \citep{berlinet2004rkhs}, uniqueness is recovered by fixing the value at a
reference point.

{
\begin{assumption}
\label{ass:even-grid}
$G=2m$ with $m\in\mathbb N$, and $i_0:=m$, so $t_{i_0}=0$. In force for the remainder of the paper.
\end{assumption}
}

\begin{definition}[Anchored finite Brownian profile space]
\label{def:anchored-finite-profile-space}
$\mathcal H_h^0=\{f\in\mathcal H_h:f(0)=0\}$; equivalently
$f=\mathcal R(\mathbf v)\in\mathcal H_h^0$ iff $v_{i_0}=0${, and
$\dim\mathcal H_h^0=G$}.
\end{definition}

{Since the anchor removes one coordinate, we work with a free parameter in
$\mathbb R^{G}$. The ordering below traverses each half-grid from its exterior endpoint toward the
anchor, which is what makes the two diagonal blocks of \Cref{thm:complete-dct-spectrum} the same
matrix.}

{
\begin{definition}[Reduced anchored coordinates]
\label{def:reduced-anchored-coordinates}
For an anchored nodal vector $\mathbf v\in\mathbb R^{G+1}$ with $v_{i_0}=0$, the \emph{reduced anchored
vector} is
$\widetilde{\mathbf v}:=(v_0,\ldots,v_{m-1},v_{2m},v_{2m-1},\ldots,v_{m+1})^{\top}\in\mathbb R^{G}$,
and the \emph{anchoring reconstruction matrix} is
$\mathbf R:=(\mathbf e_0,\ldots,\mathbf e_{m-1},\mathbf e_{2m},\mathbf e_{2m-1},\ldots,\mathbf e_{m+1})
\in\mathbb R^{(G+1)\times G}$, so that $\mathbf R\widetilde{\mathbf v}=\mathbf v$. We write
$\mathbf K_0:=\mathbf R^{\top}\mathbf K\mathbf R$ for the \emph{anchored stiffness matrix} and
$\mathbf D_0:=\mathbf D\mathbf R\in\mathbb R^{G\times G}$.
\end{definition}
}

{The anchored energy is positive definite, which is what makes the anchored space an inner-product space rather than merely a seminormed one.}

{
\begin{lemma}[Positive definiteness]
\label{lem:K0-pd}
$\ker\mathbf K=\operatorname{span}(\mathbf 1)$ and $\mathbf K_0\succ0$. Consequently
$\langle\cdot,\cdot\rangle_{B,h}$ is an inner product on $\mathcal H_h^0$.
\end{lemma}
}

We can now record the Brownian-specific kernel formula and its covariance interpretation. The value of
the result here is its exact normalization and its use in the subsequent controlled comparison, rather
than finite-dimensional RKHS existence itself.

\begin{theorem}[Finite Brownian RKHS and its reproducing kernel]
\label{thm:finite-brownian-rkhs}
The space $(\mathcal H_h^0,\langle\cdot,\cdot\rangle_{B,h})$ is a $G$-dimensional RKHS. Let
$\mathcal J:=\{0,\ldots,G\}\setminus\{i_0\}$ and let
$\pi:\mathcal J\rightarrow[G]$ give the position of a node in the reduced ordering of
\Cref{def:reduced-anchored-coordinates}. With
$\widetilde{\boldsymbol\phi}(t):=(\phi_{\pi^{-1}(1)}(t),\ldots,\phi_{\pi^{-1}(G)}(t))^{\top}$, its
reproducing kernel satisfies:
\begin{enumerate}[(i)]
\item\label{thm:finite-brownian-rkhs-i}
$k_h(s,t)=\widetilde{\boldsymbol\phi}(s)^{\top}\mathbf K_0^{-1}
\widetilde{\boldsymbol\phi}(t)$ for all $s,t\in I$.
\item\label{thm:finite-brownian-rkhs-ii}
$(\mathbf K_0^{-1})_{\pi(r),\pi(q)}=k_{\mathrm B}(t_r,t_q)$ for $r,q\in\mathcal J$, and
$k_h(\cdot,t_r)=k_{\mathrm B}(\cdot,t_r)$ for every grid node. More generally, $k_h$ is the tensor-product
piecewise-bilinear interpolant of $k_{\mathrm B}$ on the grid, and therefore agrees with
$k_{\mathrm B}$ whenever at least one argument is a node.
\end{enumerate}
\end{theorem}

\begin{corollary}[Conforming Brownian interpolation space]
\label{cor:conforming-subspace}
\begin{align}
\mathcal H_h^0
=\operatorname{span}\left\{k_{\mathrm B}(\cdot,t_r):r\in\mathcal J\right\}
\subset\mathcal H_{\mathrm B},
\qquad
\|g\|_{B,h}=\|g\|_{\mathcal H_{\mathrm B}}
\quad(g\in\mathcal H_h^0).
\end{align}
Thus the inclusion is isometric, and the $\mathcal H_{\mathrm B}$-orthogonal projection
$\Pi_h:\mathcal H_{\mathrm B}\rightarrow\mathcal H_h^0$ is exactly nodal interpolation.
\end{corollary}

\begin{corollary}[Exact residual kernel and sharp approximation]
\label{cor:approximation}
Let $r_h:=k_{\mathrm B}-k_h$ and let $T_i=[a_i,b_i]=[t_i,t_{i+1}]$. Then
\begin{align}
r_h(s,t)
=
\begin{cases}
\displaystyle
\frac{\left(\min\{s,t\}-a_i\right)\left(b_i-\max\{s,t\}\right)}{h},
& s,t\in T_i,\\[2mm]
0,&s,t\text{ lie in different grid elements}.
\end{cases}
\label{eq:residual-kernel}
\end{align}
Consequently, for $t\in T_i$ the power function is
\begin{align}
p_h(t)^2:=r_h(t,t)=\frac{(t-a_i)(b_i-t)}{h},
\qquad
|f(t)-\Pi_hf(t)|\le p_h(t)\|f\|_{\mathcal H_{\mathrm B}}.
\label{eq:power-function}
\end{align}
For every $f\in\mathcal H_{\mathrm B}$,
\begin{align}
\|f-\Pi_hf\|_{\mathcal H_{\mathrm B}}^2
&=\|f\|_{\mathcal H_{\mathrm B}}^2-\|\Pi_hf\|_{\mathcal H_{\mathrm B}}^2,\\
\|f-\Pi_hf\|_{\infty}
&\le\frac{\sqrt h}{2}\|f\|_{\mathcal H_{\mathrm B}},
&
\|f-\Pi_hf\|_{L^2(I)}
&\le\frac{h}{\pi}\|f\|_{\mathcal H_{\mathrm B}}.
\label{eq:sharp-approximation-bounds}
\end{align}
Both constants are sharp over the unit ball. These quantities depend only on the subspace
$\mathcal H_h^0$, so they are identical under every coordinate realization of that space.
\end{corollary}

The residual has the Brownian-bridge covariance form, and the bounds are standard power-function and
sharp interval-inequality consequences
\citep{rasmussen2006gaussian,fasshauer2007meshfree,brenner2008finite}; their role is to certify identical
approximation across coordinates.

\section{Equivalent Coordinate Representations}
\label{sec:equivalent-coordinate-representations}

The finite Brownian RKHS admits nodal, increment, and spectral realizations. Their invertible linear
relation is elementary; the substantive control is that they realize exactly the same space, Brownian norm,
and approximation operator. Consequently, any mapped optimization difference must arise from the
optimizer or the coordinate metric rather than representational capacity.

\subsection{Nodal and Increment Coordinate Representations}
\label{subsec:nodal-increment}

We first consider the nodal finite-element realization and its associated increment representation.

\begin{definition}[Nodal coordinates]
\label{def:nodal-coordinates}

Every function
$f\in\mathcal H_h^0$
admits the unique nodal representation
$f=\mathcal R(\mathbf R\widetilde{\mathbf v})$
{with reduced anchored vector $\widetilde{\mathbf v}\in\mathbb R^{G}$, and
$\left\|f\right\|_{B,h}^{2}=\widetilde{\mathbf v}^{\top}\mathbf K_0\widetilde{\mathbf v}
=\mathbf v^{\top}\mathbf K\mathbf v$ for $\mathbf v=\mathbf R\widetilde{\mathbf v}$.
Throughout, $\widetilde{\mathbf v}\in\mathbb R^{G}$ is the free parameter; the anchor is not a trainable
coordinate.}

\end{definition}

{The nodal realization is the parameterization induced directly by the finite-element
discretization: each trainable parameter is the value of the profile at a grid point.}

\begin{definition}[Increment coordinates]
\label{def:increment-coordinates}

The increment coordinates associated with the {reduced} nodal vector
{$\widetilde{\mathbf v}$}
are defined by
{$\mathbf w=\mathbf D_0\widetilde{\mathbf v}$}, or equivalently,
$w_i=v_{i+1}-v_i$ for $i=0,\ldots,G-1$ {with $\mathbf v=\mathbf R\widetilde{\mathbf v}$}.

\end{definition}

{The increment realization stores local differences rather than grid values. Because the
profile is anchored, $\mathbf D_0$ is invertible (\Cref{thm:equivalence-of-implementations}\ref{thm:equivalence-of-implementations-ii}) and the increments determine
$f$ uniquely. In these coordinates the Brownian energy is simply
$\|f\|_{B,h}^{2}=\tfrac1h\|\mathbf w\|_2^{2}$.}

\subsection{Spectral Coordinate Representation}
\label{subsec:spectral-coordinates}

A complementary realization is obtained by diagonalizing the anchored Brownian energy, yielding an orthogonal spectral parameterization of the finite Brownian RKHS.
\begin{definition}[Spectral coordinates]
\label{def:spectral-coordinates}

{By \Cref{lem:K0-pd}, $\mathbf K_0\succ0$. Let}
$\mathbf K_0=\mathbf Q\boldsymbol{\Lambda}\mathbf Q^\top$
be {a} fixed orthogonal eigendecomposition, with
$\mathbf Q\in\mathbb R^{G\times G}$ orthogonal and
$\boldsymbol{\Lambda}=\operatorname{diag}(\lambda_1,\ldots,\lambda_G)$
{having strictly positive diagonal.}
For the reduced anchored nodal vector
$\widetilde{\mathbf v}\in\mathbb R^G$,
define
$\mathbf c=\mathbf Q^\top\widetilde{\mathbf v}$, or equivalently,
$\widetilde{\mathbf v}=\mathbf Q\mathbf c$.

\end{definition}

These coordinates diagonalize the Brownian norm, so each spectral coefficient contributes independently
to the total energy. The reconstruction and norm {identities are} proved in
Appendix~\ref{app:coordinate-equivalence}; \Cref{sec:explicit-spectrum} gives the explicit DCT-VIII
block-structured choice of $\mathbf Q$.

The three coordinate systems differ in how they store a Brownian profile: the nodal realization stores profile values, the increment realization stores local differences, and the spectral realization stores coefficients in an orthogonal Brownian basis. The following theorem shows that they are nevertheless equivalent at the level of the function space they represent.

\begin{theorem}[Equivalence of implementations]
\label{thm:equivalence-of-implementations}

Let
$f\in\mathcal H_h^0$
be arbitrary, let
$\widetilde{\mathbf v}\in\mathbb R^G$
denote its reduced anchored nodal coefficient vector, and set
$\mathbf v=\mathbf R\widetilde{\mathbf v}\in\mathbb R^{G+1}$ and
$\mathbf D_0:=\mathbf D\mathbf R\in\mathbb R^{G\times G}$.
Define the increment and spectral coordinates by
$\mathbf{w}=\mathbf D_0\widetilde{\mathbf v}$ and
$\mathbf c=\mathbf Q^{\top}\widetilde{\mathbf v}$, respectively.
Then the following statements hold.

\begin{enumerate}[(i)]

\item
The nodal, increment, and spectral representations determine exactly the same
function $f\in\mathcal H_h^0$.\label{thm:equivalence-of-implementations-i}

\item
The maps
$\widetilde{\mathbf v}\mapsto\mathbf{w}$
and
$\widetilde{\mathbf v}\mapsto\mathbf c$
are invertible linear transformations. Consequently, each coordinate system
uniquely determines the other two.\label{thm:equivalence-of-implementations-ii}

\item\label{thm:equivalence-of-implementations-iii}
The Brownian norm admits the equivalent representations
\begin{align}
\left\|f\right\|_{B,h}^{2}
=
\mathbf v^{\top}\mathbf K\mathbf v
=
\frac{1}{h}
\left\|\mathbf{w}\right\|_2^{2}
=
\mathbf c^{\top}\boldsymbol{\Lambda}\mathbf c.
\end{align}

\item\label{thm:equivalence-of-implementations-iv}
{The three realizations therefore represent the same hypothesis class and, by
\Cref{cor:approximation}, share the same approximation properties.}

\item\label{thm:equivalence-of-implementations-v}
The nodal-to-spectral transformation is {an isometry for the Euclidean metric on the
reduced anchored coordinates}, whereas the
nodal-to-increment transformation satisfies
$\mathbf D_0^{\top}\mathbf D_0=h\mathbf K_0$.
Consequently, for $G=2m\geq4$, the increment coordinate map is
nonorthogonal, whereas the spectral coordinate map is orthogonal.

\end{enumerate}

\end{theorem}

{The restriction $G\ge4$ excludes only $G=2$, where $h\mathbf K_0=\mathbf I_2$ and the
increment map is itself orthogonal.} \Cref{thm:equivalence-of-implementations} gives both the
function-space equivalence and {the exact relation between the induced Euclidean
metrics}: the spectral map is orthogonal, the increment map induces the metric $h\mathbf K_0$.
{\Cref{sec:optimization-geometry} shows that this one matrix identity determines which
optimizers see the three realizations as the same problem.}

\section{Explicit Spectral Representation}
\label{sec:explicit-spectrum}

Using classical DCT nomenclature \citep{martucci1994symmetric,strang1999dct,masera2017odd}, we give all
eigenpairs of the fixed mixed-boundary stiffness matrix in closed block DCT-VIII form; ``complete'' does
not mean a new transform family.

\subsection{Block Decomposition}
\label{subsec:block-decomposition}

{After anchoring, the stiffness matrix decouples into two independent tridiagonal blocks. The reason is that the anchor node $t_{i_0}=0$ is the unique node coupling the two half-grids; removing it severs them. In the reduced ordering of \Cref{def:reduced-anchored-coordinates}, which traverses both half-grids from the exterior endpoint inward, the two blocks are literally the same matrix.}

\begin{theorem}[Complete DCT-VIII eigendecomposition]
\label{thm:complete-dct-spectrum}

Let {$\mathbf T_m=(\tau_{rs})_{r,s=0}^{m-1}\in\mathbb R^{m\times m}$ be the tridiagonal
matrix with $\tau_{00}=1$, $\tau_{rr}=2$ for $1\le r\le m-1$, $\tau_{rs}=-1$ for $|r-s|=1$, and
$\tau_{rs}=0$ otherwise. Then $\mathbf K_0=\tfrac1h\operatorname{blkdiag}(\mathbf T_m,\mathbf T_m)$
(\Cref{lem:block-decomposition}).} For each $k\in[m]$, define
$\theta_k=\frac{(2k-1)\pi}{2m+1}$ and
$\nu_k
=
2-2\cos(\theta_k)
=
4\sin^2\left(\frac{\theta_k}{2}\right)$,
and let the vector
$\mathbf q_k\in\mathbb R^m$
be defined componentwise by
$q_k(j)=\frac{2}{\sqrt{2m+1}}
\cos\left(\left(j+\frac12\right)\theta_k\right)$
for $j=0,\ldots,m-1$.

Then the following statements hold.

\begin{enumerate}[(i)]

\item\label{thm:complete-dct-spectrum-i}
The vectors
$\mathbf q_1,\ldots,\mathbf q_m$
form an orthonormal basis of
$\mathbb R^m$.

\item\label{thm:complete-dct-spectrum-ii}
With $\mathbf Q_m=[\mathbf q_1,\ldots,\mathbf q_m]$, the matrix $\mathbf T_m$ admits the
eigendecomposition
$\mathbf T_m=\mathbf Q_m\operatorname{diag}(\nu_1,\ldots,\nu_m)\mathbf Q_m^\top$.

\item\label{thm:complete-dct-spectrum-iii}
The anchored stiffness matrix admits the orthogonal eigendecomposition
$\mathbf K_0=\mathbf Q\boldsymbol{\Lambda}\mathbf Q^\top$,
where $\mathbf Q=\operatorname{blkdiag}(\mathbf Q_m,\mathbf Q_m)$ and
$\boldsymbol{\Lambda}=\tfrac1h\operatorname{diag}(\nu_1,\ldots,\nu_m,\nu_1,\ldots,\nu_m)$.

\item\label{thm:complete-dct-spectrum-iv}
The eigenvalues of $\mathbf K_0$ are
$\lambda_k=\tfrac4h\sin^2\big(\tfrac{(2k-1)\pi}{2(2m+1)}\big)$ for $k\in[m]$, each of multiplicity two.

\end{enumerate}

\end{theorem}

The eigendecomposition induces the basis
$\psi_i:=\mathcal R(\mathbf R\mathbf Q\mathbf e_i)$ of $\mathcal H_h^0$, satisfying
$\langle\psi_i,\psi_j\rangle_{B,h}=\lambda_i\delta_{ij}$
(\Cref{lem:explicit-spectral-basis}). Hence the finite RKHS kernel has the spectral basis expansion
\begin{align}
k_h(s,t)=\sum_{i=1}^{G}\lambda_i^{-1}\psi_i(s)\psi_i(t).
\end{align}
This is a finite RKHS expansion; no input measure or integral-operator eigenproblem is required.

\begin{remark}
Every eigenvalue has multiplicity two, so the eigenbasis is not canonical inside each left--right
eigenspace. We fix $\mathbf Q=\operatorname{blkdiag}(\mathbf Q_m,\mathbf Q_m)$, the explicit
block-structured choice matching the reduced ordering. Any such orthogonal choice gives the same gradient
descent trajectory after mapping, but coordinatewise adaptive methods may distinguish the choices.
\end{remark}

\begin{corollary}[Exact stiffness and coordinate conditioning]
\label{cor:condition-number}
Under \Cref{ass:even-grid},
\begin{align}
\kappa_2(\mathbf K_0)
&=\frac{\cos^{2}\!\left(\frac{\pi}{2m+1}\right)}
{\sin^{2}\!\left(\frac{\pi}{4m+2}\right)}
=\frac{4}{\pi^2}(2m+1)^2+O(1)=\Theta(G^2),
\label{eq:K0-condition}\\
\sigma_{\min}(\mathbf D_0)
&=2\sin\!\left(\frac{\pi}{4m+2}\right),~~~
\sigma_{\max}(\mathbf D_0)
=2\cos\!\left(\frac{\pi}{2m+1}\right),
\label{eq:D0-singular-values}\\
\kappa_2(\mathbf D_0)
&=\frac{\cos\!\left(\frac{\pi}{2m+1}\right)}
{\sin\!\left(\frac{\pi}{4m+2}\right)}
=\sqrt{\kappa_2(\mathbf K_0)}=\Theta(G).
\label{eq:D0-condition}
\end{align}
The condition numbers are independent of $A$. Moreover,
$h\lambda_{\max}(\mathbf K_0)\rightarrow4$ and
$\lambda_{\min}(\mathbf K_0)\sim\pi^2/[h(2m+1)^2]$ as $m\rightarrow\infty$.
\end{corollary}

The first identity gives the classical $O(h^{-2})$ stiffness conditioning; the second and third quantify
the anisotropy of the actual nodal-to-increment coordinate map.

\section{Coordinates and Optimizers}
\label{sec:optimization-geometry}

We begin with the standard chain-rule reparameterization identity
\citep{gunasekar2018optimization,amid2020reparameterizing,kristiadi2023geometry}. Let
$L:\mathbb R^G\rightarrow\mathbb R$ be differentiable in reduced nodal coordinates and, for an
invertible matrix $\mathbf A$, define $L_{\mathbf A}(\mathbf z):=L(\mathbf A^{-1}\mathbf z)$.

\begin{proposition}[Linear reparameterization is exact preconditioning]
\label{prop:precond}
If
$\mathbf z_{k+1}=\mathbf z_k-\eta_k\nabla L_{\mathbf A}(\mathbf z_k)$ and
$\widetilde{\mathbf v}_k=\mathbf A^{-1}\mathbf z_k$, then
\begin{align}
\widetilde{\mathbf v}_{k+1}
=\widetilde{\mathbf v}_k
-\eta_k(\mathbf A^{\top}\mathbf A)^{-1}\nabla L(\widetilde{\mathbf v}_k).
\label{eq:reparameterized-gd}
\end{align}
The same identity holds for stochastic gradients computed from the same minibatches and transformed by
the chain rule.
\end{proposition}

\begin{proposition}[Coordinate--optimizer correspondence]
\label{prop:gd-equivariance}
\label{prop:increment-natural-gradient}
For corresponding initializations, identical scalar step sizes, identical minibatches, and exact
arithmetic:
\begin{enumerate}[(i)]
\item Since $\mathbf c=\mathbf Q^{\top}\widetilde{\mathbf v}$ and $\mathbf Q$ is orthogonal, nodal and
spectral gradient descent produce identical mapped iterates and hence identical functions at every step.
\item Since $\mathbf w=\mathbf D_0\widetilde{\mathbf v}$ and
$\mathbf D_0^{\top}\mathbf D_0=h\mathbf K_0$, increment gradient descent maps to
\begin{align}
\widetilde{\mathbf v}_{k+1}
=\widetilde{\mathbf v}_k-\frac{\eta_k}{h}\mathbf K_0^{-1}
\nabla L(\widetilde{\mathbf v}_k).
\label{eq:increment-brownian-gradient}
\end{align}
Thus it is an explicit Euler step generated by the constant Brownian/Sobolev metric with Gram matrix
$\mathbf K_0$, with time step $\eta_k/h$. We use ``Riemannian-gradient step'' only in this discrete,
constant-metric sense; the statement is not equality with the continuous gradient flow.
\end{enumerate}
\end{proposition}

\begin{corollary}[Blockwise extension to recursive architectures]
\label{cor:blockwise-optimizer}
Let an arbitrary differentiable composed objective depend on profile blocks
$\widetilde{\mathbf v}^{(1)},\ldots,\widetilde{\mathbf v}^{(B)}$ and possibly other parameters, and apply
the block transformation
$\mathbf A=\operatorname{blkdiag}(\mathbf A_1,\ldots,\mathbf A_B,\mathbf I)$. Then
\Cref{prop:precond} holds with block preconditioner
$\operatorname{blkdiag}((\mathbf A_1^{\top}\mathbf A_1)^{-1},\ldots,
(\mathbf A_B^{\top}\mathbf A_B)^{-1},\mathbf I)$. Hence simultaneous spectral changes leave the full
GD/SGD trajectory invariant, while simultaneous increment coordinates implement blockwise Brownian
Riemannian-gradient steps with the corresponding per-block factors $1/h_b$. No quadratic or
layer-separability assumption is required.
\end{corollary}

This is the block-diagonal application of \Cref{prop:precond}; cross-block coupling does not invalidate the
identity. Here ``mesh-independent'' means independent of $G$ for fixed $A$, $\rho>0$, sample domain, and
normalization, not uniform in $A$ or $\rho$.

The previous identities describe trajectories. The next result gives the corresponding quantitative
conditioning statement for the regularized regression problem generated by the finite RKHS.

\begin{proposition}[Mesh-independent increment conditioning]
\label{prop:least-squares-conditioning}
For samples $x_1,\ldots,x_n\in I$, targets $y_1,\ldots,y_n\in\mathbb R$, and $\rho>0$, consider
\begin{align}
L(\widetilde{\mathbf v})
=\frac{1}{2n}\sum_{i=1}^{n}
\left(\widetilde{\boldsymbol\phi}(x_i)^{\top}\widetilde{\mathbf v}-y_i\right)^2
+\frac{\rho}{2}\widetilde{\mathbf v}^{\top}\mathbf K_0\widetilde{\mathbf v}.
\label{eq:regularized-least-squares}
\end{align}
Nodal and spectral Hessians are orthogonally similar. In increment coordinates
$\mathbf w=\mathbf D_0\widetilde{\mathbf v}$, the Hessian $\mathbf H_{\mathrm{inc}}$ satisfies
\begin{align}
\frac{\rho}{h}\mathbf I_G
\preceq \mathbf H_{\mathrm{inc}}
\preceq \frac{\rho+A}{h}\mathbf I_G,
\qquad
\kappa_2(\mathbf H_{\mathrm{inc}})\le1+\frac{A}{\rho},
\label{eq:increment-hessian-bound}
\end{align}
independently of $G$. For the pure Brownian quadratic, the increment condition number is exactly $1$,
whereas the nodal and spectral condition numbers equal $\kappa_2(\mathbf K_0)=\Theta(G^2)$.
\end{proposition}

Standard Adam requires a separate statement because momentum prevents its update from being represented
as a diagonal matrix times only the current gradient.

\begin{proposition}[Exact orthogonal equivariance group of standard Adam]
\label{prop:adam-nonequivariance}
Consider standard Adam with scalar $\beta_1,\beta_2\in[0,1)$, scalar $\varepsilon>0$, zero initial moment
states, bias correction, a positive first step size, and no weight decay, clipping, or other modification.
Among orthogonal coordinate changes, Adam is equivariant for every gradient sequence and corresponding
initialization if and only if the coordinate change is a signed permutation. For $m\ge2$, the DCT-VIII
matrix $\mathbf Q$ is not a signed permutation; therefore there exists even a linear objective for which
corresponding nodal and spectral Adam runs differ at the first step.
\end{proposition}

This maximality statement strengthens basis dependence to an if-and-only-if universal group under the
exact convention; it does not cover AdamW, clipping, nonzero moments, or matrix-valued adaptivity.

Thus GD gives exact nodal--spectral agreement; increment GD gives exact Brownian preconditioning; and Adam
is permitted, but not forced, to separate the orthogonally related realizations. A separation theorem does
not impose a universal ordering, and independently sampled initializations do not test trajectory
equivalence. These predictions require mapped initial states, the same data order and scalar step sequence,
and the same intrinsic Brownian penalty.

\section{Experiments}
\label{sec:experiments}

We separate direct theorem-verification tests from predictive comparisons. The verification suite
comprises E0, the one-layer optimization tests E1A--E1C, and the recursive extension E2. All five tests use frozen
configurations, float64 arithmetic, one common represented initialization, exact coordinate maps, the same
scalar step sequences, and, for SGD, identical minibatches; no identity-test arm is tuned separately.
Secondary EuroSAT and spatially separated Salinas studies compare different model classes and are reported
only in the appendix.

\begin{figure}[t]
\centering
\includegraphics[width=0.323\textwidth]{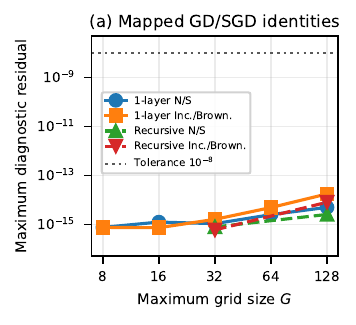}
\hfill
\includegraphics[width=0.323\textwidth]{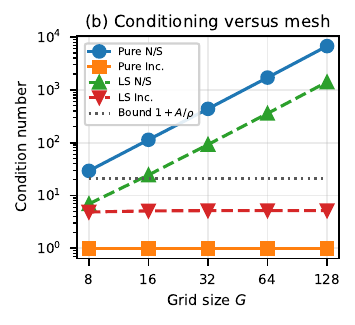}
\hfill
\includegraphics[width=0.323\textwidth]{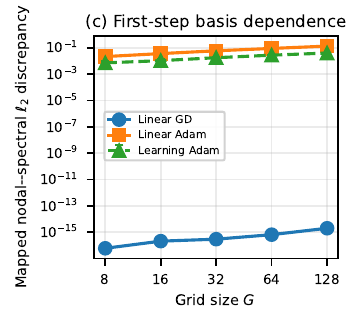}
\caption{Direct numerical verification of the coordinate--optimizer theory. (a) Each marker is the maximum over the recorded parameter, function, objective, gradient, and preconditioner diagnostics and over full-batch GD and shared-minibatch SGD. The recursive markers correspond to grid triples $(8,16,32)$ and $(32,64,128)$, plotted at their maximum $G$; dashed segments guide the eye. All measured residuals remain far below the predeclared $10^{-8}$ tolerance. (b) For the pure Brownian quadratic, nodal/spectral conditioning grows with the mesh whereas increment conditioning is exactly one. Least-squares markers are five-seed means; increment conditioning remains independent of $G$ under the stated fixed quantities and below $1+A/\rho=21$. (c) Ordinary GD remains equivariant, whereas standard Adam has a nonzero mapped nodal--spectral first-step discrepancy. Learning-objective markers are five-seed means and bars span the minimum-to-maximum range. The ordinate is a coordinate-mapping discrepancy, not a predictive loss or error. These results establish permitted basis dependence, not a universal performance ordering.}
\label{fig:theory-validation}
\end{figure}

\paragraph{Finite and mapped identities.}
Across $G\in\{8,16,32,64,128\}$, the matrix, DCT-VIII spectrum, Brownian-Gram inverse,
reconstruction, and intrinsic-energy residuals remain at floating-point precision. Mapped nodal--spectral
full-batch GD and shared-minibatch SGD agree in parameters, functions, objectives, and gradients. Increment
GD agrees with the correctly scaled Brownian-preconditioned nodal update. The same conclusions hold for a
recursively coupled three-profile model with all non-profile parameters trained and shared across arms.

\begin{table}[t]
\centering
\caption{Float64 verification of the exact finite-Brownian and mapped-optimizer identities. The last four rows report maxima over the listed diagnostics; these are identity checks, not predictive-performance comparisons.}
\label{tab:theory-validation-identities}
\small
\setlength{\tabcolsep}{4pt}
\resizebox{\textwidth}{!}{%
\begin{tabular}{llcl}
\toprule
Check & Scope & Maximum residual & Type \\
\midrule
$\mathbf D_0^\top\mathbf D_0=h\mathbf K_0$ & 5 grids & $0$ & absolute \\
$\mathbf K_0=\mathbf Q\boldsymbol\Lambda\mathbf Q^\top$ & 5 grids & $9.99\times 10^{-15}$ & relative \\
Closed-form spectrum & 5 grids & $1.41\times 10^{-13}$ & relative \\
$\mathbf K_0^{-1}$ versus Brownian Gram & 5 grids & $0$ & absolute \\
Coordinate reconstruction & 125 random trials & $3.12\times 10^{-14}$ & absolute \\
Intrinsic Brownian energy & 125 random trials & $3.99\times 10^{-15}$ & relative \\
One-layer Nodal--Spectral GD/SGD & 50 mapped runs & $5.02\times 10^{-15}$ & diagnostic envelope \\
One-layer Increment--Brownian GD/SGD & 50 mapped runs & $1.77\times 10^{-14}$ & diagnostic envelope \\
Recursive Nodal--Spectral GD/SGD & 20 mapped runs & $2.57\times 10^{-15}$ & diagnostic envelope \\
Recursive Increment--Brownian GD/SGD & 20 mapped runs & $8.10\times 10^{-15}$ & diagnostic envelope \\
\bottomrule
\end{tabular}%
}
\end{table}

\paragraph{Conditioning and standard Adam.}
For the pure Brownian quadratic, nodal/spectral conditioning grows from $29.3$ at $G=8$ to $6740.7$ at
$G=128$, whereas increment conditioning is exactly one. For Brownian-regularized least squares with
$\rho=0.1$, the largest measured increment condition number is $5.64$, below the stated
$G$-independent bound $1+A/\rho=21$, while nodal/spectral conditioning and fixed-step iteration counts
increase with $G$.
A functional standard-Adam implementation matches \texttt{torch.optim.Adam}, but mapped nodal and
spectral first steps differ on both the deterministic linear counterexample and a nontrivial regularized
learning objective. This demonstrates permitted basis dependence, not a universal performance ordering.

\begin{table}[t]
\centering
\caption{Conditioning, fixed-step convergence, and standard-Adam basis dependence across meshes. ``N/S'' denotes identical nodal/spectral quantities and ``Inc.'' the increment realization. Iteration counts use the analytic optimal scalar step and relative objective-gap tolerance $10^{-8}$; least squares uses $\rho=0.1$, with increment bound $1+A/\rho=21$. The last columns report mapped nodal--spectral first-step discrepancies.}
\label{tab:theory-validation-conditioning-adam}
\scriptsize
\setlength{\tabcolsep}{3pt}
\resizebox{\textwidth}{!}{%
\begin{tabular}{ccccccc}
\toprule
$G$ & $\kappa(\mathbf K_0)$ & Pure iters N/S\,/\,Inc. & LS $\kappa$ N/S\,/\,Inc. & LS iters N/S\,/\,Inc. & Adam $\Delta_1$ linear & Adam $\Delta_1$ learning \\
\midrule
$8$ & $29.3$ & $123.2\,/\,1.0$ & $6.94\,/\,4.86$ & $31.2\,/\,21.6$ & $0.02$ & $7.09\times 10^{-3}$ \\
$16$ & $113.5$ & $467.0\,/\,1.0$ & $24.45\,/\,5.11$ & $109.6\,/\,22.6$ & $0.04$ & $0.01$ \\
$32$ & $437.7$ & $1795.2\,/\,1.0$ & $92.15\,/\,5.16$ & $413.0\,/\,23.4$ & $0.06$ & $0.02$ \\
$64$ & $1.7\times 10^{3}$ & $7004.2\,/\,1.0$ & $357.64\,/\,5.17$ & $1601.8\,/\,23.4$ & $0.09$ & $0.03$ \\
$128$ & $6.7\times 10^{3}$ & $27629.0\,/\,1.0$ & $1408.71\,/\,5.17$ & $6310.2\,/\,23.4$ & $0.13$ & $0.04$ \\
\bottomrule
\end{tabular}%
}
\end{table}

\paragraph{External validity.}
Appendix~\ref{app:experiments} reports independently audited EuroSAT and spatially separated Salinas
BKL-versus-MLP robustness studies. They provide secondary predictive evidence, not tests of coordinate
equivariance.

\section{Conclusion}
\label{sec:conclusion}

Equivalent coordinates of one finite function space need not define equivalent optimizers. Classical
finite-element, interpolation, covariance, and DCT identities provide our fixed Brownian comparison model.
Within it, nodal and spectral GD/SGD coincide; increment descent is a $1/h$-scaled constant-metric
Brownian/Sobolev Euler step, also blockwise; the least-squares condition number is at most $1+A/\rho$
independently of $G$ under the stated fixed quantities; and signed permutations are the maximal universal
orthogonal equivariances of standard Adam. Float64 tests numerically verify these predictions.

\paragraph{Limitations.}
Closed forms assume a one-dimensional uniform grid and fixed anchor. The least-squares bound is not uniform
in $A$ or $\rho$ or a nonlinear convergence rate; the Adam theorem covers only its stated update and gives
no coordinate ranking. EuroSAT and Salinas compare model classes rather than coordinates, and eigenvalue
multiplicity makes the displayed DCT-VIII basis noncanonical within each two-dimensional eigenspace.

\FloatBarrier
\label{page:main-end}
\clearpage
\bibliographystyle{plainnat}
\bibliography{references}

\subsection*{AI use statement}
Generative AI tools were used to assist with the writing and revision of
mathematical proofs, feedback on experimental design, editing of research code,
language editing, LaTeX formatting, and literature search. All AI-assisted
material incorporated into the paper was reviewed and checked for correctness.
The authors take responsibility for the final content of this work, including
its mathematical claims, proofs, code, experiments, citations, and other
artifacts.

\clearpage
\appendix

\section{Additional Notation}\label{app:additional-notation}

\begin{table}[H]
\centering
\caption{Summary of the main theoretical results.}
\label{tab:main-results}
\begin{tabular}{ll}
\toprule
Result & Content \\
\midrule
\Cref{prop:finite-brownian-energy} & Classical stiffness identity in the Brownian normalization \\
\Cref{thm:finite-brownian-rkhs} & Brownian-specific kernel and covariance--precision formula \\
\Cref{cor:conforming-subspace,cor:approximation} & Classical interpolation consequences with exact constants \\
\Cref{thm:equivalence-of-implementations} & Exact equivalence of the three realizations \\
\Cref{thm:complete-dct-spectrum,cor:condition-number} & Closed-form block DCT-VIII specialization \\
\Cref{prop:precond,prop:gd-equivariance} & Standard reparameterization and Brownian metric consequence \\
\Cref{cor:blockwise-optimizer} & Exact blockwise recursive extension \\
\Cref{prop:least-squares-conditioning} & $G$-independent conditioning for fixed $A$ and $\rho$ \\
\Cref{prop:adam-nonequivariance} & Maximal universal orthogonal group of standard Adam \\
\bottomrule
\end{tabular}
\end{table}

{The reduced anchored ordering, the anchoring reconstruction matrix $\mathbf R$, the
anchored stiffness matrix $\mathbf K_0=\mathbf R^{\top}\mathbf K\mathbf R$, and $\mathbf D_0=\mathbf D\mathbf R$
are now defined in the main text (\Cref{def:reduced-anchored-coordinates}), together with the standing
assumption $G=2m$, $i_0=m$ (\Cref{ass:even-grid}); they are not repeated here. For reference, the
reduced ordering is}
\begin{align}
\widetilde{\mathbf v}
:=
\left(v_0,\ldots,v_{m-1},v_{2m},v_{2m-1},\ldots,v_{m+1}\right)^{\top}\in\mathbb R^{2m},
\qquad
\mathbf R\widetilde{\mathbf v}=\mathbf v,
\label{eq:reduced-anchored-ordering}
\end{align}
{so that both half-grids run from their exterior endpoint toward the anchor, and}
\begin{align}
\mathbf R
:=
\left(\mathbf e_0,\ldots,\mathbf e_{m-1},\mathbf e_{2m},\mathbf e_{2m-1},\ldots,\mathbf e_{m+1}\right)
\in\mathbb R^{(2m+1)\times 2m}.
\label{eq:anchoring-reconstruction-matrix}
\end{align}
{That the Brownian energy has a one-dimensional constant nullspace, so that anchoring
yields $\mathbf K_0\succ0$, is established by \Cref{lem:kernel-stiffness,lem:K0-pd}; it is proved rather than assumed.}

The first-difference operator
$\mathbf{D}:\mathbb{R}^{G+1}\rightarrow\mathbb{R}^{G}$
is defined by
$\left(\mathbf{D}\mathbf{v}\right)_i=v_{i+1}-v_i$
for
$i=0,\ldots,G-1$, and the finite-element stiffness matrix
$\mathbf K=(K_{rs})_{r,s=0}^{G}$ entrywise by
$K_{rs}=\int_{-A}^{A}\phi_r'(t)\phi_s'(t)\,\mathrm dt$.

\begin{figure}[!t]
\centering
\input{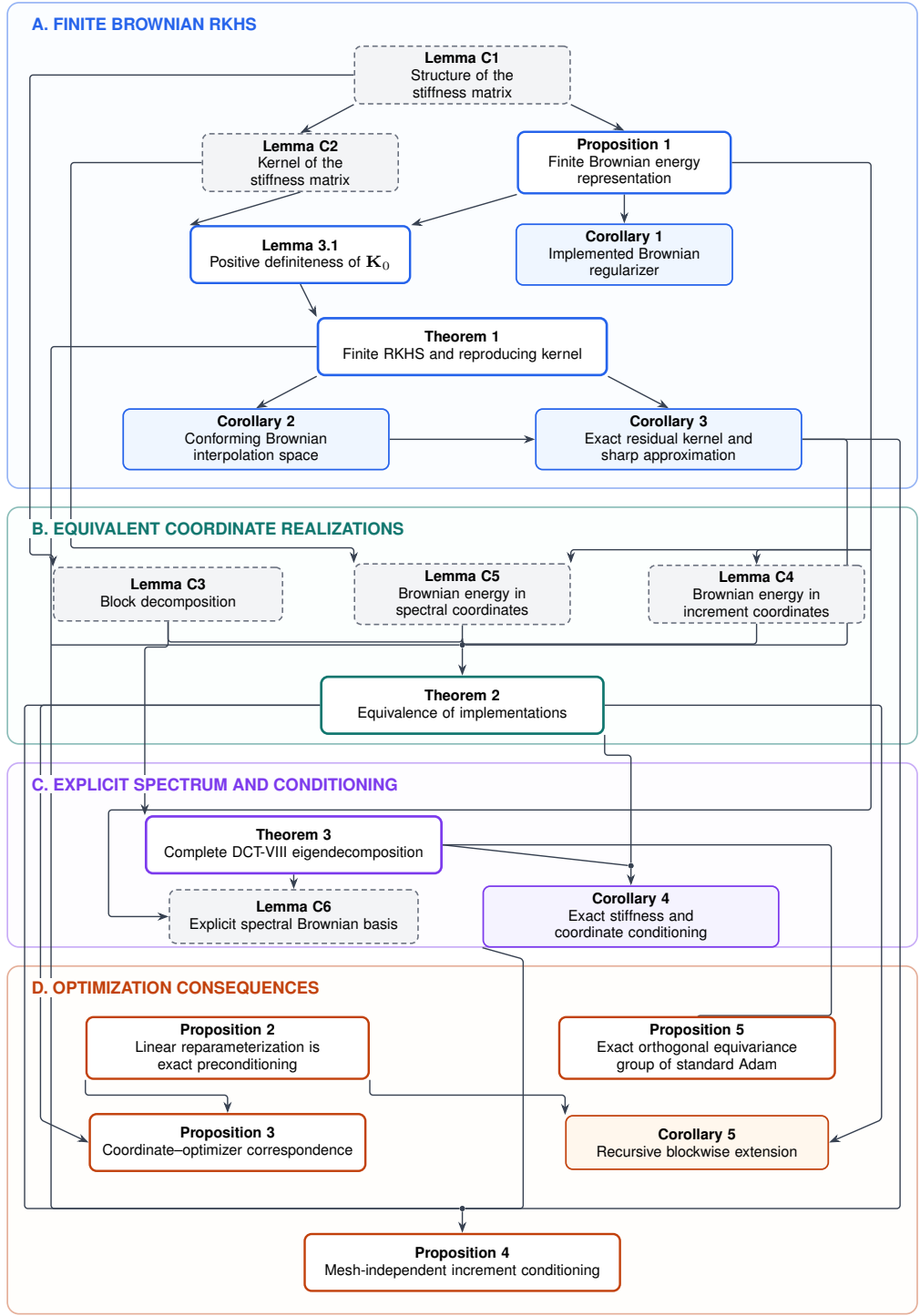}
\caption{
Proof-dependency structure of the theoretical results. The displayed artwork is the supplied dependency graph without redrawing or re-typesetting; panel headings and result nodes are clickable and jump to the corresponding section or statement.}
\label{fig:dependency_graph}
\end{figure}

\section{Proofs}
\label{app:foundational-proofs}
This section proves the foundational finite Brownian and coordinate
results. We first establish the finite Brownian energy and stiffness
representation in
\Cref{prop:finite-brownian-energy}
and its implemented-regularizer consequence in
\Cref{cor:implemented-brownian-regularizer}.
We then prove the
$G$-dimensional Hilbert-space and reproducing-kernel structure asserted in
\Cref{thm:finite-brownian-rkhs};
the explicit kernel formula and Brownian Gram identification in
items~\ref{thm:finite-brownian-rkhs-i}--\ref{thm:finite-brownian-rkhs-ii}
are isolated in
\Cref{app:new-proofs}.
Next, we prove the exact equivalence of the nodal, increment, and spectral
realizations in
\Cref{thm:equivalence-of-implementations},
including the coordinate maps, Brownian-energy identities, and induced
metric relations.
Finally, we derive the complete DCT-VIII eigendecomposition in
\Cref{thm:complete-dct-spectrum}.
The auxiliary stiffness, nullspace, block-decomposition, and
coordinate-energy results used throughout are collected in
\Cref{app:auxi_lems}.
\subsection{Proof of \texorpdfstring{\Cref{prop:finite-brownian-energy}}{Proposition 1}}
\label{app:proof-finite-brownian-energy}

\begin{noheadproof}

Since
$A>0$
and the grid is strictly increasing, its mesh size satisfies
$h=\frac{2A}{G}>0$.
Because
$f=\mathcal R(\mathbf v)$,
the definition of the reconstruction operator gives, for every
$t\in I$,
\begin{align}
f(t)
&\stackrel{(a)}{=}
\mathcal R(\mathbf v)(t)
\stackrel{(b)}{=}
\sum_{i=0}^{G}
v_i\phi_i(t).
\label{eq:proposition-energy-reconstruction}
\end{align}
Here (a) uses the assumed identity
$f=\mathcal R(\mathbf v)$,
and (b) is the definition of
$\mathcal R$.
Each basis function
$\phi_i$
is continuous and piecewise affine on the finite grid, and is therefore
absolutely continuous on
$I$
and differentiable away from the grid points.
Consequently,
$f$
is absolutely continuous and, for almost every
$t\in I$,
\begin{align}
f'(t)
&\stackrel{(a)}{=}
\left(
\sum_{i=0}^{G}
v_i\phi_i(t)
\right)'
\stackrel{(b)}{=}
\sum_{i=0}^{G}
v_i\phi_i'(t).
\label{eq:proposition-energy-derivative}
\end{align}
Here (a) differentiates
\eqref{eq:proposition-energy-reconstruction}
at a point at which all basis derivatives exist, and (b) uses the linearity
of differentiation and the fact that each coefficient
$v_i$
is independent of
$t$.
The possible failure of differentiability at the finitely many grid points
does not affect any of the integrals below.

We first identify the stiffness-matrix representation of the Brownian energy.
By Definition~\ref{def:discrete-brownian-inner-product},
\begin{align}
\left\|f\right\|_{B,h}^{2}
&\stackrel{(a)}{=}
\left\langle f,f\right\rangle_{B,h}
\stackrel{(b)}{=}
\int_{-A}^{A}
f'(t)f'(t)
\,\mathrm dt
\nonumber\\
&\stackrel{(c)}{=}
\int_{-A}^{A}
\left(
\sum_{r=0}^{G}
v_r\phi_r'(t)
\right)
\left(
\sum_{s=0}^{G}
v_s\phi_s'(t)
\right)
\,\mathrm dt
\nonumber\\
&\stackrel{(d)}{=}
\int_{-A}^{A}
\sum_{r=0}^{G}
\sum_{s=0}^{G}
v_rv_s
\phi_r'(t)\phi_s'(t)
\,\mathrm dt
\nonumber\\
&\stackrel{(e)}{=}
\sum_{r=0}^{G}
\sum_{s=0}^{G}
v_rv_s
\int_{-A}^{A}
\phi_r'(t)\phi_s'(t)
\,\mathrm dt
\nonumber\\
&\stackrel{(f)}{=}
\sum_{r=0}^{G}
\sum_{s=0}^{G}
v_r
\left\langle
\phi_r,
\phi_s
\right\rangle_{B,h}
v_s
\stackrel{(g)}{=}
\sum_{r=0}^{G}
\sum_{s=0}^{G}
v_rK_{rs}v_s
\nonumber\\
&\stackrel{(h)}{=}
\mathbf v^{\top}
\mathbf K
\mathbf v.
\label{eq:proposition-energy-stiffness}
\end{align}
Here (a) is the definition of the Brownian seminorm associated with the
symmetric bilinear energy form, (b) is the definition of that energy form,
(c) substitutes
\eqref{eq:proposition-energy-derivative},
(d) expands the product of the two finite sums, (e) interchanges the integral
with finite sums by linearity, (f) applies the definition of
$\langle\phi_r,\phi_s\rangle_{B,h}$,
(g) uses the definition
$K_{rs}=\langle\phi_r,\phi_s\rangle_{B,h}$,
and (h) is the componentwise expansion of the quadratic form
$\mathbf v^{\top}\mathbf K\mathbf v$.

By Lemma~\ref{lem:stiffness-matrix},
\begin{align}
\mathbf K
\stackrel{(a)}{=}
\frac{1}{h}
\mathbf D^{\top}
\mathbf D.
\label{eq:proposition-energy-factorization}
\end{align}
Here (a) invokes the stiffness-matrix factorization proved in
Lemma~\ref{lem:stiffness-matrix}.
In particular, the dimensions
$\mathbf D\in\mathbb R^{G\times(G+1)}$
and
$\mathbf v\in\mathbb R^{G+1}$
ensure that
$\mathbf D\mathbf v\in\mathbb R^{G}$
and that every product below is well defined.
Using
\eqref{eq:proposition-energy-factorization},
\begin{align}
\mathbf v^{\top}
\mathbf K
\mathbf v
&\stackrel{(a)}{=}
\mathbf v^{\top}
\left(
\frac{1}{h}
\mathbf D^{\top}
\mathbf D
\right)
\mathbf v
\stackrel{(b)}{=}
\frac{1}{h}
\mathbf v^{\top}
\mathbf D^{\top}
\mathbf D
\mathbf v
\nonumber\\
&\stackrel{(c)}{=}
\frac{1}{h}
\left(
\mathbf D\mathbf v
\right)^{\top}
\left(
\mathbf D\mathbf v
\right)
\stackrel{(d)}{=}
\frac{1}{h}
\left\|
\mathbf D\mathbf v
\right\|_2^{2}.
\label{eq:proposition-energy-difference-norm}
\end{align}
Here (a) substitutes
\eqref{eq:proposition-energy-factorization},
(b) moves the scalar
$\frac{1}{h}$
outside the quadratic form, (c) uses the transpose identity
$(\mathbf D\mathbf v)^{\top}=\mathbf v^{\top}\mathbf D^{\top}$
and associativity of matrix multiplication, and (d) applies
$\mathbf w^{\top}\mathbf w=\|\mathbf w\|_2^2$
with
$\mathbf w=\mathbf D\mathbf v$.

By the definition of the Euclidean norm and the componentwise definition of
$\mathbf D$,
\begin{align}
\left\|
\mathbf D\mathbf v
\right\|_2^{2}
&\stackrel{(a)}{=}
\sum_{i=0}^{G-1}
\left(
(\mathbf D\mathbf v)_i
\right)^2
\stackrel{(b)}{=}
\sum_{i=0}^{G-1}
\left(
v_{i+1}-v_i
\right)^2.
\label{eq:proposition-energy-difference-sum}
\end{align}
Here (a) expands the squared Euclidean norm of the vector
$\mathbf D\mathbf v\in\mathbb R^G$,
and (b) uses
$(\mathbf D\mathbf v)_i=v_{i+1}-v_i$
for every
$i=0,\ldots,G-1$.
Combining
\eqref{eq:proposition-energy-stiffness},
\eqref{eq:proposition-energy-difference-norm},
and
\eqref{eq:proposition-energy-difference-sum}
yields
\begin{align}
\left\|f\right\|_{B,h}^{2}
&\stackrel{(a)}{=}
\mathbf v^{\top}
\mathbf K
\mathbf v
\stackrel{(b)}{=}
\frac{1}{h}
\left\|
\mathbf D\mathbf v
\right\|_2^{2}
\stackrel{(c)}{=}
\frac{1}{h}
\sum_{i=0}^{G-1}
\left(
v_{i+1}-v_i
\right)^2.
\label{eq:proposition-energy-final-chain}
\end{align}
Here (a) is
\eqref{eq:proposition-energy-stiffness},
(b) is
\eqref{eq:proposition-energy-difference-norm},
and (c) substitutes
\eqref{eq:proposition-energy-difference-sum}.
Finally,
\eqref{eq:proposition-energy-factorization}
is exactly the equivalent matrix identity asserted in the proposition.
Therefore all claimed representations follow.

\end{noheadproof}\hfill\ensuremath{\square}

\subsection{Proof of \texorpdfstring{\Cref{cor:implemented-brownian-regularizer}}{Corollary 1}}
\label{app:proof-implemented-brownian-regularizer}

\begin{noheadproof}

Fix an arbitrary vector
$\mathbf v=(v_0,\ldots,v_G)^{\top}\in\mathbb R^{G+1}$
and define
\begin{align}
f
:=
\mathcal R\left(\mathbf v\right).
\label{eq:cor-regularizer-function}
\end{align}
By Definition~\ref{def:finite-brownian-profile-space},
$\mathcal R$ maps $\mathbb R^{G+1}$ into $\mathcal H_h$.
Consequently,
\begin{align}
f
&\stackrel{(a)}{=}
\mathcal R\left(\mathbf v\right)
\stackrel{(b)}{\in}
\mathcal H_h.
\label{eq:cor-regularizer-membership}
\end{align}
Here (a) is the definition \eqref{eq:cor-regularizer-function}, and (b) follows
from
$\mathcal R(\mathbb R^{G+1})=\mathcal H_h$
and
$\mathbf v\in\mathbb R^{G+1}$.
Therefore Proposition~\ref{prop:finite-brownian-energy} applies to $f$ and its
nodal vector $\mathbf v$.

By the definition of the implemented finite Brownian regularizer,
\begin{align}
\Omega_h\left(\mathbf v\right)
:=
\sum_{i=0}^{G-1}
\left(v_{i+1}-v_i\right)^2.
\label{eq:cor-regularizer-definition}
\end{align}
On the other hand, Proposition~\ref{prop:finite-brownian-energy}, applied using
\eqref{eq:cor-regularizer-membership}, gives
\begin{align}
\left\|
\mathcal R\left(\mathbf v\right)
\right\|_{B,h}^{2}
\stackrel{(a)}{=}
\frac{1}{h}
\sum_{i=0}^{G-1}
\left(v_{i+1}-v_i\right)^2.
\label{eq:cor-regularizer-energy}
\end{align}
Here (a) is the final identity in
Proposition~\ref{prop:finite-brownian-energy}, applied to the function
$f=\mathcal R(\mathbf v)$ established in
\eqref{eq:cor-regularizer-membership}.

{
Since $h=2A/G>0$, multiplying \eqref{eq:cor-regularizer-energy} by $h$ gives
\begin{align}
\Omega_h\left(\mathbf v\right)
=
h\left\|\mathcal R\left(\mathbf v\right)\right\|_{B,h}^{2}.
\label{eq:cor-regularizer-conclusion}
\end{align}
As $\mathbf v\in\mathbb R^{G+1}$ was arbitrary and $h$ depends only on the fixed grid, the implemented
regularizer is the discrete Brownian energy scaled by the fixed positive factor $h$.
}

\end{noheadproof}\hfill\ensuremath{\square}

\subsection{Proof of
\texorpdfstring{\Cref{thm:finite-brownian-rkhs}}{Theorem 1}:
RKHS existence and dimension}
\label{app:proof-finite-brownian-rkhs}

\begin{noheadproof}

We first identify the linear structure and dimension of the anchored space.
Set
\begin{align}
\mathcal I_0
:=
\left\{
0,\ldots,G
\right\}
\setminus
\left\{
i_0
\right\}.
\label{eq:finite-rkhs-index-set}
\end{align}
Since
$t_{i_0}=0$
and the nodal basis satisfies
$\phi_j(t_r)=\delta_{jr}$,
we have, for every
$j=0,\ldots,G$,
\begin{align}
\phi_j(0)
&\stackrel{(a)}{=}
\phi_j\left(t_{i_0}\right)
\stackrel{(b)}{=}
\delta_{j i_0}.
\label{eq:finite-rkhs-anchor-basis-values}
\end{align}
Here (a) uses
$t_{i_0}=0$,
and (b) is the nodal interpolation property of the finite-element basis.
Let
$f\in\mathcal H_h$
be arbitrary, and let
$\mathbf v=(v_0,\ldots,v_G)^{\top}$
be its unique nodal vector, so that
$f=\mathcal R(\mathbf v)$.
Evaluating this representation at the anchor and using
\eqref{eq:finite-rkhs-anchor-basis-values} gives
\begin{align}
f(0)
&\stackrel{(a)}{=}
\mathcal R(\mathbf v)(0)
\stackrel{(b)}{=}
\sum_{j=0}^{G}
v_j\phi_j(0)
\stackrel{(c)}{=}
\sum_{j=0}^{G}
v_j\delta_{j i_0}
\stackrel{(d)}{=}
v_{i_0}.
\label{eq:finite-rkhs-anchor-coordinate}
\end{align}
Here (a) uses
$f=\mathcal R(\mathbf v)$;
(b) is the definition of the reconstruction operator;
(c) substitutes
\eqref{eq:finite-rkhs-anchor-basis-values};
and (d) uses the defining property of the Kronecker delta, which removes every
term except the term with index
$j=i_0$.
Consequently,
\begin{align}
f\in\mathcal H_h^0
&\stackrel{(a)}{\Longleftrightarrow}
f(0)=0
\stackrel{(b)}{\Longleftrightarrow}
v_{i_0}=0
\stackrel{(c)}{\Longleftrightarrow}
f
=
\sum_{j\in\mathcal I_0}
v_j\phi_j.
\label{eq:finite-rkhs-basis-characterization}
\end{align}
Here (a) is the definition of
$\mathcal H_h^0$;
(b) uses
\eqref{eq:finite-rkhs-anchor-coordinate};
and (c) uses the unique nodal representation and the fact that the coefficient
of
$\phi_{i_0}$
is zero exactly when
$v_{i_0}=0$.
It follows that
\begin{align}
\mathcal H_h^0
\stackrel{(a)}{=}
\operatorname{span}
\left\{
\phi_j
:
j\in\mathcal I_0
\right\}.
\label{eq:finite-rkhs-anchor-span}
\end{align}
Here (a) is precisely the set equality expressed by
\eqref{eq:finite-rkhs-basis-characterization}.
In particular,
$\mathcal H_h^0$
is a vector subspace of
$\mathcal H_h$.

We next verify that the spanning family in
\eqref{eq:finite-rkhs-anchor-span}
is linearly independent.
Suppose that real coefficients
$(a_j)_{j\in\mathcal I_0}$
satisfy
\begin{align}
\sum_{j\in\mathcal I_0}
a_j\phi_j
=
0.
\label{eq:finite-rkhs-linear-combination-zero}
\end{align}
Fix an arbitrary
$r\in\mathcal I_0$
and evaluate
\eqref{eq:finite-rkhs-linear-combination-zero}
at the grid point
$t_r$.
Then
\begin{align}
0
&\stackrel{(a)}{=}
\sum_{j\in\mathcal I_0}
a_j\phi_j(t_r)
\stackrel{(b)}{=}
\sum_{j\in\mathcal I_0}
a_j\delta_{jr}
\stackrel{(c)}{=}
a_r.
\label{eq:finite-rkhs-basis-independence}
\end{align}
Here (a) evaluates the zero-function identity
\eqref{eq:finite-rkhs-linear-combination-zero}
at
$t_r$;
(b) uses
$\phi_j(t_r)=\delta_{jr}$;
and (c) again uses the defining property of the Kronecker delta.
Since
$r\in\mathcal I_0$
was arbitrary, every coefficient
$a_r$
vanishes.
Therefore
$\{\phi_j:j\in\mathcal I_0\}$
is a basis of
$\mathcal H_h^0$.
Because
$\mathcal I_0$
is obtained by removing one index from the
$G+1$
indices
$0,\ldots,G$,
\begin{align}
\dim\left(\mathcal H_h^0\right)
&\stackrel{(a)}{=}
\left|\mathcal I_0\right|
\stackrel{(b)}{=}
(G+1)-1
\stackrel{(c)}{=}
G.
\label{eq:finite-rkhs-dimension}
\end{align}
Here (a) uses the basis just established;
(b) uses the definition
\eqref{eq:finite-rkhs-index-set};
and (c) simplifies the integer expression.
Thus the anchored space is finite-dimensional.

We now prove that the Brownian energy form is an inner product on this space.
Every function in
$\mathcal H_h^0\subset\mathcal H_h$
is continuous and piecewise affine on a finite partition of
$I$.
Hence it is absolutely continuous on
$I$,
its derivative exists almost everywhere, and its derivative is piecewise
constant and therefore belongs to
$L^2(I)$.
For arbitrary
$f,g\in\mathcal H_h^0$,
the Cauchy--Schwarz inequality gives
\begin{align}
\int_{-A}^{A}
\left|
f'(t)g'(t)
\right|
\,\mathrm dt
&\stackrel{(a)}{\leq}
\left(
\int_{-A}^{A}
\left|f'(t)\right|^2
\,\mathrm dt
\right)^{1/2}
\left(
\int_{-A}^{A}
\left|g'(t)\right|^2
\,\mathrm dt
\right)^{1/2}
\stackrel{(b)}{<}
\infty.
\label{eq:finite-rkhs-product-integrability}
\end{align}
Here (a) is the Cauchy--Schwarz inequality on
$L^2(I)$,
and (b) follows from
$f',g'\in L^2(I)$.
Thus the integral defining
$\langle f,g\rangle_{B,h}$
is finite and the form is well defined.

Let
$f_1,f_2,g\in\mathcal H_h^0$
and
$\alpha,\beta\in\mathbb R$.
Because
$\mathcal H_h^0$
is a vector space, the function
$\alpha f_1+\beta f_2$
also belongs to
$\mathcal H_h^0$.
Using linearity of the weak derivative and of the integral,
\begin{align}
\left\langle
\alpha f_1+\beta f_2,
g
\right\rangle_{B,h}
&\stackrel{(a)}{=}
\int_{-A}^{A}
\left(
\alpha f_1+\beta f_2
\right)'(t)
g'(t)
\,\mathrm dt
\nonumber\\
&\stackrel{(b)}{=}
\int_{-A}^{A}
\left(
\alpha f_1'(t)+\beta f_2'(t)
\right)
g'(t)
\,\mathrm dt
\nonumber\\
&\stackrel{(c)}{=}
\alpha
\int_{-A}^{A}
f_1'(t)g'(t)
\,\mathrm dt
+
\beta
\int_{-A}^{A}
f_2'(t)g'(t)
\,\mathrm dt
\nonumber\\
&\stackrel{(d)}{=}
\alpha
\left\langle
f_1,g
\right\rangle_{B,h}
+
\beta
\left\langle
f_2,g
\right\rangle_{B,h}.
\label{eq:finite-rkhs-linearity-first}
\end{align}
Here (a) is the definition of the Brownian energy form;
(b) uses
$(\alpha f_1+\beta f_2)'=\alpha f_1'+\beta f_2'$
almost everywhere;
(c) distributes the product and uses linearity of the integral; and
(d) applies the definition of the Brownian energy form to the two resulting
integrals.
For arbitrary
$f,g\in\mathcal H_h^0$,
commutativity of multiplication in
$\mathbb R$
yields
\begin{align}
\left\langle
f,g
\right\rangle_{B,h}
&\stackrel{(a)}{=}
\int_{-A}^{A}
f'(t)g'(t)
\,\mathrm dt
\stackrel{(b)}{=}
\int_{-A}^{A}
g'(t)f'(t)
\,\mathrm dt
\stackrel{(c)}{=}
\left\langle
g,f
\right\rangle_{B,h}.
\label{eq:finite-rkhs-symmetry}
\end{align}
Here (a) and (c) are the definition of the Brownian energy form, and (b) uses
$f'(t)g'(t)=g'(t)f'(t)$
for almost every
$t\in I$.
Linearity in the second argument now follows explicitly from
\eqref{eq:finite-rkhs-linearity-first}
and
\eqref{eq:finite-rkhs-symmetry}:
for
$f,g_1,g_2\in\mathcal H_h^0$
and
$\alpha,\beta\in\mathbb R$,
\begin{align}
\left\langle
f,
\alpha g_1+\beta g_2
\right\rangle_{B,h}
&\stackrel{(a)}{=}
\left\langle
\alpha g_1+\beta g_2,
f
\right\rangle_{B,h}
\nonumber\\
&\stackrel{(b)}{=}
\alpha
\left\langle
g_1,f
\right\rangle_{B,h}
+
\beta
\left\langle
g_2,f
\right\rangle_{B,h}
\nonumber\\
&\stackrel{(c)}{=}
\alpha
\left\langle
f,g_1
\right\rangle_{B,h}
+
\beta
\left\langle
f,g_2
\right\rangle_{B,h}.
\label{eq:finite-rkhs-linearity-second}
\end{align}
Here (a) uses symmetry;
(b) applies
\eqref{eq:finite-rkhs-linearity-first};
and (c) uses symmetry separately in each term.
Therefore the form is bilinear and symmetric.
Moreover, for every
$f\in\mathcal H_h^0$,
\begin{align}
\left\langle
f,f
\right\rangle_{B,h}
&\stackrel{(a)}{=}
\int_{-A}^{A}
f'(t)f'(t)
\,\mathrm dt
\stackrel{(b)}{=}
\int_{-A}^{A}
\left|f'(t)\right|^2
\,\mathrm dt
\stackrel{(c)}{\geq}
0.
\label{eq:finite-rkhs-nonnegative}
\end{align}
Here (a) is the definition of the energy form;
(b) uses that the functions are real valued; and
(c) follows because
$|f'(t)|^2\geq0$
almost everywhere.

It remains to prove definiteness.
Fix
$f\in\mathcal H_h^0$
and let
$\mathbf v\in\mathbb R^{G+1}$
be its unique nodal vector, so that
\begin{align}
f
&\stackrel{(a)}{=}
\mathcal R(\mathbf v),
&
v_{i_0}
&\stackrel{(b)}{=}
0.
\label{eq:finite-rkhs-anchored-vector}
\end{align}
Here (a) follows from
$f\in\mathcal H_h$ and Definition~\ref{def:finite-brownian-profile-space},
and (b) follows from
$f\in\mathcal H_h^0$
and
\eqref{eq:finite-rkhs-anchor-coordinate}.
Assume that
\begin{align}
\left\langle
f,f
\right\rangle_{B,h}
=
0.
\label{eq:finite-rkhs-zero-assumption}
\end{align}
By the definition of the Brownian seminorm and
Proposition~\ref{prop:finite-brownian-energy},
\begin{align}
0
&\stackrel{(a)}{=}
\left\langle
f,f
\right\rangle_{B,h}
\stackrel{(b)}{=}
\left\|f\right\|_{B,h}^{2}
\stackrel{(c)}{=}
\mathbf v^{\top}
\mathbf K
\mathbf v
\stackrel{(d)}{=}
\frac{1}{h}
\left\|
\mathbf D\mathbf v
\right\|_2^2.
\label{eq:finite-rkhs-zero-energy}
\end{align}
Here (a) is the assumption
\eqref{eq:finite-rkhs-zero-assumption};
(b) is the definition of the Brownian seminorm;
and (c)--(d) are the stiffness and first-difference representations in
Proposition~\ref{prop:finite-brownian-energy}.
The mesh size is strictly positive:
\begin{align}
h
&\stackrel{(a)}{=}
\frac{2A}{G}
\stackrel{(b)}{>}
0.
\label{eq:finite-rkhs-positive-mesh}
\end{align}
Here (a) is the mesh-size definition, and (b) follows because
$A>0$
and the strictly increasing grid contains
$G\geq1$
subintervals.
Multiplying the first and last quantities in
\eqref{eq:finite-rkhs-zero-energy}
by
$h$
and using
\eqref{eq:finite-rkhs-positive-mesh},
we obtain
\begin{align}
\left\|
\mathbf D\mathbf v
\right\|_2^2
&\stackrel{(a)}{=}
h
\left(
\frac{1}{h}
\left\|
\mathbf D\mathbf v
\right\|_2^2
\right)
\stackrel{(b)}{=}
h\cdot0
\stackrel{(c)}{=}
0.
\label{eq:finite-rkhs-zero-difference-norm}
\end{align}
Here (a) uses
$h(1/h)=1$;
(b) uses
\eqref{eq:finite-rkhs-zero-energy};
and (c) uses
$h\cdot0=0$.
For each
$i=0,\ldots,G-1$,
\begin{align}
0
&\stackrel{(a)}{\leq}
\left(
(\mathbf D\mathbf v)_i
\right)^2
\stackrel{(b)}{\leq}
\sum_{r=0}^{G-1}
\left(
(\mathbf D\mathbf v)_r
\right)^2
\stackrel{(c)}{=}
\left\|
\mathbf D\mathbf v
\right\|_2^2
\stackrel{(d)}{=}
0.
\label{eq:finite-rkhs-zero-difference-components}
\end{align}
Here (a) uses nonnegativity of a square;
(b) uses that the sum on the right contains the nonnegative term on the left;
(c) is the componentwise definition of the Euclidean norm; and
(d) uses
\eqref{eq:finite-rkhs-zero-difference-norm}.
Therefore every component
$(\mathbf D\mathbf v)_i$
is zero and hence
\begin{align}
\mathbf D\mathbf v
\stackrel{(a)}{=}
\mathbf0.
\label{eq:finite-rkhs-zero-difference-vector}
\end{align}
Here (a) follows from
\eqref{eq:finite-rkhs-zero-difference-components}
for all component indices.
Using the factorization in
Proposition~\ref{prop:finite-brownian-energy},
\begin{align}
\mathbf K\mathbf v
&\stackrel{(a)}{=}
\left(
\frac{1}{h}
\mathbf D^{\top}
\mathbf D
\right)
\mathbf v
\stackrel{(b)}{=}
\frac{1}{h}
\mathbf D^{\top}
\left(
\mathbf D\mathbf v
\right)
\stackrel{(c)}{=}
\frac{1}{h}
\mathbf D^{\top}
\mathbf0
\stackrel{(d)}{=}
\mathbf0.
\label{eq:finite-rkhs-kernel-membership}
\end{align}
Here (a) uses
$\mathbf K=\frac1h\mathbf D^{\top}\mathbf D$;
(b) uses associativity of matrix multiplication;
(c) substitutes
\eqref{eq:finite-rkhs-zero-difference-vector};
and (d) uses that every matrix maps the zero vector to the zero vector.
Thus
$\mathbf v\in\ker(\mathbf K)$.
Lemma~\ref{lem:kernel-stiffness}
now yields
\begin{align}
\mathbf v
&\stackrel{(a)}{\in}
\ker(\mathbf K)
\stackrel{(b)}{=}
\operatorname{span}(\mathbf1)
\stackrel{(c)}{\Longrightarrow}
\mathbf v
=
c\mathbf1
\quad
\text{for some }
c\in\mathbb R.
\label{eq:finite-rkhs-constant-vector}
\end{align}
Here (a) follows from
\eqref{eq:finite-rkhs-kernel-membership};
(b) is Lemma~\ref{lem:kernel-stiffness};
and (c) is the definition of the one-dimensional span of
$\mathbf1=(1,\ldots,1)^{\top}\in\mathbb R^{G+1}$.
The anchoring constraint in
\eqref{eq:finite-rkhs-anchored-vector}
then gives
\begin{align}
0
&\stackrel{(a)}{=}
v_{i_0}
\stackrel{(b)}{=}
(c\mathbf1)_{i_0}
\stackrel{(c)}{=}
c.
\label{eq:finite-rkhs-anchor-constant}
\end{align}
Here (a) is the anchoring condition;
(b) substitutes
$\mathbf v=c\mathbf1$;
and (c) uses
$\mathbf1_{i_0}=1$.
Therefore
$c=0$
and
$\mathbf v=\mathbf0$.
Substituting this vector into the reconstruction formula gives
\begin{align}
f
&\stackrel{(a)}{=}
\mathcal R(\mathbf v)
\stackrel{(b)}{=}
\mathcal R(\mathbf0)
\stackrel{(c)}{=}
\sum_{j=0}^{G}
0\phi_j
\stackrel{(d)}{=}
0.
\label{eq:finite-rkhs-zero-function}
\end{align}
Here (a) is
\eqref{eq:finite-rkhs-anchored-vector};
(b) uses
$\mathbf v=\mathbf0$;
(c) is the definition of the reconstruction operator; and
(d) simplifies the zero linear combination.
Together with nonnegativity in
\eqref{eq:finite-rkhs-nonnegative},
this proves that the Brownian energy form is positive definite on
$\mathcal H_h^0$.
Hence
$\langle\cdot,\cdot\rangle_{B,h}$
is an inner product on
$\mathcal H_h^0$,
and
$\|\cdot\|_{B,h}$
is its induced norm.

{
Completeness and boundedness of evaluation are automatic in finite dimensions, but the two quantitative
bounds behind them are used later and are recorded here.

Let $f=\mathcal R(\mathbf v)\in\mathcal H_h^0$, so $v_{i_0}=0$. Telescoping from the anchor,
$v_j=-\sum_{r=j}^{i_0-1}(v_{r+1}-v_r)$ for $j<i_0$ and $v_j=\sum_{r=i_0}^{j-1}(v_{r+1}-v_r)$ for
$j>i_0$; in either case Cauchy--Schwarz over at most $G$ terms gives
$v_j^{2}\le G\sum_{i=0}^{G-1}(v_{i+1}-v_i)^{2}=Gh\|f\|_{B,h}^{2}$ by
\Cref{prop:finite-brownian-energy}. Summing over the $G+1$ nodes,
\begin{align}
\|\mathbf v\|_2\le\sqrt{G(G+1)h}\,\|f\|_{B,h},
\qquad\text{and conversely}\qquad
\|f\|_{B,h}\le\frac{2}{\sqrt h}\|\mathbf v\|_2,
\label{eq:norm-equivalence}
\end{align}
the second bound from $(v_{i+1}-v_i)^2\le2(v_{i+1}^2+v_i^2)$. Thus $\|\cdot\|_{B,h}$ and the Euclidean
norm on the coefficient vector are equivalent, so $(\mathcal H_h^0,\langle\cdot,\cdot\rangle_{B,h})$ is a
finite-dimensional inner-product space and hence complete: a $\|\cdot\|_{B,h}$-Cauchy sequence has
Cauchy coefficient vectors by \eqref{eq:norm-equivalence}, and the limit vector reconstructs the limit
function, which lies in $\mathcal H_h^0$ because $v_{i_0}=0$ is preserved.

For $t\in I$ the evaluation functional $\delta_t(f)=f(t)$ is linear, and writing
$f(t)=f(t)-f(0)=\int_{0}^{t}f'$ and applying Cauchy--Schwarz on an interval of length at most $A$,
\begin{align}
|\delta_t(f)|=\left|\int_{0}^{t}f'\right|\le\sqrt{|t|}\left(\int_{I}|f'|^2\right)^{1/2}\le\sqrt A\,\|f\|_{B,h}.
\end{align}
Hence $\|\delta_t\|\le\sqrt A$, uniformly in $h$: the bound does not degrade under mesh refinement, which
is the quantitative form of the isometric embedding of \Cref{cor:conforming-subspace}.
}

Since
$\mathcal H_h^0$
is a real Hilbert space and
$\delta_t$
is a bounded linear functional, the Riesz representation theorem
provides a unique function
$k_t\in\mathcal H_h^0$
such that
\begin{align}
\delta_t(f)
&\stackrel{(a)}{=}
\left\langle
f,k_t
\right\rangle_{B,h}
\qquad
\text{for every }
f\in\mathcal H_h^0.
\label{eq:finite-rkhs-riesz-representer}
\end{align}
Here (a) is the Riesz representation theorem applied to the bounded linear
functional
$\delta_t$.
Define
\begin{align}
k_h^0:I\times I
&\longrightarrow
\mathbb R,
&
k_h^0(s,t)
&:=
k_t(s).
\label{eq:finite-rkhs-kernel-definition}
\end{align}
For every
$f\in\mathcal H_h^0$
and
$t\in I$,
\begin{align}
f(t)
&\stackrel{(a)}{=}
\delta_t(f)
\stackrel{(b)}{=}
\left\langle
f,k_t
\right\rangle_{B,h}
\stackrel{(c)}{=}
\left\langle
f,k_h^0(\cdot,t)
\right\rangle_{B,h}.
\label{eq:finite-rkhs-reproducing-identity}
\end{align}
Here (a) is the definition of point evaluation;
(b) is
\eqref{eq:finite-rkhs-riesz-representer};
and (c) uses
$k_h^0(\cdot,t)=k_t$
from
\eqref{eq:finite-rkhs-kernel-definition}.
Thus
$k_h^0$
reproduces every point evaluation.
For completeness, the kernel is symmetric because, for arbitrary
$s,t\in I$,
\begin{align}
k_h^0(s,t)
&\stackrel{(a)}{=}
k_t(s)
\stackrel{(b)}{=}
\left\langle
k_t,k_s
\right\rangle_{B,h}
\stackrel{(c)}{=}
\left\langle
k_s,k_t
\right\rangle_{B,h}
\stackrel{(d)}{=}
k_s(t)
\stackrel{(e)}{=}
k_h^0(t,s).
\label{eq:finite-rkhs-kernel-symmetry}
\end{align}
Here (a) and (e) use the kernel definition;
(b) applies
\eqref{eq:finite-rkhs-riesz-representer}
with the function
$f=k_t$
and evaluation point
$s$;
(c) uses symmetry of the inner product; and
(d) applies
\eqref{eq:finite-rkhs-riesz-representer}
with the function
$f=k_s$
and evaluation point
$t$.
Moreover, for arbitrary
$N\geq1$,
points
$t_1,\ldots,t_N\in I$,
and coefficients
$a_1,\ldots,a_N\in\mathbb R$,
\begin{align}
\sum_{r=1}^{N}
\sum_{s=1}^{N}
a_ra_s
k_h^0(t_r,t_s)
&\stackrel{(a)}{=}
\sum_{r=1}^{N}
\sum_{s=1}^{N}
a_ra_s
\left\langle
k_{t_s},k_{t_r}
\right\rangle_{B,h}
\stackrel{(b)}{=}
\left\langle
\sum_{s=1}^{N}
a_s k_{t_s},
\sum_{r=1}^{N}
a_r k_{t_r}
\right\rangle_{B,h}
\nonumber\\
&\stackrel{(c)}{=}
\left\|
\sum_{r=1}^{N}
a_r k_{t_r}
\right\|_{B,h}^{2}
\stackrel{(d)}{\geq}
0.
\label{eq:finite-rkhs-kernel-positive-semidefinite}
\end{align}
Here (a) uses
$k_h^0(t_r,t_s)=k_{t_s}(t_r)=\langle k_{t_s},k_{t_r}\rangle_{B,h}$,
which follows from
\eqref{eq:finite-rkhs-kernel-definition}
and
\eqref{eq:finite-rkhs-riesz-representer};
(b) uses bilinearity of the inner product and the finiteness of the two sums;
(c) observes that the two displayed sums represent the same function after
renaming the dummy index; and
(d) uses nonnegativity of the squared norm.
Therefore
$k_h^0$
is a positive-semidefinite reproducing kernel for
$\mathcal H_h^0$.
Combining this reproducing property with
\eqref{eq:finite-rkhs-dimension}
and the completeness proved above shows that
\begin{align}
\left(
\mathcal H_h^0,
\left\langle
\cdot,\cdot
\right\rangle_{B,h}
\right)
\end{align}
is a finite-dimensional reproducing kernel Hilbert space, which completes the
proof. Items
\ref{thm:finite-brownian-rkhs-i}--\ref{thm:finite-brownian-rkhs-ii}
are established in \Cref{thm:items}.

\end{noheadproof}\hfill\ensuremath{\square}

\subsection{Proof of
\texorpdfstring{\Cref{thm:equivalence-of-implementations}}{Theorem 2}}
\label{app:coordinate-equivalence}

{The increment-energy identity used in
part~\ref{thm:equivalence-of-implementations-iii} is \Cref{lem:increment-energy}, and the positivity of
$\boldsymbol\Lambda$ used in the spectral part is \Cref{lem:K0-pd} with \Cref{lem:spectral-energy}. The
derivation below is self-contained, but those lemmas may be substituted for the corresponding steps.}

\begin{noheadproof}

Let
$f\in\mathcal H_h^0$
be arbitrary. By
Definition~\ref{def:anchored-finite-profile-space},
$G$ is even. Write
$G=2m$
and
$i_0=m$.
Since
$f\in\mathcal H_h^0\subseteq\mathcal H_h$,
Definition~\ref{def:finite-brownian-profile-space} yields a unique vector
$\mathbf v=(v_0,\ldots,v_G)^{\top}\in\mathbb R^{G+1}$ such that
\begin{align}
f
&\stackrel{(a)}{=}
\mathcal R(\mathbf v)
\stackrel{(b)}{=}
\sum_{r=0}^{G}v_r\phi_r.
\label{eq:equiv-proof-full-nodal-representation}
\end{align}
Here (a) is the unique finite-element representation of $f$, and (b) is the
definition of the reconstruction operator $\mathcal R$.
For every grid index $j=0,\ldots,G$, the nodal property of the basis gives
\begin{align}
f(t_j)
&\stackrel{(a)}{=}
\sum_{r=0}^{G}v_r\phi_r(t_j)
\stackrel{(b)}{=}
\sum_{r=0}^{G}v_r\delta_{rj}
\stackrel{(c)}{=}
v_j.
\label{eq:equiv-proof-nodal-values}
\end{align}
Here (a) evaluates
\eqref{eq:equiv-proof-full-nodal-representation}
at $t_j$, (b) uses
$\phi_r(t_j)=\delta_{rj}$, and (c) uses the defining property of the
Kronecker delta. Since
$t_{i_0}=0$
and
$f(0)=0$, it follows that
\begin{align}
v_{i_0}
&\stackrel{(a)}{=}
f(t_{i_0})
\stackrel{(b)}{=}
f(0)
\stackrel{(c)}{=}
0.
\label{eq:equiv-proof-anchor-component}
\end{align}
Here (a) applies
\eqref{eq:equiv-proof-nodal-values}, (b) uses
$t_{i_0}=0$, and (c) uses
$f\in\mathcal H_h^0$.

Let
$\widetilde{\mathbf v}\in\mathbb R^G$
be the reduced anchored vector ordered as in
\eqref{eq:reduced-anchored-ordering}. By the definition of the anchoring
reconstruction matrix in
\eqref{eq:anchoring-reconstruction-matrix},
\begin{align}
\mathbf v
&\stackrel{(a)}{=}
\mathbf R\widetilde{\mathbf v}.
\label{eq:equiv-proof-anchor-reconstruction}
\end{align}
Equality (a) inserts the zero component at index $i_0$ and places every
remaining reduced coordinate at its corresponding nodal index. Define
\begin{align}
\mathbf D_0
&:=
\mathbf D\mathbf R
\in
\mathbb R^{G\times G}.
\label{eq:equiv-proof-D0-definition}
\end{align}
Then the increment coordinates satisfy
\begin{align}
\mathbf{w}
&\stackrel{(a)}{=}
\mathbf D\mathbf v
\stackrel{(b)}{=}
\mathbf D\mathbf R\widetilde{\mathbf v}
\stackrel{(c)}{=}
\mathbf D_0\widetilde{\mathbf v}.
\label{eq:equiv-proof-increment-definition}
\end{align}
Here (a) is Definition~\ref{def:increment-coordinates}, (b) substitutes
\eqref{eq:equiv-proof-anchor-reconstruction}, and (c) uses
\eqref{eq:equiv-proof-D0-definition}. The spectral coordinates satisfy
\begin{align}
\mathbf c
&\stackrel{(a)}{=}
\mathbf Q^{\top}\widetilde{\mathbf v}.
\label{eq:equiv-proof-spectral-definition}
\end{align}
Equality (a) is Definition~\ref{def:spectral-coordinates}.

We first establish
\ref{thm:equivalence-of-implementations-ii}, because the inverse coordinate
maps will then make the assertion in
\ref{thm:equivalence-of-implementations-i} explicit.
For arbitrary
$\mathbf x,\mathbf y\in\mathbb R^G$
and
$\alpha,\beta\in\mathbb R$,
\begin{align}
\mathbf D_0
\left(
\alpha\mathbf x+
\beta\mathbf y
\right)
&\stackrel{(a)}{=}
\mathbf D\mathbf R
\left(
\alpha\mathbf x+
\beta\mathbf y
\right)
\nonumber\\
&\stackrel{(b)}{=}
\mathbf D
\left(
\alpha\mathbf R\mathbf x+
\beta\mathbf R\mathbf y
\right)
\nonumber\\
&\stackrel{(c)}{=}
\alpha\mathbf D\mathbf R\mathbf x+
\beta\mathbf D\mathbf R\mathbf y
\stackrel{(d)}{=}
\alpha\mathbf D_0\mathbf x+
\beta\mathbf D_0\mathbf y.
\label{eq:equiv-proof-D0-linearity}
\end{align}
Here (a) substitutes
$\mathbf D_0=\mathbf D\mathbf R$, (b) uses the linearity of multiplication by
$\mathbf R$, (c) uses the linearity of multiplication by $\mathbf D$, and
(d) again uses the definition of $\mathbf D_0$. Thus
$\mathbf x\mapsto\mathbf D_0\mathbf x$
is linear.

We next prove injectivity. Let
$\mathbf w\in\mathbb R^G$
satisfy
\begin{align}
\mathbf D_0\mathbf w
&=
\mathbf0.
\label{eq:equiv-proof-D0-kernel-assumption}
\end{align}
Set
$\mathbf z:=\mathbf R\mathbf w\in\mathbb R^{G+1}$.
Then
\begin{align}
\mathbf D\mathbf z
&\stackrel{(a)}{=}
\mathbf D\mathbf R\mathbf w
\stackrel{(b)}{=}
\mathbf D_0\mathbf w
\stackrel{(c)}{=}
\mathbf0.
\label{eq:equiv-proof-Dz-zero}
\end{align}
Here (a) substitutes
$\mathbf z=\mathbf R\mathbf w$, (b) uses
$\mathbf D_0=\mathbf D\mathbf R$, and (c) applies
\eqref{eq:equiv-proof-D0-kernel-assumption}. By the componentwise definition of
$\mathbf D$, for every
$i=0,\ldots,G-1$,
\begin{align}
0
&\stackrel{(a)}{=}
(\mathbf D\mathbf z)_i
\stackrel{(b)}{=}
z_{i+1}-z_i,
\end{align}
where (a) follows from
\eqref{eq:equiv-proof-Dz-zero}, and (b) is the definition of the first-difference
operator. Therefore
$z_{i+1}=z_i$
for every
$i=0,\ldots,G-1$. Repeated application of these equalities gives
\begin{align}
z_j
&\stackrel{(a)}{=}
z_0,
\qquad
j=0,\ldots,G.
\label{eq:equiv-proof-z-constant}
\end{align}
Equality (a) follows by induction on $j$: it is immediate for $j=0$, and if
$z_j=z_0$, then
$z_{j+1}=z_j=z_0$.

The row of $\mathbf R$ corresponding to the anchor index $i_0$ is zero because
none of the columns of $\mathbf R$ equals the omitted canonical vector
$\mathbf e_{i_0}$. Hence
\begin{align}
z_{i_0}
&\stackrel{(a)}{=}
(\mathbf R\mathbf w)_{i_0}
\stackrel{(b)}{=}
0.
\label{eq:equiv-proof-z-anchor-zero}
\end{align}
Here (a) uses
$\mathbf z=\mathbf R\mathbf w$, and (b) uses the zero anchor row of
$\mathbf R$. Combining
\eqref{eq:equiv-proof-z-constant}
and
\eqref{eq:equiv-proof-z-anchor-zero}
gives, for every $j$,
\begin{align}
z_j
&\stackrel{(a)}{=}
z_{i_0}
\stackrel{(b)}{=}
0,
\end{align}
where (a) uses the constancy of $\mathbf z$, and (b) uses the anchored value.
Thus
$\mathbf z=\mathbf0$.

The columns of $\mathbf R$ are pairwise distinct canonical basis vectors of
$\mathbb R^{G+1}$. Consequently,
\begin{align}
\mathbf R^{\top}\mathbf R
&\stackrel{(a)}{=}
\mathbf I_G.
\label{eq:equiv-proof-RtR}
\end{align}
Indeed, the $(r,s)$ entry of
$\mathbf R^{\top}\mathbf R$
is the Euclidean inner product of columns $r$ and $s$ of $\mathbf R$, which is
$1$ when $r=s$ and $0$ otherwise. It follows that
\begin{align}
\mathbf w
&\stackrel{(a)}{=}
\mathbf I_G\mathbf w
\stackrel{(b)}{=}
\mathbf R^{\top}\mathbf R\mathbf w
\stackrel{(c)}{=}
\mathbf R^{\top}\mathbf z
\stackrel{(d)}{=}
\mathbf0.
\label{eq:equiv-proof-D0-injective}
\end{align}
Here (a) uses the identity matrix, (b) applies
\eqref{eq:equiv-proof-RtR}, (c) uses
$\mathbf z=\mathbf R\mathbf w$, and (d) uses
$\mathbf z=\mathbf0$. Therefore
$\ker(\mathbf D_0)=\{\mathbf0\}$,
so the increment map is injective.

We now prove surjectivity directly. Let
$\mathbf d=(d_0,\ldots,d_{G-1})^{\top}\in\mathbb R^G$
be arbitrary, and define
$\mathbf u(\mathbf d)=(u_0,\ldots,u_G)^{\top}\in\mathbb R^{G+1}$
by
\begin{align}
u_j
:={}
\begin{cases}
-\displaystyle\sum_{r=j}^{i_0-1}d_r,
&
0\leq j\leq i_0,
\\[2mm]
\displaystyle\sum_{r=i_0}^{j-1}d_r,
&
i_0\leq j\leq G.
\end{cases}
\label{eq:equiv-proof-cumulative-reconstruction}
\end{align}
The two formulas agree at $j=i_0$ because both sums are empty, and every empty
sum is understood to equal zero. In particular,
\begin{align}
u_{i_0}
&\stackrel{(a)}{=}
0.
\label{eq:equiv-proof-cumulative-anchor}
\end{align}
Equality (a) follows from the empty-sum convention.
For
$i=0,\ldots,i_0-1$,
\begin{align}
u_{i+1}-u_i
&\stackrel{(a)}{=}
-\sum_{r=i+1}^{i_0-1}d_r
+
\sum_{r=i}^{i_0-1}d_r
\stackrel{(b)}{=}
d_i.
\label{eq:equiv-proof-left-cumulative-difference}
\end{align}
Here (a) applies the first branch of
\eqref{eq:equiv-proof-cumulative-reconstruction}, including its empty-sum value
when $i=i_0-1$, and (b) cancels the common terms
$d_{i+1},\ldots,d_{i_0-1}$, leaving only $d_i$. For
$i=i_0,\ldots,G-1$,
\begin{align}
u_{i+1}-u_i
&\stackrel{(a)}{=}
\sum_{r=i_0}^{i}d_r
-
\sum_{r=i_0}^{i-1}d_r
\stackrel{(b)}{=}
d_i.
\label{eq:equiv-proof-right-cumulative-difference}
\end{align}
Here (a) applies the second branch of
\eqref{eq:equiv-proof-cumulative-reconstruction}, including the empty second
sum when $i=i_0$, and (b) cancels the common terms
$d_{i_0},\ldots,d_{i-1}$, leaving only $d_i$. Hence, for every
$i=0,\ldots,G-1$,
\begin{align}
(\mathbf D\mathbf u(\mathbf d))_i
&\stackrel{(a)}{=}
u_{i+1}-u_i
\stackrel{(b)}{=}
d_i,
\end{align}
where (a) is the definition of $\mathbf D$, while (b) applies
\eqref{eq:equiv-proof-left-cumulative-difference}
when $i<i_0$ and
\eqref{eq:equiv-proof-right-cumulative-difference}
when $i\geq i_0$. Therefore
\begin{align}
\mathbf D\mathbf u(\mathbf d)
&\stackrel{(a)}{=}
\mathbf d.
\label{eq:equiv-proof-cumulative-Du}
\end{align}
Equality (a) follows because the two vectors have equal components at every
index.

Because the columns of $\mathbf R$ contain every canonical basis vector except
$\mathbf e_{i_0}$, one also has
\begin{align}
\mathbf R\mathbf R^{\top}
&\stackrel{(a)}{=}
\sum_{j\neq i_0}
\mathbf e_j\mathbf e_j^{\top}
\stackrel{(b)}{=}
\mathbf I_{G+1}
-
\mathbf e_{i_0}\mathbf e_{i_0}^{\top}.
\label{eq:equiv-proof-RRt}
\end{align}
Here (a) expands the product as the sum of the outer products of the columns of
$\mathbf R$, and (b) removes the single omitted canonical projector from the
canonical decomposition of the identity. Define
\begin{align}
\mathbf x(\mathbf d)
&:=
\mathbf R^{\top}\mathbf u(\mathbf d)
\in
\mathbb R^G.
\label{eq:equiv-proof-explicit-D0-inverse-vector}
\end{align}
Then
\begin{align}
\mathbf R\mathbf x(\mathbf d)
&\stackrel{(a)}{=}
\mathbf R\mathbf R^{\top}\mathbf u(\mathbf d)
\nonumber\\
&\stackrel{(b)}{=}
\left(
\mathbf I_{G+1}
-
\mathbf e_{i_0}\mathbf e_{i_0}^{\top}
\right)
\mathbf u(\mathbf d)
\nonumber\\
&\stackrel{(c)}{=}
\mathbf u(\mathbf d)
-
\mathbf e_{i_0}u_{i_0}
\stackrel{(d)}{=}
\mathbf u(\mathbf d).
\label{eq:equiv-proof-reconstruct-cumulative-vector}
\end{align}
Here (a) substitutes
\eqref{eq:equiv-proof-explicit-D0-inverse-vector}, (b) uses
\eqref{eq:equiv-proof-RRt}, (c) evaluates
$\mathbf e_{i_0}^{\top}\mathbf u(\mathbf d)=u_{i_0}$, and (d) uses
\eqref{eq:equiv-proof-cumulative-anchor}. Consequently,
\begin{align}
\mathbf D_0\mathbf x(\mathbf d)
&\stackrel{(a)}{=}
\mathbf D\mathbf R\mathbf x(\mathbf d)
\stackrel{(b)}{=}
\mathbf D\mathbf u(\mathbf d)
\stackrel{(c)}{=}
\mathbf d.
\label{eq:equiv-proof-D0-surjective}
\end{align}
Here (a) uses
$\mathbf D_0=\mathbf D\mathbf R$, (b) applies
\eqref{eq:equiv-proof-reconstruct-cumulative-vector}, and (c) applies
\eqref{eq:equiv-proof-cumulative-Du}. Since $\mathbf d\in\mathbb R^G$ was
arbitrary, $\mathbf D_0$ is surjective. Together with injectivity, this proves
that $\mathbf D_0$ is invertible. Moreover, the preceding construction gives
the explicit inverse
\begin{align}
\mathbf D_0^{-1}\mathbf d
&\stackrel{(a)}{=}
\mathbf R^{\top}\mathbf u(\mathbf d),
\label{eq:equiv-proof-explicit-D0-inverse}
\end{align}
where (a) follows from uniqueness of the preimage under the injective map
$\mathbf D_0$ and
\eqref{eq:equiv-proof-D0-surjective}.

We next treat the spectral map. For arbitrary
$\mathbf x,\mathbf y\in\mathbb R^G$
and
$\alpha,\beta\in\mathbb R$,
\begin{align}
\mathbf Q^{\top}
\left(
\alpha\mathbf x+
\beta\mathbf y
\right)
&\stackrel{(a)}{=}
\alpha\mathbf Q^{\top}\mathbf x+
\beta\mathbf Q^{\top}\mathbf y.
\label{eq:equiv-proof-spectral-linearity}
\end{align}
Equality (a) is the distributivity and homogeneity of matrix multiplication, so
the spectral map is linear. Since $\mathbf Q$ is orthogonal,
\begin{align}
\mathbf Q^{\top}\mathbf Q
&\stackrel{(a)}{=}
\mathbf I_G,
\qquad
\mathbf Q\mathbf Q^{\top}
\stackrel{(b)}{=}
\mathbf I_G.
\label{eq:equiv-proof-Q-orthogonality}
\end{align}
Equalities (a) and (b) are the two defining inverse identities for an
orthogonal square matrix. Thus, for every
$\mathbf x,\mathbf a\in\mathbb R^G$,
\begin{align}
\mathbf Q
\left(
\mathbf Q^{\top}\mathbf x
\right)
&\stackrel{(a)}{=}
\left(
\mathbf Q\mathbf Q^{\top}
\right)
\mathbf x
\stackrel{(b)}{=}
\mathbf I_G\mathbf x
\stackrel{(c)}{=}
\mathbf x,
\label{eq:equiv-proof-Q-left-inverse}
\\
\mathbf Q^{\top}
\left(
\mathbf Q\mathbf a
\right)
&\stackrel{(a)}{=}
\left(
\mathbf Q^{\top}\mathbf Q
\right)
\mathbf a
\stackrel{(b)}{=}
\mathbf I_G\mathbf a
\stackrel{(c)}{=}
\mathbf a.
\label{eq:equiv-proof-Q-right-inverse}
\end{align}
In each line, (a) uses associativity, (b) applies
\eqref{eq:equiv-proof-Q-orthogonality}, and (c) uses the identity matrix.
Therefore
$\mathbf x\mapsto\mathbf Q^{\top}\mathbf x$
is invertible, with inverse
$\mathbf a\mapsto\mathbf Q\mathbf a$.

The complete coordinate conversion formulas are now
\begin{align}
\mathbf{w}
&\stackrel{(a)}{=}
\mathbf D_0\widetilde{\mathbf v},
&
\mathbf c
&\stackrel{(b)}{=}
\mathbf Q^{\top}\widetilde{\mathbf v},
\nonumber\\
\widetilde{\mathbf v}
&\stackrel{(c)}{=}
\mathbf D_0^{-1}\mathbf{w},
&
\widetilde{\mathbf v}
&\stackrel{(d)}{=}
\mathbf Q\mathbf c,
\nonumber\\
\mathbf c
&\stackrel{(e)}{=}
\mathbf Q^{\top}\mathbf D_0^{-1}\mathbf{w},
&
\mathbf{w}
&\stackrel{(f)}{=}
\mathbf D_0\mathbf Q\mathbf c.
\label{eq:equiv-proof-all-coordinate-conversions}
\end{align}
Here (a) and (b) are the coordinate definitions, (c) uses the inverse of
$\mathbf D_0$, (d) uses the inverse of $\mathbf Q^{\top}$, (e) substitutes
(c) into (b), and (f) substitutes (d) into (a). Hence each coordinate system
uniquely determines the other two, proving
\ref{thm:equivalence-of-implementations-ii}.

We now prove
\ref{thm:equivalence-of-implementations-i}.
Define the three reconstruction maps by
\begin{align}
\Phi_{\mathrm{nod}}(\mathbf x)
&:=
\mathcal R(\mathbf R\mathbf x),
\label{eq:equiv-proof-nodal-reconstruction-map}
\\
\Phi_{\mathrm{inc}}(\mathbf d)
&:=
\mathcal R
\left(
\mathbf R\mathbf D_0^{-1}\mathbf d
\right),
\label{eq:equiv-proof-increment-reconstruction-map}
\\
\Phi_{\mathrm{spec}}(\mathbf a)
&:=
\mathcal R
\left(
\mathbf R\mathbf Q\mathbf a
\right).
\label{eq:equiv-proof-spectral-reconstruction-map}
\end{align}
Each map takes values in $\mathcal H_h^0$. Indeed, for every
$\mathbf x\in\mathbb R^G$,
\begin{align}
\Phi_{\mathrm{nod}}(\mathbf x)(0)
&\stackrel{(a)}{=}
\sum_{j=0}^{G}
(\mathbf R\mathbf x)_j\phi_j(0)
\stackrel{(b)}{=}
\sum_{j=0}^{G}
(\mathbf R\mathbf x)_j\delta_{ji_0}
\nonumber\\
&\stackrel{(c)}{=}
(\mathbf R\mathbf x)_{i_0}
\stackrel{(d)}{=}
0.
\label{eq:equiv-proof-reconstruction-anchor}
\end{align}
Here (a) expands
$\mathcal R(\mathbf R\mathbf x)$, (b) uses
$0=t_{i_0}$ and
$\phi_j(t_{i_0})=\delta_{ji_0}$, (c) applies the Kronecker delta, and (d) uses
the zero anchor row of $\mathbf R$. Thus
$\Phi_{\mathrm{nod}}(\mathbf x)\in\mathcal H_h^0$.
The increment and spectral maps are compositions of
$\Phi_{\mathrm{nod}}$
with vectors in $\mathbb R^G$, so they also take values in
$\mathcal H_h^0$.

For the coordinates associated with the fixed function $f$,
\begin{align}
\Phi_{\mathrm{nod}}
\left(
\widetilde{\mathbf v}
\right)
&\stackrel{(a)}{=}
\mathcal R
\left(
\mathbf R\widetilde{\mathbf v}
\right)
\stackrel{(b)}{=}
\mathcal R(\mathbf v)
\stackrel{(c)}{=}
f.
\label{eq:equiv-proof-nodal-same-function}
\end{align}
Here (a) uses
\eqref{eq:equiv-proof-nodal-reconstruction-map}, (b) uses
\eqref{eq:equiv-proof-anchor-reconstruction}, and (c) uses
\eqref{eq:equiv-proof-full-nodal-representation}. Likewise,
\begin{align}
\Phi_{\mathrm{inc}}
\left(
\mathbf{w}
\right)
&\stackrel{(a)}{=}
\mathcal R
\left(
\mathbf R\mathbf D_0^{-1}
\mathbf{w}
\right)
\stackrel{(b)}{=}
\mathcal R
\left(
\mathbf R\mathbf D_0^{-1}
\mathbf D_0\widetilde{\mathbf v}
\right)
\stackrel{(c)}{=}
\mathcal R
\left(
\mathbf R\mathbf I_G\widetilde{\mathbf v}
\right)
\stackrel{(d)}{=}
\mathcal R
\left(
\mathbf R\widetilde{\mathbf v}
\right)
\stackrel{(e)}{=}
f.
\label{eq:equiv-proof-increment-same-function}
\end{align}
Here (a) uses
\eqref{eq:equiv-proof-increment-reconstruction-map}, (b) substitutes
$\mathbf{w}=\mathbf D_0\widetilde{\mathbf v}$, (c) uses
$\mathbf D_0^{-1}\mathbf D_0=\mathbf I_G$, (d) uses the identity matrix, and
(e) applies
\eqref{eq:equiv-proof-nodal-same-function}. Finally,
\begin{align}
\Phi_{\mathrm{spec}}
\left(
\mathbf c
\right)
&\stackrel{(a)}{=}
\mathcal R
\left(
\mathbf R\mathbf Q\mathbf c
\right)
\stackrel{(b)}{=}
\mathcal R
\left(
\mathbf R\mathbf Q\mathbf Q^{\top}
\widetilde{\mathbf v}
\right)
\stackrel{(c)}{=}
\mathcal R
\left(
\mathbf R\mathbf I_G\widetilde{\mathbf v}
\right)
\stackrel{(d)}{=}
\mathcal R
\left(
\mathbf R\widetilde{\mathbf v}
\right)
\stackrel{(e)}{=}
f.
\label{eq:equiv-proof-spectral-same-function}
\end{align}
Here (a) uses
\eqref{eq:equiv-proof-spectral-reconstruction-map}, (b) substitutes
$\mathbf c=\mathbf Q^{\top}\widetilde{\mathbf v}$, (c) uses
$\mathbf Q\mathbf Q^{\top}=\mathbf I_G$, (d) uses the identity matrix, and
(e) applies
\eqref{eq:equiv-proof-nodal-same-function}. Therefore the three coordinate
representations reconstruct exactly the same function $f$, proving
\ref{thm:equivalence-of-implementations-i}.

We next prove
\ref{thm:equivalence-of-implementations-iii}.
By Proposition~\ref{prop:finite-brownian-energy},
\begin{align}
\left\|f\right\|_{B,h}^{2}
&\stackrel{(a)}{=}
\mathbf v^{\top}\mathbf K\mathbf v
\stackrel{(b)}{=}
\frac{1}{h}
\left(
\mathbf D\mathbf v
\right)^{\top}
\left(
\mathbf D\mathbf v
\right)
\nonumber\\
&\stackrel{(c)}{=}
\frac{1}{h}
\mathbf{w}^{\top}\mathbf{w}
\stackrel{(d)}{=}
\frac{1}{h}
\left\|\mathbf{w}\right\|_2^2.
\label{eq:equiv-proof-nodal-increment-energy}
\end{align}
Here (a) is the stiffness representation of the Brownian energy, (b) uses
$\mathbf K=h^{-1}\mathbf D^{\top}\mathbf D$
and associativity, (c) uses
$\mathbf{w}=\mathbf D\mathbf v$, and (d) is the definition of the
Euclidean norm.

For the spectral representation, first
\begin{align}
\left\|f\right\|_{B,h}^{2}
&\stackrel{(a)}{=}
\left(
\mathbf R\widetilde{\mathbf v}
\right)^{\top}
\mathbf K
\left(
\mathbf R\widetilde{\mathbf v}
\right)
\stackrel{(b)}{=}
\widetilde{\mathbf v}^{\top}
\mathbf R^{\top}\mathbf K\mathbf R
\widetilde{\mathbf v}
\stackrel{(c)}{=}
\widetilde{\mathbf v}^{\top}
\mathbf K_0
\widetilde{\mathbf v}.
\label{eq:equiv-proof-reduced-energy}
\end{align}
Here (a) substitutes
$\mathbf v=\mathbf R\widetilde{\mathbf v}$
into the first equality of
\eqref{eq:equiv-proof-nodal-increment-energy}, (b) uses
$(\mathbf R\widetilde{\mathbf v})^{\top}
=\widetilde{\mathbf v}^{\top}\mathbf R^{\top}$
and associativity, and (c) uses
$\mathbf K_0=\mathbf R^{\top}\mathbf K\mathbf R$.
By Definition~\ref{def:spectral-coordinates},
\begin{align}
\widetilde{\mathbf v}
&\stackrel{(a)}{=}
\mathbf Q\mathbf c,
\qquad
\mathbf K_0
\stackrel{(b)}{=}
\mathbf Q\boldsymbol{\Lambda}\mathbf Q^{\top}.
\label{eq:equiv-proof-spectral-inverse-and-decomposition}
\end{align}
Here (a) is the inverse spectral transformation proved above, and (b) is the
eigendecomposition of the anchored stiffness matrix, whose existence with strictly
positive spectrum is {established in \Cref{lem:K0-pd} and
\Cref{lem:spectral-energy}, rather than assumed in
Definition~\ref{def:spectral-coordinates}}. Consequently,
\begin{align}
\left\|f\right\|_{B,h}^{2}
&\stackrel{(a)}{=}
\left(
\mathbf Q\mathbf c
\right)^{\top}
\left(
\mathbf Q\boldsymbol{\Lambda}\mathbf Q^{\top}
\right)
\left(
\mathbf Q\mathbf c
\right)
\stackrel{(b)}{=}
\mathbf c^{\top}
\mathbf Q^{\top}\mathbf Q
\boldsymbol{\Lambda}
\mathbf Q^{\top}\mathbf Q
\mathbf c
\stackrel{(c)}{=}
\mathbf c^{\top}
\mathbf I_G
\boldsymbol{\Lambda}
\mathbf I_G
\mathbf c
\stackrel{(d)}{=}
\mathbf c^{\top}
\boldsymbol{\Lambda}
\mathbf c.
\label{eq:equiv-proof-spectral-energy}
\end{align}
Here (a) substitutes both identities in
\eqref{eq:equiv-proof-spectral-inverse-and-decomposition}
into
\eqref{eq:equiv-proof-reduced-energy}, (b) uses
$(\mathbf Q\mathbf c)^{\top}=\mathbf c^{\top}\mathbf Q^{\top}$
and associativity, (c) uses
$\mathbf Q^{\top}\mathbf Q=\mathbf I_G$, and (d) uses the identity matrix.
Combining
\eqref{eq:equiv-proof-nodal-increment-energy}
and
\eqref{eq:equiv-proof-spectral-energy}
gives
\begin{align}
\left\|f\right\|_{B,h}^{2}
&\stackrel{(a)}{=}
\mathbf v^{\top}\mathbf K\mathbf v
\stackrel{(b)}{=}
\frac{1}{h}
\left\|\mathbf{w}\right\|_2^2
\stackrel{(c)}{=}
\mathbf c^{\top}\boldsymbol{\Lambda}\mathbf c.
\end{align}
Here (a) and (b) are the nodal and increment equalities in
\eqref{eq:equiv-proof-nodal-increment-energy}, while (c) uses
\eqref{eq:equiv-proof-spectral-energy}. This proves
\ref{thm:equivalence-of-implementations-iii}.

We now establish
\ref{thm:equivalence-of-implementations-iv}.
Define the ranges of the three reconstruction maps by
\begin{align}
\mathcal F_{\mathrm{nod}}
&:=
\Phi_{\mathrm{nod}}(\mathbb R^G),
&
\mathcal F_{\mathrm{inc}}
&:=
\Phi_{\mathrm{inc}}(\mathbb R^G),
&
\mathcal F_{\mathrm{spec}}
&:=
\Phi_{\mathrm{spec}}(\mathbb R^G).
\label{eq:equiv-proof-three-ranges}
\end{align}
Equation
\eqref{eq:equiv-proof-reconstruction-anchor}
shows that
$\mathcal F_{\mathrm{nod}}\subseteq\mathcal H_h^0$.
Conversely, let
$g\in\mathcal H_h^0$
be arbitrary and let
$\mathbf a\in\mathbb R^{G+1}$
be its unique nodal vector. As in
\eqref{eq:equiv-proof-anchor-component},
$a_{i_0}=0$. Set
$\mathbf x:=\mathbf R^{\top}\mathbf a\in\mathbb R^G$.
Then
\begin{align}
\mathbf R\mathbf x
&\stackrel{(a)}{=}
\mathbf R\mathbf R^{\top}\mathbf a
\stackrel{(b)}{=}
\left(
\mathbf I_{G+1}
-
\mathbf e_{i_0}\mathbf e_{i_0}^{\top}
\right)
\mathbf a
\nonumber\\
&\stackrel{(c)}{=}
\mathbf a
-
\mathbf e_{i_0}a_{i_0}
\stackrel{(d)}{=}
\mathbf a.
\label{eq:equiv-proof-any-anchored-vector-reconstructed}
\end{align}
Here (a) uses
$\mathbf x=\mathbf R^{\top}\mathbf a$, (b) applies
\eqref{eq:equiv-proof-RRt}, (c) evaluates the rank-one projector, and (d) uses
$a_{i_0}=0$. Therefore
\begin{align}
g
&\stackrel{(a)}{=}
\mathcal R(\mathbf a)
\stackrel{(b)}{=}
\mathcal R(\mathbf R\mathbf x)
\stackrel{(c)}{=}
\Phi_{\mathrm{nod}}(\mathbf x),
\end{align}
where (a) is the nodal representation of $g$, (b) applies
\eqref{eq:equiv-proof-any-anchored-vector-reconstructed}, and (c) uses
\eqref{eq:equiv-proof-nodal-reconstruction-map}. Hence
$g\in\mathcal F_{\mathrm{nod}}$, so
$\mathcal H_h^0\subseteq\mathcal F_{\mathrm{nod}}$. We have proved
\begin{align}
\mathcal F_{\mathrm{nod}}
&\stackrel{(a)}{=}
\mathcal H_h^0.
\label{eq:equiv-proof-nodal-range}
\end{align}
Equality (a) combines the two set inclusions.

Using the invertibility of $\mathbf D_0$,
\begin{align}
\mathcal F_{\mathrm{inc}}
&\stackrel{(a)}{=}
\left\{
\Phi_{\mathrm{nod}}
\left(
\mathbf D_0^{-1}\mathbf d
\right)
:
\mathbf d\in\mathbb R^G
\right\}
\stackrel{(b)}{=}
\left\{
\Phi_{\mathrm{nod}}(\mathbf x)
:
\mathbf x\in\mathbb R^G
\right\}
\stackrel{(c)}{=}
\mathcal F_{\mathrm{nod}}.
\label{eq:equiv-proof-increment-range}
\end{align}
Here (a) uses
\eqref{eq:equiv-proof-increment-reconstruction-map}, (b) uses the bijective
change of variable
$\mathbf x=\mathbf D_0^{-1}\mathbf d$, and (c) uses the definition of
$\mathcal F_{\mathrm{nod}}$. Similarly, using the orthogonality and hence
surjectivity of $\mathbf Q$,
\begin{align}
\mathcal F_{\mathrm{spec}}
&\stackrel{(a)}{=}
\left\{
\Phi_{\mathrm{nod}}
\left(
\mathbf Q\mathbf a
\right)
:
\mathbf a\in\mathbb R^G
\right\}
\stackrel{(b)}{=}
\left\{
\Phi_{\mathrm{nod}}(\mathbf x)
:
\mathbf x\in\mathbb R^G
\right\}
\stackrel{(c)}{=}
\mathcal F_{\mathrm{nod}}.
\label{eq:equiv-proof-spectral-range}
\end{align}
Here (a) uses
\eqref{eq:equiv-proof-spectral-reconstruction-map}, (b) uses the bijective
change of variable
$\mathbf x=\mathbf Q\mathbf a$, and (c) again uses the definition of
$\mathcal F_{\mathrm{nod}}$. Combining
\eqref{eq:equiv-proof-nodal-range}--\eqref{eq:equiv-proof-spectral-range}
gives
\begin{align}
\mathcal F_{\mathrm{nod}}
&\stackrel{(a)}{=}
\mathcal F_{\mathrm{inc}}
\stackrel{(b)}{=}
\mathcal F_{\mathrm{spec}}
\stackrel{(c)}{=}
\mathcal H_h^0.
\label{eq:equiv-proof-common-function-space}
\end{align}
Here (a) is
\eqref{eq:equiv-proof-increment-range}, (b) is
\eqref{eq:equiv-proof-spectral-range}, and (c) is
\eqref{eq:equiv-proof-nodal-range}.

For completeness, we also verify that the coordinate realizations pull back the
same Brownian inner product, not merely the same norm. Let
$\mathbf x,\mathbf y\in\mathbb R^G$, define
\begin{align}
g
&:=
\Phi_{\mathrm{nod}}(\mathbf x),
&
h
&:=
\Phi_{\mathrm{nod}}(\mathbf y),
\nonumber\\
\boldsymbol{\alpha}
&:=
\mathbf D_0\mathbf x,
&
\boldsymbol{\beta}
&:=
\mathbf D_0\mathbf y,
\nonumber\\
\mathbf p
&:=
\mathbf Q^{\top}\mathbf x,
&
\mathbf q
&:=
\mathbf Q^{\top}\mathbf y.
\label{eq:equiv-proof-two-function-coordinates}
\end{align}
The bilinear stiffness identity gives
\begin{align}
\left\langle g,h\right\rangle_{B,h}
&\stackrel{(a)}{=}
\left(
\mathbf R\mathbf x
\right)^{\top}
\mathbf K
\left(
\mathbf R\mathbf y
\right)
\stackrel{(b)}{=}
\mathbf x^{\top}
\mathbf K_0
\mathbf y.
\label{eq:equiv-proof-bilinear-nodal}
\end{align}
Here (a) follows by expanding both reconstructed functions in the hat basis and
using
$K_{rs}=\langle\phi_r,\phi_s\rangle_{B,h}$, while (b) uses
$\mathbf K_0=\mathbf R^{\top}\mathbf K\mathbf R$.
Moreover,
\begin{align}
\mathbf x^{\top}\mathbf K_0\mathbf y
&\stackrel{(a)}{=}
\mathbf x^{\top}
\mathbf R^{\top}\mathbf K\mathbf R
\mathbf y
\stackrel{(b)}{=}
\frac{1}{h}
\mathbf x^{\top}
\mathbf R^{\top}\mathbf D^{\top}
\mathbf D\mathbf R
\mathbf y
\stackrel{(c)}{=}
\frac{1}{h}
\left(
\mathbf D_0\mathbf x
\right)^{\top}
\left(
\mathbf D_0\mathbf y
\right)
\stackrel{(d)}{=}
\frac{1}{h}
\boldsymbol{\alpha}^{\top}\boldsymbol{\beta}.
\label{eq:equiv-proof-bilinear-increment}
\end{align}
Here (a) uses the definition of $\mathbf K_0$, (b) substitutes
$\mathbf K=h^{-1}\mathbf D^{\top}\mathbf D$, (c) uses
$\mathbf D_0=\mathbf D\mathbf R$
and the transpose identity
$(\mathbf D_0\mathbf x)^{\top}=\mathbf x^{\top}\mathbf D_0^{\top}$, and (d)
uses the definitions of
$\boldsymbol{\alpha}$
and
$\boldsymbol{\beta}$. Since
$\mathbf x=\mathbf Q\mathbf p$
and
$\mathbf y=\mathbf Q\mathbf q$,
\begin{align}
\mathbf x^{\top}\mathbf K_0\mathbf y
&\stackrel{(a)}{=}
\left(
\mathbf Q\mathbf p
\right)^{\top}
\left(
\mathbf Q\boldsymbol{\Lambda}\mathbf Q^{\top}
\right)
\left(
\mathbf Q\mathbf q
\right)
\nonumber\\
&\stackrel{(b)}{=}
\mathbf p^{\top}
\mathbf Q^{\top}\mathbf Q
\boldsymbol{\Lambda}
\mathbf Q^{\top}\mathbf Q
\mathbf q
\stackrel{(c)}{=}
\mathbf p^{\top}\boldsymbol{\Lambda}\mathbf q.
\label{eq:equiv-proof-bilinear-spectral}
\end{align}
Here (a) substitutes the inverse spectral transformations and the
eigendecomposition of $\mathbf K_0$, (b) uses the transpose identity and
associativity, and (c) applies
$\mathbf Q^{\top}\mathbf Q=\mathbf I_G$.
Thus
\begin{align}
\left\langle g,h\right\rangle_{B,h}
&\stackrel{(a)}{=}
\mathbf x^{\top}\mathbf K_0\mathbf y
\stackrel{(b)}{=}
\frac{1}{h}
\boldsymbol{\alpha}^{\top}\boldsymbol{\beta}
\stackrel{(c)}{=}
\mathbf p^{\top}\boldsymbol{\Lambda}\mathbf q.
\label{eq:equiv-proof-common-inner-product}
\end{align}
Here (a), (b), and (c) apply
\eqref{eq:equiv-proof-bilinear-nodal},
\eqref{eq:equiv-proof-bilinear-increment}, and
\eqref{eq:equiv-proof-bilinear-spectral}, respectively. Therefore, when the nodal, increment, and spectral parameter spaces are
respectively equipped with the pulled-back inner products displayed in
\eqref{eq:equiv-proof-common-inner-product}, their reconstruction maps are
linear isometric isomorphisms onto the common space
$\mathcal H_h^0$ equipped with its Brownian inner product. By
Theorem~\ref{thm:finite-brownian-rkhs}, this common inner-product space is a
finite-dimensional RKHS. Hence all three realizations generate the same finite
Brownian RKHS and the same hypothesis class, proving
\ref{thm:equivalence-of-implementations-iv}.

It remains to establish
\ref{thm:equivalence-of-implementations-v}.
Let
$\widetilde{\mathbf v}_1,\widetilde{\mathbf v}_2\in\mathbb R^G$
be arbitrary, and define
\begin{align}
\mathbf w
&:=
\widetilde{\mathbf v}_1-
\widetilde{\mathbf v}_2,
\nonumber\\
\mathbf c_r
&:=
\mathbf Q^{\top}\widetilde{\mathbf v}_r,
\qquad
\mathbf{w}_r
:=
\mathbf D_0\widetilde{\mathbf v}_r,
\qquad
r\in\{1,2\}.
\label{eq:equiv-proof-pairwise-coordinate-differences}
\end{align}
For the spectral coordinates,
\begin{align}
\left\|
\mathbf c_1-
\mathbf c_2
\right\|_2^2
&\stackrel{(a)}{=}
\left\|
\mathbf Q^{\top}
\left(
\widetilde{\mathbf v}_1-
\widetilde{\mathbf v}_2
\right)
\right\|_2^2
\stackrel{(b)}{=}
\left(
\mathbf Q^{\top}\mathbf w
\right)^{\top}
\left(
\mathbf Q^{\top}\mathbf w
\right)
\nonumber\\
&\stackrel{(c)}{=}
\mathbf w^{\top}
\mathbf Q\mathbf Q^{\top}
\mathbf w
\stackrel{(d)}{=}
\mathbf w^{\top}\mathbf I_G\mathbf w
\stackrel{(e)}{=}
\left\|\mathbf w\right\|_2^2.
\label{eq:equiv-proof-spectral-isometry}
\end{align}
Here (a) uses linearity of $\mathbf Q^{\top}$ and the definitions of
$\mathbf c_1$ and $\mathbf c_2$, (b) is the definition of the squared Euclidean
norm, (c) uses
$(\mathbf Q^{\top}\mathbf w)^{\top}=\mathbf w^{\top}\mathbf Q$, (d) uses
$\mathbf Q\mathbf Q^{\top}=\mathbf I_G$, and (e) again uses the definition of
the Euclidean norm. Thus the nodal-to-spectral map is a Euclidean isometry.

For the increment map,
\begin{align}
\mathbf D_0^{\top}\mathbf D_0
&\stackrel{(a)}{=}
\left(
\mathbf D\mathbf R
\right)^{\top}
\left(
\mathbf D\mathbf R
\right)
\stackrel{(b)}{=}
\mathbf R^{\top}
\mathbf D^{\top}\mathbf D
\mathbf R
\stackrel{(c)}{=}
\mathbf R^{\top}
\left(
h\mathbf K
\right)
\mathbf R
\stackrel{(d)}{=}
h
\mathbf R^{\top}\mathbf K\mathbf R
\stackrel{(e)}{=}
h\mathbf K_0.
\label{eq:equiv-proof-D0-Gram}
\end{align}
Here (a) substitutes
$\mathbf D_0=\mathbf D\mathbf R$, (b) uses
$(\mathbf D\mathbf R)^{\top}=\mathbf R^{\top}\mathbf D^{\top}$
and associativity, (c) rearranges the stiffness identity
$\mathbf K=h^{-1}\mathbf D^{\top}\mathbf D$
as
$\mathbf D^{\top}\mathbf D=h\mathbf K$, (d) extracts the scalar $h$, and
(e) uses
$\mathbf K_0=\mathbf R^{\top}\mathbf K\mathbf R$. Consequently,
\begin{align}
\left\|
\mathbf{w}_1-
\mathbf{w}_2
\right\|_2^2
&\stackrel{(a)}{=}
\left\|
\mathbf D_0
\left(
\widetilde{\mathbf v}_1-
\widetilde{\mathbf v}_2
\right)
\right\|_2^2
\stackrel{(b)}{=}
\left(
\mathbf D_0\mathbf w
\right)^{\top}
\left(
\mathbf D_0\mathbf w
\right)
\stackrel{(c)}{=}
\mathbf w^{\top}
\mathbf D_0^{\top}\mathbf D_0
\mathbf w
\stackrel{(d)}{=}
h
\mathbf w^{\top}\mathbf K_0\mathbf w.
\label{eq:equiv-proof-increment-metric}
\end{align}
Here (a) uses linearity of $\mathbf D_0$ and the definitions of
$\mathbf{w}_1$ and $\mathbf{w}_2$, (b) is the definition of
the squared Euclidean norm, (c) uses the transpose identity and associativity,
and (d) applies
\eqref{eq:equiv-proof-D0-Gram}.

Finally, Lemma~\ref{lem:block-decomposition} gives
\begin{align}
h\mathbf K_0
&\stackrel{(a)}{=}
\begin{pmatrix}
\mathbf T_m & \mathbf0
\\
\mathbf0 & \mathbf T_m
\end{pmatrix}.
\label{eq:equiv-proof-block-Gram}
\end{align}
Equality (a) is exactly the block decomposition after multiplication by $h$.
If $m\geq2$, then the entrywise definition of $\mathbf T_m$ gives
\begin{align}
(\mathbf T_m)_{0,1}
&\stackrel{(a)}{=}
-1
\stackrel{(b)}{\neq}
0
\stackrel{(c)}{=}
(\mathbf I_m)_{0,1}.
\end{align}
Here (a) uses
$|0-1|=1$, (b) is the elementary inequality
$-1\neq0$, and (c) uses the zero off-diagonal entries of the identity matrix.
Therefore
$\mathbf T_m\neq\mathbf I_m$
and, by
\eqref{eq:equiv-proof-block-Gram},
\begin{align}
\mathbf D_0^{\top}\mathbf D_0
&\stackrel{(a)}{\neq}
\mathbf I_G,
\qquad
m\geq2.
\end{align}
Inequality (a) follows from
\eqref{eq:equiv-proof-D0-Gram},
\eqref{eq:equiv-proof-block-Gram}, and
$\mathbf T_m\neq\mathbf I_m$. Hence the increment map is nonorthogonal for
$G=2m\geq4$. In contrast,
\eqref{eq:equiv-proof-spectral-isometry}
shows that the spectral map is an orthogonal change of coordinates of the reduced nodal
coordinates. This proves
\ref{thm:equivalence-of-implementations-v}
and completes the proof.

\end{noheadproof}\hfill\ensuremath{\square}

\subsection{Proof of
\texorpdfstring{\Cref{thm:complete-dct-spectrum}}{Theorem 3}}
\label{app:proof-complete-dct-spectrum}

\begin{noheadproof}

By Lemma~\ref{lem:block-decomposition},
\begin{align}
\mathbf K_0
&\stackrel{(a)}{=}
\frac{1}{h}
\begin{pmatrix}
\mathbf T_m & \mathbf0
\\
\mathbf0 & \mathbf T_m
\end{pmatrix},
\label{eq:dct8-proof-K0-block}
\end{align}
where
$\mathbf T_m=(\tau_{rs})_{r,s=0}^{m-1}$
is defined by
\begin{align}
\tau_{rs}
:=
\begin{cases}
1,
&
 r=s=0,
\\
2,
&
 1\leq r=s\leq m-1,
\\
-1,
&
 \left|r-s\right|=1,
\\
0,
&
 \text{otherwise}.
\end{cases}
\label{eq:dct8-proof-Tm-entrywise}
\end{align}
Equality (a) is exactly the block decomposition proved in
Lemma~\ref{lem:block-decomposition}. In particular,
$\mathbf T_1=(1)$.
It is therefore sufficient first to construct an orthonormal
eigendecomposition of $\mathbf T_m$ and then to apply it independently to the
two diagonal blocks in \eqref{eq:dct8-proof-K0-block}.

For every $k\in[m]$, set
\begin{align}
\alpha_m
&:=
\frac{2}{\sqrt{2m+1}},
\qquad
\theta_k
:=
\frac{\left(2k-1\right)\pi}{2m+1},
\qquad
\nu_k
:=
2-2\cos\left(\theta_k\right),
\label{eq:dct8-proof-basic-definitions}
\end{align}
and define the extended cosine sequence
\begin{align}
\widehat q_k(j)
&:=
\alpha_m
\cos\left(
\left(j+\frac12\right)\theta_k
\right),
\qquad
j\in\left\{-1,0,\ldots,m\right\}.
\label{eq:dct8-proof-extended-sequence}
\end{align}
By the definition of $\mathbf q_k$ in the theorem statement,
\begin{align}
q_k(j)
&\stackrel{(a)}{=}
\widehat q_k(j),
\qquad
j=0,\ldots,m-1.
\label{eq:dct8-proof-q-restriction}
\end{align}
Equality (a) follows because
\eqref{eq:dct8-proof-extended-sequence} and the theorem statement use the same
normalization factor and the same cosine formula at these indices.
To identify the transform type explicitly, set $r:=k-1\in\{0,\ldots,m-1\}$.
Then
\begin{align}
\left(j+\frac12\right)\theta_k
&\stackrel{(a)}{=}
\frac{2\pi
\left(j+\frac12\right)
\left(r+\frac12\right)}{2m+1},
\qquad
j,r\in\{0,\ldots,m-1\}.
\label{eq:dct8-proof-transform-kernel}
\end{align}
Equality (a) uses $r=k-1$, and hence
$2k-1=2\left(r+\frac12\right)$. Therefore the entries of $\mathbf Q_m$ are
the type-VIII discrete-cosine kernel with the normalization factor
$\frac{2}{\sqrt{2m+1}}$; the orthonormality of this normalization is proved
below rather than assumed.

We first verify the two boundary relations encoded by the extended sequence.
At the outer boundary,
\begin{align}
\widehat q_k(-1)
&\stackrel{(a)}{=}
\alpha_m
\cos\left(
-\frac{\theta_k}{2}
\right)
\stackrel{(b)}{=}
\alpha_m
\cos\left(
\frac{\theta_k}{2}
\right)
\stackrel{(c)}{=}
\widehat q_k(0).
\label{eq:dct8-proof-left-ghost}
\end{align}
Here (a) substitutes $j=-1$ into
\eqref{eq:dct8-proof-extended-sequence}, (b) uses the evenness of the cosine,
and (c) substitutes $j=0$ into the same definition.
At the anchored boundary,
\begin{align}
\left(m+\frac12\right)\theta_k
&\stackrel{(a)}{=}
\frac{2m+1}{2}
\frac{\left(2k-1\right)\pi}{2m+1}
\stackrel{(b)}{=}
\frac{\left(2k-1\right)\pi}{2},
\label{eq:dct8-proof-anchor-angle}
\end{align}
where (a) uses
$m+\frac12=\frac{2m+1}{2}$
and the definition of $\theta_k$, while (b) cancels the nonzero factor
$2m+1$. Consequently,
\begin{align}
\widehat q_k(m)
&\stackrel{(a)}{=}
\alpha_m
\cos\left(
\frac{\left(2k-1\right)\pi}{2}
\right)
\stackrel{(b)}{=}
0.
\label{eq:dct8-proof-right-ghost}
\end{align}
Here (a) combines
\eqref{eq:dct8-proof-extended-sequence} and
\eqref{eq:dct8-proof-anchor-angle}, and (b) uses that the cosine of every odd
multiple of $\frac{\pi}{2}$ is zero.

We now prove the eigenvalue relation
$\mathbf T_m\mathbf q_k=\nu_k\mathbf q_k$.
The entrywise definition
\eqref{eq:dct8-proof-Tm-entrywise}, together with
\eqref{eq:dct8-proof-left-ghost} and
\eqref{eq:dct8-proof-right-ghost}, gives, for every
$j\in\left\{0,\ldots,m-1\right\}$,
\begin{align}
\left(\mathbf T_m\mathbf q_k\right)_j
&\stackrel{(a)}{=}
-\widehat q_k(j-1)
+2\widehat q_k(j)
-\widehat q_k(j+1).
\label{eq:dct8-proof-unified-difference}
\end{align}
To justify (a) without suppressing the endpoint cases, observe the following.
If $m=1$, then $j=0$, $\mathbf T_1=(1)$,
$\widehat q_k(-1)=q_k(0)$, and $\widehat q_k(1)=0$, so the right-hand side of
\eqref{eq:dct8-proof-unified-difference} equals $q_k(0)$.
If $m\geq2$ and $j=0$, the first row of $\mathbf T_m$ gives
$q_k(0)-q_k(1)$, which equals the displayed second difference because
$\widehat q_k(-1)=q_k(0)$.
If $m\geq3$ and $1\leq j\leq m-2$, the equality is the ordinary tridiagonal
row formula.
Finally, if $m\geq2$ and $j=m-1$, the last row gives
$-q_k(m-2)+2q_k(m-1)$, which equals the displayed second difference because
$\widehat q_k(m)=0$.
These cases exhaust all admissible pairs $(m,j)$.

For fixed $k\in[m]$ and
$j\in\left\{0,\ldots,m-1\right\}$, define
\begin{align}
a_{j,k}
&:=
\left(j+\frac12\right)\theta_k.
\label{eq:dct8-proof-a-jk}
\end{align}
Then
\begin{align}
\widehat q_k(j-1)
&\stackrel{(a)}{=}
\alpha_m\cos\left(a_{j,k}-\theta_k\right),
\qquad
\widehat q_k(j)
\stackrel{(b)}{=}
\alpha_m\cos\left(a_{j,k}\right),
\nonumber\\
\widehat q_k(j+1)
&\stackrel{(c)}{=}
\alpha_m\cos\left(a_{j,k}+\theta_k\right).
\label{eq:dct8-proof-neighbor-values}
\end{align}
Equalities (a)--(c) follow by substituting $j-1$, $j$, and $j+1$ into
\eqref{eq:dct8-proof-extended-sequence} and using
\eqref{eq:dct8-proof-a-jk}.
Therefore,
\begin{align}
\left(\mathbf T_m\mathbf q_k\right)_j
&\stackrel{(a)}{=}
\alpha_m
\left[
-\cos\left(a_{j,k}-\theta_k\right)
+2\cos\left(a_{j,k}\right)
-\cos\left(a_{j,k}+\theta_k\right)
\right]
\nonumber\\
&\stackrel{(b)}{=}
\alpha_m
\left[
2\cos\left(a_{j,k}\right)
-2\cos\left(a_{j,k}\right)\cos\left(\theta_k\right)
\right]
\stackrel{(c)}{=}
\left(2-2\cos\left(\theta_k\right)\right)
\alpha_m\cos\left(a_{j,k}\right)
\nonumber\\
&\stackrel{(d)}{=}
\nu_k\widehat q_k(j)
\stackrel{(e)}{=}
\nu_k q_k(j).
\label{eq:dct8-proof-component-eigenvalue}
\end{align}
Here (a) combines
\eqref{eq:dct8-proof-unified-difference} and
\eqref{eq:dct8-proof-neighbor-values}; (b) uses
\begin{align}
\cos\left(a_{j,k}-\theta_k\right)
+
\cos\left(a_{j,k}+\theta_k\right)
&\stackrel{(a)}{=}
2\cos\left(a_{j,k}\right)\cos\left(\theta_k\right),
\end{align}
which is the elementary cosine addition identity; (c) factors the common
quantity $\alpha_m\cos\left(a_{j,k}\right)$; (d) uses the definitions of
$\nu_k$, $\widehat q_k(j)$, and $a_{j,k}$; and (e) applies
\eqref{eq:dct8-proof-q-restriction}.
Since
\eqref{eq:dct8-proof-component-eigenvalue} holds at every component,
\begin{align}
\mathbf T_m\mathbf q_k
&\stackrel{(a)}{=}
\nu_k\mathbf q_k,
\qquad
k\in[m].
\label{eq:dct8-proof-eigenpair}
\end{align}
Equality (a) follows from equality of the corresponding components in
$\mathbb R^m$.

We next prove that the eigenvalues $\nu_1,\ldots,\nu_m$ are positive and
pairwise distinct. For every $k\in[m]$,
\begin{align}
1
&\stackrel{(a)}{\leq}
2k-1
\stackrel{(b)}{\leq}
2m-1
\stackrel{(c)}{<}
2m+1.
\label{eq:dct8-proof-integer-angle-bounds}
\end{align}
Here (a) uses $k\geq1$ and (b) uses $k\leq m$. Since $2m+1>0$, multiplication by
$\frac{\pi}{2m+1}>0$ preserves the inequalities and gives
\begin{align}
0
&\stackrel{(a)}{<}
\theta_k
\stackrel{(b)}{<}
\pi.
\label{eq:dct8-proof-angle-open-interval}
\end{align}
Here (a) follows from the first inequality in
\eqref{eq:dct8-proof-integer-angle-bounds}, and (b) follows from its last
strict inequality.
Moreover, if $1\leq k<\ell\leq m$, then
\begin{align}
2k-1
&\stackrel{(a)}{<}
2\ell-1
\quad\Longrightarrow\quad
\theta_k
\stackrel{(b)}{<}
\theta_{\ell}.
\label{eq:dct8-proof-angle-order}
\end{align}
Inequality (a) follows from $k<\ell$, and the implication to inequality (b)
follows after multiplication by the positive factor $\frac{\pi}{2m+1}$.
{Since
$\tfrac{\mathrm d}{\mathrm d\vartheta}\left(2-2\cos\vartheta\right)=2\sin\vartheta>0$ on
$(0,\pi)$, the map $\vartheta\mapsto2-2\cos\vartheta$ is strictly increasing there.} Combining
\eqref{eq:dct8-proof-angle-open-interval},
\eqref{eq:dct8-proof-angle-order} with {this monotonicity} yields
\begin{align}
\nu_k
&\stackrel{(a)}{>}
0,
\qquad
k\in[m],
\nonumber\\
\nu_k
&\stackrel{(b)}{<}
\nu_{\ell},
\qquad
1\leq k<\ell\leq m.
\label{eq:dct8-proof-mu-order}
\end{align}
Here (a) uses $\theta_k>0$ and
$2-2\cos\left(0\right)=0$, while (b) uses strict monotonicity and the strict
ordering of the angles. In particular, the values
$\nu_1,\ldots,\nu_m$ are positive and pairwise distinct.

We now prove pairwise orthogonality directly, without invoking the spectral
theorem. The entrywise formula
\eqref{eq:dct8-proof-Tm-entrywise} is invariant under interchange of $r$ and
$s$, because the conditions $r=s$, $r=s=0$, and $\left|r-s\right|=1$ are all
symmetric. Thus
\begin{align}
\tau_{rs}
&\stackrel{(a)}{=}
\tau_{sr}
\quad\Longrightarrow\quad
\mathbf T_m^{\top}
\stackrel{(b)}{=}
\mathbf T_m.
\label{eq:dct8-proof-Tm-symmetric}
\end{align}
Here (a) follows from the cases in
\eqref{eq:dct8-proof-Tm-entrywise}, and (b) is the entrywise definition of a
symmetric matrix.
Fix $k,\ell\in[m]$ with $k\neq\ell$. Then
\begin{align}
\nu_k\mathbf q_k^{\top}\mathbf q_{\ell}
&\stackrel{(a)}{=}
\left(\nu_k\mathbf q_k\right)^{\top}\mathbf q_{\ell}
\stackrel{(b)}{=}
\left(\mathbf T_m\mathbf q_k\right)^{\top}\mathbf q_{\ell}
\nonumber\\
&\stackrel{(c)}{=}
\mathbf q_k^{\top}\mathbf T_m^{\top}\mathbf q_{\ell}
\stackrel{(d)}{=}
\mathbf q_k^{\top}\mathbf T_m\mathbf q_{\ell}
\nonumber\\
&\stackrel{(e)}{=}
\mathbf q_k^{\top}\left(\nu_{\ell}\mathbf q_{\ell}\right)
\stackrel{(f)}{=}
\nu_{\ell}\mathbf q_k^{\top}\mathbf q_{\ell}.
\label{eq:dct8-proof-orthogonality-chain}
\end{align}
Here (a) moves the real scalar $\nu_k$ inside the transpose; (b) applies the
eigenvalue relation
\eqref{eq:dct8-proof-eigenpair} for $k$; (c) uses
$\left(\mathbf A\mathbf x\right)^{\top}=\mathbf x^{\top}\mathbf A^{\top}$;
(d) uses
\eqref{eq:dct8-proof-Tm-symmetric}; (e) applies
\eqref{eq:dct8-proof-eigenpair} for $\ell$; and (f) extracts the scalar
$\nu_{\ell}$. Subtracting the right-hand side from the left-hand side gives
\begin{align}
\left(\nu_k-\nu_{\ell}\right)
\mathbf q_k^{\top}\mathbf q_{\ell}
&\stackrel{(a)}{=}
0.
\end{align}
Equality (a) is a rearrangement of
\eqref{eq:dct8-proof-orthogonality-chain}. Since
$\nu_k\neq\nu_{\ell}$ by
\eqref{eq:dct8-proof-mu-order}, division by the nonzero scalar
$\nu_k-\nu_{\ell}$ yields
\begin{align}
\mathbf q_k^{\top}\mathbf q_{\ell}
&\stackrel{(a)}{=}
0,
\qquad
k\neq\ell.
\label{eq:dct8-proof-pairwise-orthogonality}
\end{align}
Equality (a) uses the elementary fact that a product of a nonzero real number
and a real number can vanish only if the second factor vanishes.

It remains to compute the norm of each $\mathbf q_k$. By definition,
\begin{align}
\left\|\mathbf q_k\right\|_2^2
&\stackrel{(a)}{=}
\sum_{j=0}^{m-1}q_k(j)^2
\stackrel{(b)}{=}
\frac{4}{2m+1}
\sum_{j=0}^{m-1}
\cos^2\left(
\left(j+\frac12\right)\theta_k
\right).
\label{eq:dct8-proof-norm-start}
\end{align}
Here (a) is the definition of the squared Euclidean norm, and (b) substitutes
the component formula and uses
$\alpha_m^2=\frac{4}{2m+1}$.
The identity
$\cos^2\left(x\right)=\frac{1+\cos\left(2x\right)}{2}$ gives
\begin{align}
\sum_{j=0}^{m-1}
\cos^2\left(
\left(j+\frac12\right)\theta_k
\right)
&\stackrel{(a)}{=}
\frac12
\sum_{j=0}^{m-1}
\left[
1+
\cos\left(
\left(2j+1\right)\theta_k
\right)
\right]
\nonumber\\
&\stackrel{(b)}{=}
\frac{m}{2}
+
\frac12
\sum_{j=0}^{m-1}
\cos\left(
\left(2j+1\right)\theta_k
\right).
\label{eq:dct8-proof-cos-square-split}
\end{align}
Here (a) applies the double-angle identity term by term, and (b) uses
$\sum_{j=0}^{m-1}1=m$ and separates the two finite sums.

Define
\begin{align}
S_k
&:=
\sum_{j=0}^{m-1}
\cos\left(
\left(2j+1\right)\theta_k
\right).
\label{eq:dct8-proof-Sk-definition}
\end{align}
For every $j\in\left\{0,\ldots,m-1\right\}$, the sine addition formula gives
\begin{align}
2\sin\left(\theta_k\right)
\cos\left(
\left(2j+1\right)\theta_k
\right)
&\stackrel{(a)}{=}
\sin\left(
\left(2j+2\right)\theta_k
\right)
-
\sin\left(
2j\theta_k
\right).
\label{eq:dct8-proof-telescoping-identity}
\end{align}
Indeed, (a) is the identity
$2\sin\left(x\right)\cos\left(y\right)
=
\sin\left(x+y\right)+\sin\left(x-y\right)$
with
$x=\theta_k$ and
$y=\left(2j+1\right)\theta_k$, together with the oddness of the sine.
Summing
\eqref{eq:dct8-proof-telescoping-identity} from $j=0$ to $j=m-1$ yields
\begin{align}
2\sin\left(\theta_k\right)S_k
&\stackrel{(a)}{=}
\sum_{j=0}^{m-1}
\left[
\sin\left(
\left(2j+2\right)\theta_k
\right)
-
\sin\left(
2j\theta_k
\right)
\right]
\nonumber\\
&\stackrel{(b)}{=}
\sin\left(2m\theta_k\right)
-
\sin\left(0\right)
\stackrel{(c)}{=}
\sin\left(2m\theta_k\right).
\label{eq:dct8-proof-telescoping-sum}
\end{align}
Here (a) uses the definition of $S_k$ and linearity of finite summation; (b)
uses telescopic cancellation, since every intermediate term
$\sin\left(2r\theta_k\right)$ for $r=1,\ldots,m-1$ appears once with sign
$+1$ and once with sign $-1$; and (c) uses $\sin\left(0\right)=0$.
Furthermore,
\begin{align}
2m\theta_k
&\stackrel{(a)}{=}
\left(2m+1\right)\theta_k-\theta_k
\stackrel{(b)}{=}
\left(2k-1\right)\pi-\theta_k,
\label{eq:dct8-proof-2m-angle}
\end{align}
where (a) adds and subtracts $\theta_k$, and (b) uses
$\left(2m+1\right)\theta_k=\left(2k-1\right)\pi$.
Since $2k-1$ is odd,
\begin{align}
\sin\left(2m\theta_k\right)
&\stackrel{(a)}{=}
\sin\left(
\left(2k-1\right)\pi-\theta_k
\right)
\stackrel{(b)}{=}
\sin\left(\theta_k\right).
\label{eq:dct8-proof-sine-reflection}
\end{align}
Here (a) applies
\eqref{eq:dct8-proof-2m-angle}, and (b) uses
$\sin\left(\left(2r-1\right)\pi-x\right)=\sin\left(x\right)$ for every
integer $r$.
Combining
\eqref{eq:dct8-proof-telescoping-sum} and
\eqref{eq:dct8-proof-sine-reflection} gives
\begin{align}
2\sin\left(\theta_k\right)S_k
&\stackrel{(a)}{=}
\sin\left(\theta_k\right).
\end{align}
Equality (a) substitutes
\eqref{eq:dct8-proof-sine-reflection} into
\eqref{eq:dct8-proof-telescoping-sum}. By
\eqref{eq:dct8-proof-angle-open-interval},
$\sin\left(\theta_k\right)>0$, and hence it is nonzero. Division by
$2\sin\left(\theta_k\right)$ therefore yields
\begin{align}
S_k
&\stackrel{(a)}{=}
\frac12.
\label{eq:dct8-proof-Sk-value}
\end{align}
Equality (a) performs the valid division by the nonzero quantity
$2\sin\left(\theta_k\right)$.
Substituting
\eqref{eq:dct8-proof-Sk-value} into
\eqref{eq:dct8-proof-cos-square-split} gives
\begin{align}
\sum_{j=0}^{m-1}
\cos^2\left(
\left(j+\frac12\right)\theta_k
\right)
&\stackrel{(a)}{=}
\frac{m}{2}
+
\frac12\left(\frac12\right)
\stackrel{(b)}{=}
\frac{2m+1}{4}.
\label{eq:dct8-proof-cos-square-value}
\end{align}
Here (a) uses
\eqref{eq:dct8-proof-Sk-value}, and (b) puts the two terms over the common
denominator $4$. Finally,
\begin{align}
\left\|\mathbf q_k\right\|_2^2
&\stackrel{(a)}{=}
\frac{4}{2m+1}
\frac{2m+1}{4}
\stackrel{(b)}{=}
1.
\label{eq:dct8-proof-unit-norm}
\end{align}
Here (a) combines
\eqref{eq:dct8-proof-norm-start} and
\eqref{eq:dct8-proof-cos-square-value}, and (b) cancels the positive factor
$2m+1$ and the nonzero factor $4$.

Equations
\eqref{eq:dct8-proof-pairwise-orthogonality} and
\eqref{eq:dct8-proof-unit-norm} imply
\begin{align}
\mathbf q_k^{\top}\mathbf q_{\ell}
&\stackrel{(a)}{=}
\delta_{k\ell},
\qquad
k,\ell\in[m].
\label{eq:dct8-proof-orthonormality}
\end{align}
Here (a) uses the zero cross-products when $k\neq\ell$ and the unit norm when
$k=\ell$.
To verify explicitly that the family is a basis, suppose that
$a_1,\ldots,a_m\in\mathbb R$ satisfy
\begin{align}
\sum_{k=1}^{m}a_k\mathbf q_k
&\stackrel{(a)}{=}
\mathbf0.
\label{eq:dct8-proof-linear-combination-zero}
\end{align}
For every fixed $\ell\in[m]$, left multiplication by
$\mathbf q_{\ell}^{\top}$ gives
\begin{align}
0
&\stackrel{(a)}{=}
\mathbf q_{\ell}^{\top}\mathbf0
\stackrel{(b)}{=}
\mathbf q_{\ell}^{\top}
\sum_{k=1}^{m}a_k\mathbf q_k
\stackrel{(c)}{=}
\sum_{k=1}^{m}a_k
\mathbf q_{\ell}^{\top}\mathbf q_k
\nonumber\\
&\stackrel{(d)}{=}
\sum_{k=1}^{m}a_k\delta_{\ell k}
\stackrel{(e)}{=}
a_{\ell}.
\end{align}
Here (a) uses $\mathbf q_{\ell}^{\top}\mathbf0=0$; (b) applies
\eqref{eq:dct8-proof-linear-combination-zero}; (c) uses linearity of the inner
product in a finite sum; (d) uses
\eqref{eq:dct8-proof-orthonormality}; and (e) uses the defining property of
the Kronecker delta. Thus $a_{\ell}=0$ for every $\ell\in[m]$, so the family
is linearly independent. Since it contains exactly $m$ vectors in the
$m$-dimensional space $\mathbb R^m$, it is a basis. This proves
\ref{thm:complete-dct-spectrum-i}.

We next prove
\ref{thm:complete-dct-spectrum-ii}.
Define
\begin{align}
\mathbf Q_m
&:=
\left[
\mathbf q_1,\ldots,\mathbf q_m
\right],
\qquad
\boldsymbol{\Theta}
:=
\operatorname{diag}\left(
\nu_1,\ldots,\nu_m
\right).
\label{eq:dct8-proof-Qm-Theta}
\end{align}
For $k,\ell\in[m]$, the $(k,\ell)$ entry of
$\mathbf Q_m^{\top}\mathbf Q_m$ satisfies
\begin{align}
\left(\mathbf Q_m^{\top}\mathbf Q_m\right)_{k\ell}
&\stackrel{(a)}{=}
\mathbf q_k^{\top}\mathbf q_{\ell}
\stackrel{(b)}{=}
\delta_{k\ell}
\stackrel{(c)}{=}
\left(\mathbf I_m\right)_{k\ell}.
\end{align}
Here (a) is the rule for multiplying a matrix by its transpose in terms of its
columns; (b) applies
\eqref{eq:dct8-proof-orthonormality}; and (c) is the entrywise definition of
the identity matrix. Equality of all entries gives
\begin{align}
\mathbf Q_m^{\top}\mathbf Q_m
&\stackrel{(a)}{=}
\mathbf I_m.
\label{eq:dct8-proof-Qm-left-orthogonal}
\end{align}
Equality (a) follows because the two $m\times m$ matrices have identical
entries. Since the columns of $\mathbf Q_m$ form a basis,
$\mathbf Q_m$ is invertible. Multiplying
\eqref{eq:dct8-proof-Qm-left-orthogonal} from the right by
$\mathbf Q_m^{-1}$ gives
\begin{align}
\mathbf Q_m^{\top}
&\stackrel{(a)}{=}
\mathbf Q_m^{-1},
\end{align}
where (a) uses
$\mathbf Q_m^{\top}\mathbf Q_m\mathbf Q_m^{-1}
=
\mathbf I_m\mathbf Q_m^{-1}$.
Therefore,
\begin{align}
\mathbf Q_m\mathbf Q_m^{\top}
&\stackrel{(a)}{=}
\mathbf Q_m\mathbf Q_m^{-1}
\stackrel{(b)}{=}
\mathbf I_m.
\label{eq:dct8-proof-Qm-right-orthogonal}
\end{align}
Here (a) substitutes $\mathbf Q_m^{\top}=\mathbf Q_m^{-1}$, and (b) uses the
defining property of an inverse.

Collecting the eigenvalue identities
\eqref{eq:dct8-proof-eigenpair} columnwise gives
\begin{align}
\mathbf T_m\mathbf Q_m
&\stackrel{(a)}{=}
\mathbf T_m
\left[
\mathbf q_1,\ldots,\mathbf q_m
\right]
\stackrel{(b)}{=}
\left[
\mathbf T_m\mathbf q_1,
\ldots,
\mathbf T_m\mathbf q_m
\right]
\nonumber\\
&\stackrel{(c)}{=}
\left[
\nu_1\mathbf q_1,
\ldots,
\nu_m\mathbf q_m
\right]
\stackrel{(d)}{=}
\mathbf Q_m\boldsymbol{\Theta}.
\label{eq:dct8-proof-columnwise-eigenvalue}
\end{align}
Here (a) substitutes the definition of $\mathbf Q_m$; (b) applies matrix
multiplication column by column; (c) uses
\eqref{eq:dct8-proof-eigenpair}; and (d) uses the definition of the diagonal
matrix $\boldsymbol{\Theta}$. Consequently,
\begin{align}
\mathbf T_m
&\stackrel{(a)}{=}
\mathbf T_m\mathbf I_m
\stackrel{(b)}{=}
\mathbf T_m\mathbf Q_m\mathbf Q_m^{\top}
\stackrel{(c)}{=}
\mathbf Q_m\boldsymbol{\Theta}\mathbf Q_m^{\top}
\nonumber\\
&\stackrel{(d)}{=}
\mathbf Q_m
\operatorname{diag}\left(
\nu_1,\ldots,\nu_m
\right)
\mathbf Q_m^{\top}.
\label{eq:dct8-proof-Tm-decomposition}
\end{align}
Here (a) inserts the identity on the right; (b) applies
\eqref{eq:dct8-proof-Qm-right-orthogonal}; (c) uses
\eqref{eq:dct8-proof-columnwise-eigenvalue}; and (d) substitutes the
definition of $\boldsymbol{\Theta}$. This proves
\ref{thm:complete-dct-spectrum-ii}.

We now prove
\ref{thm:complete-dct-spectrum-iii}.
Because $A>0$, $G=2m>0$, and
$h=\frac{2A}{G}$,
\begin{align}
h
&\stackrel{(a)}{=}
\frac{2A}{2m}
\stackrel{(b)}{=}
\frac{A}{m}
\stackrel{(c)}{>}
0.
\label{eq:dct8-proof-h-positive}
\end{align}
Here (a) substitutes $G=2m$, (b) cancels the nonzero factor $2$, and (c) uses
$A>0$ and $m\geq1$.
Define
\begin{align}
\mathbf Q
&:=
\begin{pmatrix}
\mathbf Q_m & \mathbf0
\\
\mathbf0 & \mathbf Q_m
\end{pmatrix},
\qquad
\boldsymbol{\Lambda}
:=
\frac{1}{h}
\begin{pmatrix}
\boldsymbol{\Theta} & \mathbf0
\\
\mathbf0 & \boldsymbol{\Theta}
\end{pmatrix}.
\label{eq:dct8-proof-full-Q-Lambda}
\end{align}
Using
\eqref{eq:dct8-proof-Qm-left-orthogonal},
\begin{align}
\mathbf Q^{\top}\mathbf Q
&\stackrel{(a)}{=}
\begin{pmatrix}
\mathbf Q_m^{\top} & \mathbf0
\\
\mathbf0 & \mathbf Q_m^{\top}
\end{pmatrix}
\begin{pmatrix}
\mathbf Q_m & \mathbf0
\\
\mathbf0 & \mathbf Q_m
\end{pmatrix}
\nonumber\\
&\stackrel{(b)}{=}
\begin{pmatrix}
\mathbf Q_m^{\top}\mathbf Q_m & \mathbf0
\\
\mathbf0 & \mathbf Q_m^{\top}\mathbf Q_m
\end{pmatrix}
\stackrel{(c)}{=}
\begin{pmatrix}
\mathbf I_m & \mathbf0
\\
\mathbf0 & \mathbf I_m
\end{pmatrix}
\stackrel{(d)}{=}
\mathbf I_{2m}
\stackrel{(e)}{=}
\mathbf I_G.
\label{eq:dct8-proof-full-Q-left-orthogonal}
\end{align}
Here (a) transposes the block-diagonal matrix in
\eqref{eq:dct8-proof-full-Q-Lambda}; (b) performs the block multiplication;
(c) uses
\eqref{eq:dct8-proof-Qm-left-orthogonal}; (d) identifies the resulting block
matrix as the identity in dimension $2m$; and (e) uses $G=2m$.
Similarly, using
\eqref{eq:dct8-proof-Qm-right-orthogonal},
\begin{align}
\mathbf Q\mathbf Q^{\top}
&\stackrel{(a)}{=}
\begin{pmatrix}
\mathbf Q_m\mathbf Q_m^{\top} & \mathbf0
\\
\mathbf0 & \mathbf Q_m\mathbf Q_m^{\top}
\end{pmatrix}
\stackrel{(b)}{=}
\mathbf I_G.
\label{eq:dct8-proof-full-Q-right-orthogonal}
\end{align}
Here (a) performs the analogous block multiplication, and (b) applies
\eqref{eq:dct8-proof-Qm-right-orthogonal} and $G=2m$. Thus $\mathbf Q$ is
orthogonal.

Substituting
\eqref{eq:dct8-proof-Tm-decomposition} into
\eqref{eq:dct8-proof-K0-block} gives
\begin{align}
\mathbf K_0
&\stackrel{(a)}{=}
\frac{1}{h}
\begin{pmatrix}
\mathbf Q_m\boldsymbol{\Theta}\mathbf Q_m^{\top}
&
\mathbf0
\\
\mathbf0
&
\mathbf Q_m\boldsymbol{\Theta}\mathbf Q_m^{\top}
\end{pmatrix}
\nonumber\\
&\stackrel{(b)}{=}
\begin{pmatrix}
\mathbf Q_m & \mathbf0
\\
\mathbf0 & \mathbf Q_m
\end{pmatrix}
\left[
\frac{1}{h}
\begin{pmatrix}
\boldsymbol{\Theta} & \mathbf0
\\
\mathbf0 & \boldsymbol{\Theta}
\end{pmatrix}
\right]
\begin{pmatrix}
\mathbf Q_m^{\top} & \mathbf0
\\
\mathbf0 & \mathbf Q_m^{\top}
\end{pmatrix}
\nonumber\\
&\stackrel{(c)}{=}
\mathbf Q\boldsymbol{\Lambda}\mathbf Q^{\top}.
\label{eq:dct8-proof-K0-decomposition}
\end{align}
Here (a) substitutes the eigendecomposition of both copies of $\mathbf T_m$;
(b) factors the block-diagonal product and places the scalar factor
$\frac1h$ in the middle block; and (c) applies
\eqref{eq:dct8-proof-full-Q-Lambda}. This proves
\ref{thm:complete-dct-spectrum-iii}.

It remains to prove
\ref{thm:complete-dct-spectrum-iv}.
The elementary half-angle identity gives
\begin{align}
\nu_k
&\stackrel{(a)}{=}
2-2\cos\left(\theta_k\right)
\stackrel{(b)}{=}
4\sin^2\left(
\frac{\theta_k}{2}
\right),
\label{eq:dct8-proof-mu-half-angle}
\end{align}
where (a) is the definition of $\nu_k$, and (b) uses
$1-\cos\left(x\right)=2\sin^2\left(\frac{x}{2}\right)$.
Thus the distinct eigenvalue values appearing on the diagonal of
$\boldsymbol{\Lambda}$ are
\begin{align}
\lambda_k
&\stackrel{(a)}{=}
\frac{\nu_k}{h}
\stackrel{(b)}{=}
\frac{4}{h}
\sin^2\left(
\frac{\theta_k}{2}
\right)
\nonumber\\
&\stackrel{(c)}{=}
\frac{4}{h}
\sin^2\left(
\frac{\left(2k-1\right)\pi}{2\left(2m+1\right)}
\right),
\qquad
k=1,\ldots,m.
\label{eq:dct8-proof-lambda-formula}
\end{align}
Here (a) reads the diagonal scaling from
\eqref{eq:dct8-proof-full-Q-Lambda}; (b) applies
\eqref{eq:dct8-proof-mu-half-angle}; and (c) substitutes the definition of
$\theta_k$.
Because $h>0$ by
\eqref{eq:dct8-proof-h-positive} and the $\nu_k$ are strictly ordered by
\eqref{eq:dct8-proof-mu-order},
\begin{align}
\lambda_k
&\stackrel{(a)}{>}
0,
\qquad
k\in[m],
\nonumber\\
\lambda_k
&\stackrel{(b)}{<}
\lambda_{\ell},
\qquad
1\leq k<\ell\leq m.
\label{eq:dct8-proof-lambda-order}
\end{align}
Here (a) and (b) follow by division of the corresponding inequalities in
\eqref{eq:dct8-proof-mu-order} by the positive scalar $h$.

By
\eqref{eq:dct8-proof-full-Q-Lambda},
\begin{align}
\boldsymbol{\Lambda}
&\stackrel{(a)}{=}
\operatorname{diag}\left(
\lambda_1,\ldots,\lambda_m,
\lambda_1,\ldots,\lambda_m
\right).
\label{eq:dct8-proof-Lambda-repeated}
\end{align}
Equality (a) uses $\lambda_k=\frac{\nu_k}{h}$ in each of the two diagonal
blocks. Since
\eqref{eq:dct8-proof-K0-decomposition} and
\eqref{eq:dct8-proof-full-Q-left-orthogonal} imply
$\mathbf Q^{-1}=\mathbf Q^{\top}$, the matrices $\mathbf K_0$ and
$\boldsymbol{\Lambda}$ are similar. Hence, for an indeterminate $z$,
\begin{align}
\det\left(z\mathbf I_G-\mathbf K_0\right)
&\stackrel{(a)}{=}
\det\left(
 z\mathbf I_G
-
\mathbf Q\boldsymbol{\Lambda}\mathbf Q^{\top}
\right)
\stackrel{(b)}{=}
\det\left(
\mathbf Q
\left(z\mathbf I_G-\boldsymbol{\Lambda}\right)
\mathbf Q^{\top}
\right)
\nonumber\\
&\stackrel{(c)}{=}
\det\left(\mathbf Q\right)
\det\left(z\mathbf I_G-\boldsymbol{\Lambda}\right)
\det\left(\mathbf Q^{\top}\right)
\stackrel{(d)}{=}
\det\left(
\mathbf Q\mathbf Q^{\top}
\right)
\det\left(z\mathbf I_G-\boldsymbol{\Lambda}\right)
\nonumber\\
&\stackrel{(e)}{=}
\det\left(z\mathbf I_G-\boldsymbol{\Lambda}\right)
\stackrel{(f)}{=}
\prod_{k=1}^{m}
\left(z-\lambda_k\right)^2.
\label{eq:dct8-proof-characteristic-polynomial}
\end{align}
Here (a) applies
\eqref{eq:dct8-proof-K0-decomposition}; (b) uses
$z\mathbf I_G=\mathbf Q\left(z\mathbf I_G\right)\mathbf Q^{\top}$, which
follows from
\eqref{eq:dct8-proof-full-Q-right-orthogonal}; (c) uses multiplicativity of
the determinant; (d) uses
$\det\left(\mathbf Q\right)\det\left(\mathbf Q^{\top}\right)
=
\det\left(\mathbf Q\mathbf Q^{\top}\right)$; (e) applies
\eqref{eq:dct8-proof-full-Q-right-orthogonal} and
$\det\left(\mathbf I_G\right)=1$; and (f) uses the repeated diagonal form
\eqref{eq:dct8-proof-Lambda-repeated}.
By
\eqref{eq:dct8-proof-lambda-order}, the numbers
$\lambda_1,\ldots,\lambda_m$ are distinct. Therefore each factor
$z-\lambda_k$ occurs exactly twice in the characteristic polynomial
\eqref{eq:dct8-proof-characteristic-polynomial}, and each $\lambda_k$ has
algebraic multiplicity two. This proves
\ref{thm:complete-dct-spectrum-iv} and completes the proof.

\end{noheadproof}\hfill\ensuremath{\square}

\begin{table}[!ht]
\centering
\caption{Summary of the auxiliary theoretical results. Their logical dependencies are illustrated in Fig.~\ref{fig:dependency_graph}. For the main theoretical results, see Table~\ref{tab:main-results}.}
\label{tab:auxiliary-results}

\begin{tabular}{lll}
\toprule
Result & Content & Page \\
\midrule
Lemma~\ref{lem:stiffness-matrix}
&
Structure of the stiffness matrix
&
page~\pageref{lem:stiffness-matrix}
\\

Lemma~\ref{lem:kernel-stiffness}
&
Kernel of the stiffness matrix
&
page~\pageref{lem:kernel-stiffness}
\\

Lemma~\ref{lem:block-decomposition}
&
Block decomposition of the anchored stiffness matrix
&
page~\pageref{lem:block-decomposition}
\\

Lemma~\ref{lem:increment-energy}
&
Brownian energy in increment coordinates
&
page~\pageref{lem:increment-energy}
\\

Lemma~\ref{lem:spectral-energy}
&
Brownian energy in spectral coordinates
&
page~\pageref{lem:spectral-energy}
\\

Lemma~\ref{lem:explicit-spectral-basis}
&
Explicit spectral Brownian basis
&
page~\pageref{lem:explicit-spectral-basis}
\\
\bottomrule
\end{tabular}

\end{table}

\section{Proofs of the Kernel, Conditioning, and Optimization Results}
\label{app:new-proofs}

This section completes the explicit kernel, approximation, conditioning,
and optimization consequences of the foundational results proved above.
We first establish the positive definiteness of the anchored stiffness
matrix in
\Cref{lem:K0-pd},
identify the reproducing kernel and Brownian Gram inverse in
\Cref{thm:finite-brownian-rkhs},
items~\ref{thm:finite-brownian-rkhs-i}--\ref{thm:finite-brownian-rkhs-ii},
and derive the conforming-subspace and sharp-approximation results in
\Cref{cor:conforming-subspace,cor:approximation}.
We then obtain the exact stiffness and coordinate condition numbers in
\Cref{cor:condition-number}.
Finally, we prove the linear-reparameterization identity in
\Cref{prop:precond},
the nodal--spectral and increment--Brownian gradient correspondences in
\Cref{prop:gd-equivariance},
their recursive blockwise extension in
\Cref{cor:blockwise-optimizer},
the mesh-independent increment least-squares bound in
\Cref{prop:least-squares-conditioning},
and the signed-permutation characterization of standard Adam equivariance
in
\Cref{prop:adam-nonequivariance}.

Throughout,
$\mathcal J=\{0,\ldots,G\}\setminus\{i_0\}$
indexes the reduced anchored coordinates of
\Cref{def:reduced-anchored-coordinates},
and
$\widetilde{\boldsymbol\phi}(s)
=(\phi_j(s))_{j\in\mathcal J}$
is understood in that same reduced ordering.
Thus, for
$f,g\in\mathcal H_h^0$
with reduced coordinates
$\widetilde{\mathbf v},\widetilde{\mathbf u}$,
\begin{align}
f(s)
&=
\widetilde{\boldsymbol\phi}(s)^{\top}
\widetilde{\mathbf v},
&
\left\langle
f,
g
\right\rangle_{B,h}
&=
\widetilde{\mathbf v}^{\top}
\mathbf K_0
\widetilde{\mathbf u}.
\label{eq:new-eval-and-inner}
\end{align}

\subsection{Proof of \texorpdfstring{\Cref{lem:K0-pd}}{Lemma 3.1}}
By \Cref{lem:kernel-stiffness},
$\ker\mathbf K=\operatorname{span}(\mathbf 1)$, and by
\Cref{prop:finite-brownian-energy}, $\mathbf K\succeq0$. If
$\mathbf x^{\top}\mathbf K_0\mathbf x=0$, then
$(\mathbf R\mathbf x)^{\top}\mathbf K(\mathbf R\mathbf x)=0$, so
$\mathbf R\mathbf x=c\mathbf1$. The anchor entry of $\mathbf R\mathbf x$ is zero; hence $c=0$. Since
$\mathbf R^{\top}\mathbf R=\mathbf I_G$, $\mathbf x=0$. Thus $\mathbf K_0\succ0$, and
\eqref{eq:new-eval-and-inner} is an inner product. \hfill$\square$

\subsection{Proof of \texorpdfstring{\Cref{thm:finite-brownian-rkhs}}{Theorem 1}, items (i)--(ii)}\label{thm:items}
Finite-dimensional RKHS existence and dimension are proved in
Appendix~\ref{app:proof-finite-brownian-rkhs}. Fix $t\in I$. If the representer
$k_h(\cdot,t)$ has reduced vector $\widetilde{\mathbf k}_t$, then for every
$\widetilde{\mathbf v}$,
\begin{align}
\widetilde{\mathbf v}^{\top}\mathbf K_0\widetilde{\mathbf k}_t
=f(t)=\widetilde{\boldsymbol\phi}(t)^{\top}\widetilde{\mathbf v}.
\end{align}
Therefore $\mathbf K_0\widetilde{\mathbf k}_t=\widetilde{\boldsymbol\phi}(t)$, and positive
definiteness gives
$\widetilde{\mathbf k}_t=\mathbf K_0^{-1}\widetilde{\boldsymbol\phi}(t)$, proving item (i).

For $r\in\mathcal J$, set $g_r=k_{\mathrm B}(\cdot,t_r)$. Its only possible kinks are at $0$ and
$t_r$, both grid nodes, so $g_r\in\mathcal H_h^0$. If $t_r>0$, then
$g_r'=\mathbb1_{(0,t_r)}$ almost everywhere; if $t_r<0$, then
$g_r'=-\mathbb1_{(t_r,0)}$. Hence, for every $f\in\mathcal H_h^0$,
\begin{align}
\langle f,g_r\rangle_{B,h}
&=
\begin{cases}
\int_0^{t_r}f'(t)\,\mathrm dt,&t_r>0,\\
-\int_{t_r}^{0}f'(t)\,\mathrm dt,&t_r<0,
\end{cases}
=f(t_r)-f(0)=f(t_r).
\end{align}
Thus $g_r=k_h(\cdot,t_r)$. Since
$\widetilde{\boldsymbol\phi}(t_r)=\mathbf e_{\pi(r)}$,
\begin{align}
(\mathbf K_0^{-1})_{\pi(r),\pi(q)}
=k_h(t_r,t_q)=k_{\mathrm B}(t_r,t_q).
\end{align}
Substitution into item (i) gives
\begin{align}
k_h(s,t)=\sum_{r,q\in\mathcal J}
\phi_r(s)k_{\mathrm B}(t_r,t_q)\phi_q(t),
\end{align}
which is the tensor-product bilinear interpolant. If one argument is a node, one sum collapses and the
preceding representer identity gives exact agreement. \hfill$\square$

\subsection{Proof of \texorpdfstring{\Cref{cor:conforming-subspace}}{Corollary 2}}
The functions $g_r=k_{\mathrm B}(\cdot,t_r)$ belong to $\mathcal H_h^0$, and their Gram matrix is
$\big(k_{\mathrm B}(t_r,t_q)\big)_{r,q}=\mathbf K_0^{-1}$ in reduced order. They are therefore linearly
independent, and their number equals $\dim\mathcal H_h^0=G$, so they span the space. Every
$g\in\mathcal H_h^0$ is continuous, piecewise affine, anchored, and has derivative in $L^2(I)$; hence
$g\in\mathcal H_{\mathrm B}$ and
$\|g\|_{\mathcal H_{\mathrm B}}^2=\int_I(g')^2=\|g\|_{B,h}^2$.

Let $I_hf$ be nodal interpolation of $f\in\mathcal H_{\mathrm B}$. For every $r\in\mathcal J$,
\begin{align}
\langle f-I_hf,g_r\rangle_{\mathcal H_{\mathrm B}}
=(f-I_hf)(t_r)=0.
\end{align}
Since the $g_r$ span $\mathcal H_h^0$, the residual is orthogonal to that space, so $I_h=\Pi_h$.
\hfill$\square$

\subsection{Proof of \texorpdfstring{\Cref{cor:approximation}}{Corollary 3}}
Fix a cell $T_i=[a,b]$ and write $h=b-a$. No cell crosses the anchor because $0$ is a grid node. Suppose
first $0\le a\le s\le t\le b$, and set $\alpha=(s-a)/h$, $\beta=(t-a)/h$. The four Brownian corner values
are $a,a,a,b$, so bilinear interpolation gives
\begin{align}
k_h(s,t)
&=(1-\alpha)(1-\beta)a+(1-\alpha)\beta a
+\alpha(1-\beta)a+\alpha\beta b=a+\alpha\beta h.
\end{align}
Since $k_{\mathrm B}(s,t)=s=a+\alpha h$,
\begin{align}
k_{\mathrm B}(s,t)-k_h(s,t)
=\alpha(1-\beta)h=\frac{(s-a)(b-t)}{h}.
\end{align}
For $a\le s\le t\le b\le0$, the corner values are $-a,-b,-b,-b$ and the same bilinear calculation gives
$k_h(s,t)=-b+(1-\alpha)(1-\beta)h$. Since $k_{\mathrm B}(s,t)=-t$, the same residual
$(s-a)(b-t)/h$ follows. Symmetry replaces $s,t$ by their minimum and maximum. On two different cells the
diagonal kink $s=t$ is absent from the corresponding rectangle and $k_{\mathrm B}$ is affine in each
argument there; hence its bilinear interpolant is exact. This proves \eqref{eq:residual-kernel}.

Because $\Pi_h$ is orthogonal, the error functional
$f\mapsto f(t)-\Pi_hf(t)$ has representer
$r_h(\cdot,t)$ and squared norm $r_h(t,t)$. Cauchy--Schwarz gives
\eqref{eq:power-function}; equality holds for a normalized residual section whenever $t$ is not a node.
Maximizing $(t-a)(b-t)/h$ on a cell gives $h/4$, proving the sharp $\sqrt h/2$ constant.

The Pythagorean identity follows from $f-\Pi_hf\perp\Pi_hf$. For the $L^2$ bound, set
$e=f-\Pi_hf$. On every cell, $e(a)=e(b)=0$, so the sharp Wirtinger inequality gives
\begin{align}
\|e\|_{L^2(T_i)}^2\le\frac{h^2}{\pi^2}\|e'\|_{L^2(T_i)}^2.
\end{align}
The derivative of $\Pi_hf$ on $T_i$ is the cell average of $f'$, so $e'$ is the residual after the
$L^2(T_i)$-orthogonal projection onto constants. Summing over cells,
\begin{align}
\|e\|_{L^2(I)}^2
\le\frac{h^2}{\pi^2}\|e'\|_{L^2(I)}^2
\le\frac{h^2}{\pi^2}\|f'\|_{L^2(I)}^2.
\end{align}
Sharpness follows by taking a sine arch supported on one cell,
$f(t)=\sin(\pi(t-a)/h)$ on that cell and zero elsewhere; it has zero nodal values and attains the
Wirtinger constant. \hfill$\square$

\subsection{Proof of \texorpdfstring{\Cref{cor:condition-number}}{Corollary 4}}
By \Cref{thm:complete-dct-spectrum}, the distinct stiffness eigenvalues are
\begin{align}
\lambda_k=\frac4h\sin^2\!\left(\frac{(2k-1)\pi}{2(2m+1)}\right),
\qquad k\in[m].
\end{align}
They increase with $k$. Thus
\begin{align}
\lambda_{\min}&=\frac4h\sin^2\!\left(\frac{\pi}{4m+2}\right),\\
\lambda_{\max}&=\frac4h\cos^2\!\left(\frac{\pi}{2m+1}\right),
\end{align}
which gives \eqref{eq:K0-condition}. Since
$\mathbf D_0^{\top}\mathbf D_0=h\mathbf K_0$, its squared singular values are
$h\lambda_k$. Taking square roots yields \eqref{eq:D0-singular-values} and
\eqref{eq:D0-condition}. Taylor expansion of sine and cosine gives the stated asymptotics. In particular
$h\lambda_{\max}\to4$; writing $\lambda_{\max}\to4/h$ would be meaningless because $h$ varies with $m$.
\hfill$\square$

\subsection{Proof of \texorpdfstring{\Cref{prop:precond}}{Proposition 2}}
The chain rule gives
$\nabla L_{\mathbf A}(\mathbf z)=\mathbf A^{-\top}\nabla L(\mathbf A^{-1}\mathbf z)$. Applying
$\mathbf A^{-1}$ to the update gives
\begin{align}
\widetilde{\mathbf v}_{k+1}
&=\mathbf A^{-1}\mathbf z_k-
\eta_k\mathbf A^{-1}\mathbf A^{-\top}\nabla L(\widetilde{\mathbf v}_k)\\
&=\widetilde{\mathbf v}_k-
\eta_k(\mathbf A^{\top}\mathbf A)^{-1}\nabla L(\widetilde{\mathbf v}_k).
\end{align}
The stochastic statement is identical with a shared stochastic gradient and its chain-rule transform.
\hfill$\square$

\subsection{Proof of \texorpdfstring{\Cref{prop:gd-equivariance}}{Proposition 3}}
For $\mathbf A=\mathbf Q^{\top}$, orthogonality gives
$\mathbf A^{\top}\mathbf A=\mathbf I_G$, so \eqref{eq:reparameterized-gd} is exactly the nodal update.
Mapped initializations and induction give identical function iterates. For $\mathbf A=\mathbf D_0$,
$\mathbf A^{\top}\mathbf A=h\mathbf K_0$, and \eqref{eq:reparameterized-gd} becomes
\eqref{eq:increment-brownian-gradient}. The Riemannian gradient for the constant metric with Gram matrix
$\mathbf K_0$ is $\mathbf K_0^{-1}\nabla L$; hence the discrete update is explicit Euler for that flow
with time increment $\eta_k/h$. \hfill$\square$

\subsection{Proof of \texorpdfstring{\Cref{cor:blockwise-optimizer}}{Corollary 5}}
Apply \Cref{prop:precond} to the block-diagonal matrix $\mathbf A$. Its Gram matrix and inverse are
block diagonal:
\begin{align}
(\mathbf A^{\top}\mathbf A)^{-1}
=\operatorname{blkdiag}\left(
(\mathbf A_1^{\top}\mathbf A_1)^{-1},\ldots,
(\mathbf A_B^{\top}\mathbf A_B)^{-1},\mathbf I\right).
\end{align}
If every $\mathbf A_b$ is orthogonal, this is the identity. If every $\mathbf A_b=\mathbf D_{0,b}$, the
$b$th block is $(h_b\mathbf K_{0,b})^{-1}$. The proof never separates the objective by blocks, so arbitrary
cross-layer coupling is allowed. \hfill$\square$

\subsection{Proof of \texorpdfstring{\Cref{prop:least-squares-conditioning}}{Proposition 4}}
Let
\begin{align}
\mathbf B:=\frac1n\sum_{i=1}^{n}
\widetilde{\boldsymbol\phi}(x_i)\widetilde{\boldsymbol\phi}(x_i)^{\top}
\succeq0.
\end{align}
The nodal Hessian is $\mathbf H_{\mathrm{nod}}=\mathbf B+\rho\mathbf K_0$; the spectral Hessian is
$\mathbf Q^{\top}\mathbf H_{\mathrm{nod}}\mathbf Q$, so the two have identical spectra. Since
$\mathbf D_0^{\top}\mathbf D_0=h\mathbf K_0$, the polar decomposition has the form
$\mathbf D_0=\sqrt h\,\mathbf U\mathbf K_0^{1/2}$ for an orthogonal $\mathbf U$. Therefore
\begin{align}
\mathbf H_{\mathrm{inc}}
&=\mathbf D_0^{-\top}(\mathbf B+\rho\mathbf K_0)\mathbf D_0^{-1}\\
&=\frac1h\mathbf U\left(
\mathbf K_0^{-1/2}\mathbf B\mathbf K_0^{-1/2}+\rho\mathbf I_G
\right)\mathbf U^{\top}.
\label{eq:increment-hessian-polar}
\end{align}
Set $\mathbf C=\mathbf K_0^{-1/2}\mathbf B\mathbf K_0^{-1/2}\succeq0$. Then
\begin{align}
\lambda_{\max}(\mathbf C)
&\le\operatorname{tr}(\mathbf C)
=\frac1n\sum_{i=1}^{n}
\widetilde{\boldsymbol\phi}(x_i)^{\top}\mathbf K_0^{-1}
\widetilde{\boldsymbol\phi}(x_i)\\
&=\frac1n\sum_{i=1}^{n}k_h(x_i,x_i)
\le\frac1n\sum_{i=1}^{n}k_{\mathrm B}(x_i,x_i)
=\frac1n\sum_{i=1}^{n}|x_i|\le A.
\end{align}
The penultimate inequality follows from the nonnegative residual diagonal in
\eqref{eq:power-function}. Hence every eigenvalue of $\mathbf C+\rho\mathbf I_G$ lies in
$[\rho,\rho+A]$, and \eqref{eq:increment-hessian-polar} proves
\eqref{eq:increment-hessian-bound}. If $\mathbf B=0$, the increment Hessian is
$(\rho/h)\mathbf I_G$, while nodal and spectral Hessians have condition number
$\kappa_2(\mathbf K_0)$. \hfill$\square$

\subsection{Proof of \texorpdfstring{\Cref{prop:adam-nonequivariance}}{Proposition 5}}
Write standard Adam, for gradient $\mathbf g_k$, as
\begin{align}
\mathbf m_k&=\beta_1\mathbf m_{k-1}+(1-\beta_1)\mathbf g_k,\\
\mathbf s_k&=\beta_2\mathbf s_{k-1}+(1-\beta_2)(\mathbf g_k\odot\mathbf g_k),\\
\widehat{\mathbf m}_k&=\frac{\mathbf m_k}{1-\beta_1^{k+1}},
&
\widehat{\mathbf s}_k&=\frac{\mathbf s_k}{1-\beta_2^{k+1}},\\
\mathbf x_{k+1}&=\mathbf x_k-
\eta_k\frac{\widehat{\mathbf m}_k}{\sqrt{\widehat{\mathbf s}_k}+\varepsilon\mathbf1},
\end{align}
where the last two operations are coordinatewise. If $\mathbf S$ is a signed permutation, gradients and
first moments transform by $\mathbf S^{\top}$, while squared gradients and second moments are permuted by
$|\mathbf S|^{\top}$. Induction through the recursion shows that the mapped update is exactly the original
one, so signed permutations are equivariances.

Conversely, let $\mathbf U$ be orthogonal and suppose Adam is equivariant for every gradient sequence. At the first
step, bias correction reduces the Adam direction to
$\mathbf F_{\varepsilon}(\mathbf g)=\mathbf g/(|\mathbf g|+\varepsilon\mathbf1)$. Equivariance under
$\mathbf z=\mathbf U^{\top}\mathbf x$ requires
\begin{align}
\mathbf U\mathbf F_{\varepsilon}(\mathbf U^{\top}\mathbf g)
=\mathbf F_{\varepsilon}(\mathbf g)
\qquad\text{for every }\mathbf g.
\label{eq:adam-first-step-equivariance}
\end{align}
Take $\mathbf g=t\mathbf U\mathbf e_j$ with $t>0$ and let $t\to\infty$. The left side tends to
$\mathbf U\mathbf e_j$, while the right side tends coordinatewise to
$\operatorname{sgn}(\mathbf U\mathbf e_j)$. Thus every nonzero coordinate of every column
$\mathbf U\mathbf e_j$ equals $\pm1$. Since each column has unit norm, it has exactly one nonzero entry;
orthogonality then makes $\mathbf U$ a signed permutation.

For $m\ge2$, the first column of each DCT-VIII block has $m$ strictly positive entries, so
$\mathbf Q$ is not a signed permutation. Failure of \eqref{eq:adam-first-step-equivariance} for some vector gives a linear
objective with that constant gradient; because the first step size is positive, the first mapped Adam
iterates differ. \hfill$\square$

\section{Internal Lemmas}\label{app:auxi_lems}
\setcounter{lemma}{0}
\renewcommand{\thelemma}{\thesection\arabic{lemma}}

This section proves the six auxiliary linear-algebra and coordinate results
used in
\Cref{app:foundational-proofs,app:new-proofs}:
the stiffness factorization and nullspace, the anchored block
decomposition, the increment and spectral Brownian-energy identities,
and the explicit spectral Brownian basis.
Their logical roles are displayed in
\Cref{fig:dependency_graph},
and their statements and locations are summarized in
\Cref{tab:auxiliary-results}.

\begin{lemma}[Structure of the stiffness matrix]
\label{lem:stiffness-matrix}

The finite-element stiffness matrix
$\mathbf K=(K_{rs})_{r,s=0}^{G}$,
with
\begin{align}
K_{rs}
=
\int_{-A}^{A}
\phi_r'(t)\phi_s'(t)
\,\mathrm dt,
\qquad
r,s=0,\ldots,G,
\end{align}
satisfies
\begin{align}
\mathbf{K}
=
\frac{1}{h}
\mathbf{D}^{\top}
\mathbf{D}.
\end{align}
Equivalently,
\begin{align}
\mathbf{K}
=
\frac{1}{h}
\begin{pmatrix}
1 & -1 & 0 & \cdots & 0 \\
-1 & 2 & -1 & \ddots & \vdots \\
0 & -1 & 2 & \ddots & 0 \\
\vdots & \ddots & \ddots & \ddots & -1 \\
0 & \cdots & 0 & -1 & 1
\end{pmatrix}.
\end{align}
Consequently, for every
$\mathbf v\in\mathbb R^{G+1}$
and
$f=\mathcal R(\mathbf v)\in\mathcal H_h$,
\begin{align}
\left\|
f
\right\|_{B,h}^{2}
=
\frac{1}{h}
\left\|
\mathbf{D}\mathbf{v}
\right\|_2^{2}
=
\frac{1}{h}
\sum_{i=0}^{G-1}
\left(
v_{i+1}-v_i
\right)^2.
\end{align}

\end{lemma}

\begin{proof}

Let
$\mathbf v=(v_0,\ldots,v_G)^{\top}\in\mathbb R^{G+1}$
be arbitrary, and set
$f=\mathcal R(\mathbf v)$.
Fix an interval index
$i\in\{0,\ldots,G-1\}$.
For every
$t\in[t_i,t_{i+1}]$,
only the two basis functions
$\phi_i$
and
$\phi_{i+1}$
can be nonzero. Moreover, on this interval,
\begin{align}
\phi_i(t)
=
\frac{t_{i+1}-t}{h},
\qquad
\phi_{i+1}(t)
=
\frac{t-t_i}{h}.
\end{align}
Therefore,
\begin{align}
f(t)
&\stackrel{(a)}{=}
\sum_{j=0}^{G}
v_j\phi_j(t)
\stackrel{(b)}{=}
v_i\phi_i(t)
+
v_{i+1}\phi_{i+1}(t)
\nonumber\\
&\stackrel{(c)}{=}
v_i
\frac{t_{i+1}-t}{h}
+
v_{i+1}
\frac{t-t_i}{h}.
\end{align}
Here (a) is the definition of the finite-element reconstruction operator,
(b) uses the local support of the nodal basis functions, and
(c) substitutes the explicit affine formulas for
$\phi_i$
and
$\phi_{i+1}$
on
$[t_i,t_{i+1}]$.
Differentiating the preceding affine expression on the open interval
$(t_i,t_{i+1})$
gives
\begin{align}
f'(t)
&\stackrel{(a)}{=}
v_i
\left(
-\frac{1}{h}
\right)
+
v_{i+1}
\left(
\frac{1}{h}
\right)
\stackrel{(b)}{=}
\frac{v_{i+1}-v_i}{h}.
\label{eq:stiffness-local-derivative}
\end{align}
Here (a) differentiates the two affine terms, and
(b) combines them over the common denominator
$h$.
The derivative may fail to exist at the finitely many grid points, but
\eqref{eq:stiffness-local-derivative}
holds almost everywhere on
$I$,
which is sufficient for every integral below.

By the definition of the discrete Brownian seminorm,
\begin{align}
\left\|f\right\|_{B,h}^{2}
&\stackrel{(a)}{=}
\int_{-A}^{A}
\left(f'(t)\right)^2
\,\mathrm dt
\stackrel{(b)}{=}
\sum_{i=0}^{G-1}
\int_{t_i}^{t_{i+1}}
\left(f'(t)\right)^2
\,\mathrm dt
\stackrel{(c)}{=}
\sum_{i=0}^{G-1}
\int_{t_i}^{t_{i+1}}
\left(
\frac{v_{i+1}-v_i}{h}
\right)^2
\,\mathrm dt
\nonumber\\
&\stackrel{(d)}{=}
\sum_{i=0}^{G-1}
\left(
\frac{v_{i+1}-v_i}{h}
\right)^2
\int_{t_i}^{t_{i+1}}
1
\,\mathrm dt
\stackrel{(e)}{=}
\sum_{i=0}^{G-1}
\left(
\frac{v_{i+1}-v_i}{h}
\right)^2
\left(t_{i+1}-t_i\right)
\nonumber\\
&\stackrel{(f)}{=}
\sum_{i=0}^{G-1}
\left(
\frac{v_{i+1}-v_i}{h}
\right)^2
h
\stackrel{(g)}{=}
\frac{1}{h}
\sum_{i=0}^{G-1}
\left(v_{i+1}-v_i\right)^2.
\label{eq:stiffness-energy}
\end{align}
Here (a) expands the squared seminorm,
(b) partitions
$[-A,A]$
into its
$G$
mesh intervals,
(c) substitutes
\eqref{eq:stiffness-local-derivative},
(d) moves the intervalwise constant factor outside each integral,
(e) evaluates the integral of the constant function
$1$,
(f) uses the uniform-grid identity
$t_{i+1}-t_i=h$,
and
(g) simplifies
$h/h^2=1/h$.

By the componentwise definition of the first-difference operator,
\begin{align}
\left\|
\mathbf D\mathbf v
\right\|_2^2
&\stackrel{(a)}{=}
\sum_{i=0}^{G-1}
\left(
\left(
\mathbf D\mathbf v
\right)_i
\right)^2
\stackrel{(b)}{=}
\sum_{i=0}^{G-1}
\left(v_{i+1}-v_i\right)^2.
\label{eq:stiffness-difference}
\end{align}
Here (a) is the definition of the squared Euclidean norm on
$\mathbb R^G$,
and (b) uses
$(\mathbf D\mathbf v)_i=v_{i+1}-v_i$.
Combining
\eqref{eq:stiffness-energy}
and
\eqref{eq:stiffness-difference}
yields
\begin{align}
\left\|f\right\|_{B,h}^{2}
=
\frac{1}{h}
\left\|\mathbf D\mathbf v\right\|_2^2
=
\frac{1}{h}
\mathbf v^{\top}
\mathbf D^{\top}
\mathbf D
\mathbf v.
\label{eq:stiffness-energy-difference-form}
\end{align}
The last equality follows from
$\|\mathbf w\|_2^2=\mathbf w^{\top}\mathbf w$
with
$\mathbf w=\mathbf D\mathbf v$,
together with
$(\mathbf D\mathbf v)^{\top}=\mathbf v^{\top}\mathbf D^{\top}$.

We next express the same energy through the stiffness matrix.
By the entrywise definition of
$\mathbf K$,
\begin{align}
\mathbf v^{\top}\mathbf K\mathbf v
&\stackrel{(a)}{=}
\sum_{r=0}^{G}
\sum_{s=0}^{G}
v_r K_{rs}v_s
\stackrel{(b)}{=}
\sum_{r=0}^{G}
\sum_{s=0}^{G}
v_rv_s
\int_{-A}^{A}
\phi_r'(t)\phi_s'(t)
\,\mathrm dt
\nonumber\\
&\stackrel{(c)}{=}
\int_{-A}^{A}
\sum_{r=0}^{G}
\sum_{s=0}^{G}
v_rv_s
\phi_r'(t)\phi_s'(t)
\,\mathrm dt
\stackrel{(d)}{=}
\int_{-A}^{A}
\left(
\sum_{r=0}^{G}
v_r\phi_r'(t)
\right)
\left(
\sum_{s=0}^{G}
v_s\phi_s'(t)
\right)
\,\mathrm dt
\nonumber\\
&\stackrel{(e)}{=}
\int_{-A}^{A}
\left(f'(t)\right)^2
\,\mathrm dt
\stackrel{(f)}{=}
\left\|f\right\|_{B,h}^{2}.
\label{eq:stiffness-energy-matrix-form}
\end{align}
Here (a) expands the quadratic form componentwise,
(b) substitutes the definition of
$K_{rs}$,
(c) interchanges the integral with two finite sums,
(d) factors the double sum,
(e) uses
$f'=\sum_{r=0}^{G}v_r\phi_r'$
almost everywhere, and
(f) is the definition of the discrete Brownian seminorm.

Combining
\eqref{eq:stiffness-energy-difference-form}
and
\eqref{eq:stiffness-energy-matrix-form}
gives
\begin{align}
\mathbf v^{\top}
\left(
\mathbf K
-
\frac{1}{h}
\mathbf D^{\top}
\mathbf D
\right)
\mathbf v
=
0,
\qquad
\forall
\mathbf v\in\mathbb R^{G+1}.
\label{eq:stiffness-zero-quadratic-form}
\end{align}
Define
\begin{align}
\mathbf M
:=
\mathbf K
-
\frac{1}{h}
\mathbf D^{\top}
\mathbf D.
\end{align}
The matrix
$\mathbf M$
is symmetric because
$K_{rs}=K_{sr}$
and
$(\mathbf D^{\top}\mathbf D)^{\top}
=
\mathbf D^{\top}\mathbf D$.
Let
$\mathbf x,\mathbf y\in\mathbb R^{G+1}$
be arbitrary. Then
\begin{align}
2\mathbf x^{\top}\mathbf M\mathbf y
&\stackrel{(a)}{=}
(\mathbf x+\mathbf y)^{\top}
\mathbf M
(\mathbf x+\mathbf y)
-
\mathbf x^{\top}\mathbf M\mathbf x
-
\mathbf y^{\top}\mathbf M\mathbf y
\stackrel{(b)}{=}
0.
\end{align}
Here (a) expands the first quadratic form and uses the symmetry identity
$\mathbf y^{\top}\mathbf M\mathbf x
=
\mathbf x^{\top}\mathbf M\mathbf y$,
while (b) applies
\eqref{eq:stiffness-zero-quadratic-form}
to
$\mathbf x+\mathbf y$,
$\mathbf x$,
and
$\mathbf y$.
Thus
$\mathbf x^{\top}\mathbf M\mathbf y=0$
for all
$\mathbf x,\mathbf y$.
Taking
$\mathbf x=\mathbf e_r$
and
$\mathbf y=\mathbf e_s$
gives
\begin{align}
M_{rs}
&\stackrel{(a)}{=}
\mathbf e_r^{\top}
\mathbf M
\mathbf e_s
\stackrel{(b)}{=}
0,
\qquad
r,s=0,\ldots,G.
\end{align}
Here (a) extracts the
$(r,s)$
entry of
$\mathbf M$,
and (b) uses the preceding bilinear identity.
Therefore
$\mathbf M=\mathbf0$,
and hence
\begin{align}
\mathbf K
=
\frac{1}{h}
\mathbf D^{\top}
\mathbf D.
\label{eq:stiffness-factorization}
\end{align}

It remains to expand
$\mathbf D^{\top}\mathbf D$.
The entries of
$\mathbf D$
are
\begin{align}
D_{ir}
=
\begin{cases}
-1,
&
 r=i,
\\
1,
&
 r=i+1,
\\
0,
&
\text{otherwise},
\end{cases}
\qquad
i=0,\ldots,G-1,
\quad
r=0,\ldots,G.
\end{align}
Consequently,
\begin{align}
\left(
\mathbf D^{\top}\mathbf D
\right)_{rs}
&\stackrel{(a)}{=}
\sum_{i=0}^{G-1}
D_{ir}D_{is}
\stackrel{(b)}{=}
\begin{cases}
1,
&
r=s\in\{0,G\},
\\
2,
&
r=s\in\{1,\ldots,G-1\},
\\
-1,
&
|r-s|=1,
\\
0,
&
|r-s|\geq2.
\end{cases}
\label{eq:stiffness-entrywise-product}
\end{align}
Here (a) is the entrywise formula for a matrix product.
For (b), an endpoint column of
$\mathbf D$
contains one nonzero entry of magnitude one, an interior column contains two
nonzero entries of magnitude one, adjacent columns overlap in exactly one row
with product
$-1$,
and nonadjacent columns have disjoint supports.
Thus
\begin{align}
\mathbf D^{\top}\mathbf D
=
\begin{pmatrix}
1 & -1 & 0 & \cdots & 0 \\
-1 & 2 & -1 & \ddots & \vdots \\
0 & -1 & 2 & \ddots & 0 \\
\vdots & \ddots & \ddots & \ddots & -1 \\
0 & \cdots & 0 & -1 & 1
\end{pmatrix}.
\end{align}
Combining this identity with
\eqref{eq:stiffness-factorization}
proves the matrix formula, while
\eqref{eq:stiffness-energy}
and
\eqref{eq:stiffness-difference}
prove the asserted energy identities.

\end{proof}

\begin{lemma}[Kernel of the stiffness matrix]
\label{lem:kernel-stiffness}

The stiffness matrix satisfies
\begin{align}
\ker
\left(
\mathbf K
\right)
=
\operatorname{span}
\left(
\mathbf1
\right),
\end{align}
where
\begin{align}
\mathbf1
:=
\left(
1,\ldots,1
\right)^{\top}
\in
\mathbb R^{G+1}.
\end{align}

\end{lemma}

\begin{proof}

By Lemma~\ref{lem:stiffness-matrix}, the stiffness matrix admits the
factorization
\begin{align}
\mathbf K
=
\frac{1}{h}
\mathbf D^{\top}
\mathbf D.
\label{eq:kernel-factorization}
\end{align}
We first prove that
\begin{align}
\ker(\mathbf K)
=
\ker(\mathbf D).
\label{eq:kernel-equal-nullspaces}
\end{align}

Let
$\mathbf v\in\ker(\mathbf K)$
be arbitrary. By the definition of the kernel,
\begin{align}
\mathbf K\mathbf v
=
\mathbf 0.
\label{eq:kernel-Kv-zero}
\end{align}
Multiplying this identity from the left by
$\mathbf v^{\top}$
and using
\eqref{eq:kernel-factorization}
gives
\begin{align}
0
&\stackrel{(a)}{=}
\mathbf v^{\top}\mathbf 0
\stackrel{(b)}{=}
\mathbf v^{\top}
\left(
\mathbf K\mathbf v
\right)
\stackrel{(c)}{=}
\mathbf v^{\top}
\mathbf K
\mathbf v
\nonumber\\
&\stackrel{(d)}{=}
\frac{1}{h}
\mathbf v^{\top}
\mathbf D^{\top}
\mathbf D
\mathbf v
\stackrel{(e)}{=}
\frac{1}{h}
\left(
\mathbf D\mathbf v
\right)^{\top}
\left(
\mathbf D\mathbf v
\right)
\stackrel{(f)}{=}
\frac{1}{h}
\left\|
\mathbf D\mathbf v
\right\|_2^2.
\label{eq:kernel-zero-energy}
\end{align}
Here (a) uses the identity
$\mathbf v^{\top}\mathbf0=0$,
(b) substitutes
\eqref{eq:kernel-Kv-zero},
(c) uses the associativity of matrix multiplication,
(d) substitutes the factorization
\eqref{eq:kernel-factorization},
(e) uses
$\mathbf v^{\top}\mathbf D^{\top}
=(\mathbf D\mathbf v)^{\top}$,
and (f) applies the Euclidean identity
$\mathbf w^{\top}\mathbf w=\|\mathbf w\|_2^2$
with
$\mathbf w=\mathbf D\mathbf v$.
Since
$A>0$
and
$G\geq1$,
the mesh size
$h=2A/G$
satisfies
$h>0$.
Multiplying
\eqref{eq:kernel-zero-energy}
by
$h$
therefore yields
\begin{align}
\left\|
\mathbf D\mathbf v
\right\|_2^2
=
0.
\label{eq:kernel-Dv-zero-norm}
\end{align}
By the componentwise definition of the Euclidean norm,
\begin{align}
0
&\stackrel{(a)}{=}
\left\|
\mathbf D\mathbf v
\right\|_2^2
\stackrel{(b)}{=}
\sum_{i=0}^{G-1}
\left(
(\mathbf D\mathbf v)_i
\right)^2.
\label{eq:kernel-sum-of-squares}
\end{align}
Here (a) is
\eqref{eq:kernel-Dv-zero-norm},
and (b) is the definition of the squared Euclidean norm on
$\mathbb R^G$.
Fix
$i\in\{0,\ldots,G-1\}$.
Since every real square is nonnegative, we have
\begin{align}
0
&\stackrel{(a)}{\leq}
\left(
(\mathbf D\mathbf v)_i
\right)^2
\stackrel{(b)}{\leq}
\sum_{r=0}^{G-1}
\left(
(\mathbf D\mathbf v)_r
\right)^2
\stackrel{(c)}{=}
0.
\label{eq:kernel-component-sandwich}
\end{align}
Here (a) uses the nonnegativity of the square of a real number,
(b) uses the fact that the sum contains the displayed term and that all of its
remaining terms are nonnegative, and
(c) applies
\eqref{eq:kernel-sum-of-squares}.
The two inequalities in
\eqref{eq:kernel-component-sandwich}
force
$((\mathbf D\mathbf v)_i)^2=0$.
A real number has square zero if and only if the number itself is zero, so
\begin{align}
(\mathbf D\mathbf v)_i
=
0.
\end{align}
Since
$i$
was arbitrary, every component of
$\mathbf D\mathbf v$
is zero. It follows that
\begin{align}
\mathbf D\mathbf v
=
\mathbf0.
\label{eq:kernel-Dv-zero}
\end{align}
Thus
$\mathbf v\in\ker(\mathbf D)$,
and consequently
\begin{align}
\ker(\mathbf K)
\subseteq
\ker(\mathbf D).
\label{eq:kernel-first-inclusion}
\end{align}

Conversely, let
$\mathbf v\in\ker(\mathbf D)$
be arbitrary. Then
\begin{align}
\mathbf D\mathbf v
=
\mathbf0.
\label{eq:kernel-Dv-zero-converse}
\end{align}
Using
\eqref{eq:kernel-factorization},
we obtain
\begin{align}
\mathbf K\mathbf v
&\stackrel{(a)}{=}
\frac{1}{h}
\mathbf D^{\top}
\mathbf D
\mathbf v
\stackrel{(b)}{=}
\frac{1}{h}
\mathbf D^{\top}
\left(
\mathbf D\mathbf v
\right)
\stackrel{(c)}{=}
\frac{1}{h}
\mathbf D^{\top}
\mathbf0
\stackrel{(d)}{=}
\mathbf0.
\label{eq:kernel-Kv-zero-converse}
\end{align}
Here (a) substitutes
\eqref{eq:kernel-factorization},
(b) uses the associativity of matrix multiplication,
(c) substitutes
\eqref{eq:kernel-Dv-zero-converse},
and (d) uses
$\mathbf D^{\top}\mathbf0=\mathbf0$
and
$(1/h)\mathbf0=\mathbf0$.
Therefore
$\mathbf v\in\ker(\mathbf K)$,
and hence
\begin{align}
\ker(\mathbf D)
\subseteq
\ker(\mathbf K).
\label{eq:kernel-second-inclusion}
\end{align}
Combining
\eqref{eq:kernel-first-inclusion}
and
\eqref{eq:kernel-second-inclusion}
proves
\eqref{eq:kernel-equal-nullspaces}.

It remains to characterize
$\ker(\mathbf D)$.
Let
$\mathbf v=(v_0,\ldots,v_G)^{\top}\in\ker(\mathbf D)$.
Then
\begin{align}
\mathbf D\mathbf v
=
\mathbf0.
\label{eq:kernel-D-characterization-start}
\end{align}
For every
$i\in\{0,\ldots,G-1\}$,
we therefore have
\begin{align}
0
&\stackrel{(a)}{=}
(\mathbf D\mathbf v)_i
\stackrel{(b)}{=}
 v_{i+1}-v_i.
\label{eq:kernel-adjacent-equality}
\end{align}
Here (a) takes the
$i$th component of
\eqref{eq:kernel-D-characterization-start},
and (b) uses the definition of the first-difference operator.
Rearranging
\eqref{eq:kernel-adjacent-equality}
gives
\begin{align}
v_{i+1}
=
v_i,
\qquad
 i=0,\ldots,G-1.
\label{eq:kernel-equal-neighbors}
\end{align}
We now show by induction that
\begin{align}
v_j
=
v_0,
\qquad
j=0,\ldots,G.
\label{eq:kernel-all-components-equal}
\end{align}
The statement is immediate for
$j=0$.
Assume that it holds for some
$j\in\{0,\ldots,G-1\}$,
so that
$v_j=v_0$.
Applying
\eqref{eq:kernel-equal-neighbors}
with
$i=j$
gives
\begin{align}
v_{j+1}
&\stackrel{(a)}{=}
v_j
\stackrel{(b)}{=}
v_0.
\end{align}
Here (a) is
\eqref{eq:kernel-equal-neighbors},
and (b) is the induction hypothesis.
Thus
\eqref{eq:kernel-all-components-equal}
holds for every
$j=0,\ldots,G$.
Setting
$c:=v_0$
and recalling that
$\mathbf1=(1,\ldots,1)^{\top}\in\mathbb R^{G+1}$,
we obtain
\begin{align}
\mathbf v
&\stackrel{(a)}{=}
\begin{pmatrix}
v_0\\
\vdots\\
v_G
\end{pmatrix}
\stackrel{(b)}{=}
\begin{pmatrix}
c\\
\vdots\\
c
\end{pmatrix}
\stackrel{(c)}{=}
c\mathbf1.
\label{eq:kernel-vector-is-constant}
\end{align}
Here (a) writes
$\mathbf v$
componentwise,
(b) uses
\eqref{eq:kernel-all-components-equal},
and (c) factors out the common scalar
$c$.
Therefore every vector in
$\ker(\mathbf D)$
belongs to
$\operatorname{span}(\mathbf1)$, and hence
\begin{align}
\ker(\mathbf D)
\subseteq
\operatorname{span}(\mathbf1).
\label{eq:kernel-D-subset-constants}
\end{align}

For the reverse inclusion, let
$\mathbf v\in\operatorname{span}(\mathbf1)$.
By the definition of the linear span, there exists
$c\in\mathbb R$
such that
\begin{align}
\mathbf v
=
c\mathbf1.
\label{eq:kernel-constant-vector}
\end{align}
For every
$i\in\{0,\ldots,G-1\}$,
the corresponding component of
$\mathbf D\mathbf v$
satisfies
\begin{align}
(\mathbf D\mathbf v)_i
&\stackrel{(a)}{=}
v_{i+1}-v_i
\stackrel{(b)}{=}
c-c
\stackrel{(c)}{=}
0.
\end{align}
Here (a) uses the definition of
$\mathbf D$,
(b) uses
\eqref{eq:kernel-constant-vector},
and (c) simplifies the scalar difference.
Thus every component of
$\mathbf D\mathbf v$
is zero, so
\begin{align}
\mathbf D\mathbf v
=
\mathbf0.
\end{align}
Therefore
$\mathbf v\in\ker(\mathbf D)$,
which proves
\begin{align}
\operatorname{span}(\mathbf1)
\subseteq
\ker(\mathbf D).
\label{eq:kernel-constants-subset-D}
\end{align}
Combining
\eqref{eq:kernel-D-subset-constants}
and
\eqref{eq:kernel-constants-subset-D}
yields
\begin{align}
\ker(\mathbf D)
=
\operatorname{span}(\mathbf1).
\label{eq:kernel-D-final}
\end{align}
Finally, combining
\eqref{eq:kernel-equal-nullspaces}
and
\eqref{eq:kernel-D-final}
gives
\begin{align}
\ker(\mathbf K)
=
\ker(\mathbf D)
=
\operatorname{span}(\mathbf1),
\end{align}
which proves the claim.

\end{proof}

\begin{lemma}[Block decomposition]
\label{lem:block-decomposition}

Assume that $G=2m$, so that the anchor is located at
$t_{i_0}=0$ with $i_0=m$, and use the reduced anchored ordering in
\eqref{eq:reduced-anchored-ordering}. Then the anchored stiffness matrix
satisfies
\begin{align}
\mathbf K_0
=
\frac{1}{h}
\begin{pmatrix}
\mathbf T_m
&
\mathbf 0
\\
\mathbf 0
&
\mathbf T_m
\end{pmatrix},
\label{eq:block-decomposition-claim}
\end{align}
where
$\mathbf T_m=(\tau_{rs})_{r,s=0}^{m-1}\in\mathbb R^{m\times m}$
is defined entrywise by
\begin{align}
\tau_{rs}
:=
\begin{cases}
1,
&
 r=s=0,
\\
2,
&
 r=s\in\{1,\ldots,m-1\},
\\
-1,
&
 |r-s|=1,
\\
0,
&
 \text{otherwise}.
\end{cases}
\label{eq:Tm-entrywise}
\end{align}
Equivalently, for $m\geq2$,
\begin{align}
\mathbf T_m
=
\begin{pmatrix}
1 & -1 & 0 & \cdots & 0 \\
-1 & 2 & -1 & \ddots & \vdots \\
0 & -1 & 2 & \ddots & 0 \\
\vdots & \ddots & \ddots & \ddots & -1 \\
0 & \cdots & 0 & -1 & 2
\end{pmatrix},
\end{align}
whereas $\mathbf T_1=(1)$.

\end{lemma}

\begin{proof}

Let
$\widetilde{\mathbf v}\in\mathbb R^{2m}$
be arbitrary, and partition it as
\begin{align}
\widetilde{\mathbf v}
=
\begin{pmatrix}
\mathbf a
\\
\mathbf b
\end{pmatrix},
\qquad
\mathbf a
=
(a_0,\ldots,a_{m-1})^{\top},
\qquad
\mathbf b
=
(b_0,\ldots,b_{m-1})^{\top}.
\label{eq:block-partition-reduced-vector}
\end{align}
Set
\begin{align}
\mathbf v
:=
\mathbf R\widetilde{\mathbf v}
\in
\mathbb R^{2m+1}.
\end{align}
By the reduced-coordinate ordering in
\eqref{eq:reduced-anchored-ordering}, the components of the three vectors
satisfy
\begin{align}
a_j
=
\widetilde v_j
=
v_j,
\qquad
b_j
=
\widetilde v_{m+j}
=
v_{2m-j},
\qquad
j=0,\ldots,m-1,
\label{eq:block-coordinate-identification}
\end{align}
and the inserted anchored component satisfies
\begin{align}
v_m
=
0.
\label{eq:block-anchor-zero}
\end{align}

Define
$\mathbf E_m=(E_{rs})_{r,s=0}^{m-1}\in\mathbb R^{m\times m}$
entrywise by
\begin{align}
E_{rs}
:=
\begin{cases}
-1,
&
 s=r,
\\
1,
&
 s=r+1
\text{ and }
r\in\{0,\ldots,m-2\},
\\
0,
&
 \text{otherwise},
\end{cases}
\label{eq:block-half-difference-matrix}
\end{align}
and let
$\mathbf J_m=(J_{rs})_{r,s=0}^{m-1}$
denote the reversal permutation matrix defined by
\begin{align}
J_{rs}
:=
\begin{cases}
1,
&
 r+s=m-1,
\\
0,
&
 \text{otherwise}.
\end{cases}
\label{eq:block-reversal-matrix}
\end{align}
The definition of
$\mathbf E_m$
gives, for every
$\mathbf x=(x_0,\ldots,x_{m-1})^{\top}\in\mathbb R^m$,
\begin{align}
(\mathbf E_m\mathbf x)_r
=
\begin{cases}
x_{r+1}-x_r,
&
 r=0,\ldots,m-2,
\\
-x_{m-1},
&
 r=m-1.
\end{cases}
\label{eq:block-E-action}
\end{align}

We first compute the action of
$\mathbf D\mathbf R$
on the left half-grid. For
$i=0,\ldots,m-1$,
\begin{align}
(\mathbf D\mathbf R\widetilde{\mathbf v})_i
&\stackrel{(a)}{=}
(\mathbf D\mathbf v)_i
\stackrel{(b)}{=}
v_{i+1}-v_i
\nonumber\\
&\stackrel{(c)}{=}
\begin{cases}
a_{i+1}-a_i,
&
 i=0,\ldots,m-2,
\\
-a_{m-1},
&
 i=m-1,
\end{cases}
\stackrel{(d)}{=}
(\mathbf E_m\mathbf a)_i.
\label{eq:block-left-action}
\end{align}
Here (a) uses
$\mathbf v=\mathbf R\widetilde{\mathbf v}$,
(b) uses the componentwise definition of
$\mathbf D$,
(c) uses
\eqref{eq:block-coordinate-identification}
and
\eqref{eq:block-anchor-zero},
and (d) follows from
\eqref{eq:block-E-action}.

We next compute the right-half action. By
\eqref{eq:block-reversal-matrix}, for every
$s=0,\ldots,m-1$,
\begin{align}
(-\mathbf J_m\mathbf E_m\mathbf b)_s
&\stackrel{(a)}{=}
-(\mathbf E_m\mathbf b)_{m-1-s}
\nonumber\\
&\stackrel{(b)}{=}
\begin{cases}
b_{m-1},
&
 s=0,
\\
b_{m-s-1}-b_{m-s},
&
 s=1,\ldots,m-1.
\end{cases}
\label{eq:block-reversed-E-action}
\end{align}
Here (a) uses the fact that
$\mathbf J_m$
reverses the order of the components, and (b) applies
\eqref{eq:block-E-action}: when
$s=0$, the index is
$m-1$, whereas when
$s\geq1$, the index
$m-1-s$
belongs to
$\{0,\ldots,m-2\}$.
On the other hand, for
$s=0,\ldots,m-1$,
\begin{align}
(\mathbf D\mathbf R\widetilde{\mathbf v})_{m+s}
&\stackrel{(a)}{=}
(\mathbf D\mathbf v)_{m+s}
\stackrel{(b)}{=}
v_{m+s+1}-v_{m+s}
\nonumber\\
&\stackrel{(c)}{=}
\begin{cases}
b_{m-1},
&
 s=0,
\\
b_{m-s-1}-b_{m-s},
&
 s=1,\ldots,m-1,
\end{cases}
\stackrel{(d)}{=}
(-\mathbf J_m\mathbf E_m\mathbf b)_s.
\label{eq:block-right-action}
\end{align}
Here (a) again uses
$\mathbf v=\mathbf R\widetilde{\mathbf v}$,
(b) uses the definition of
$\mathbf D$,
(c) uses
$v_m=0$
and the identity
$v_{m+r}=b_{m-r}$
for
$r=1,\ldots,m$,
which follows from
\eqref{eq:block-coordinate-identification},
and (d) invokes
\eqref{eq:block-reversed-E-action}.

Combining
\eqref{eq:block-left-action}
and
\eqref{eq:block-right-action}
yields
\begin{align}
\mathbf D\mathbf R\widetilde{\mathbf v}
&\stackrel{(a)}{=}
\begin{pmatrix}
\mathbf E_m\mathbf a
\\
-\mathbf J_m\mathbf E_m\mathbf b
\end{pmatrix}
\stackrel{(b)}{=}
\begin{pmatrix}
\mathbf E_m
&
\mathbf0
\\
\mathbf0
&
-\mathbf J_m\mathbf E_m
\end{pmatrix}
\begin{pmatrix}
\mathbf a
\\
\mathbf b
\end{pmatrix}
\nonumber\\
&\stackrel{(c)}{=}
\begin{pmatrix}
\mathbf E_m
&
\mathbf0
\\
\mathbf0
&
-\mathbf J_m\mathbf E_m
\end{pmatrix}
\widetilde{\mathbf v}.
\label{eq:block-DR-action}
\end{align}
Here (a) concatenates the first
$m$
components from
\eqref{eq:block-left-action}
and the last
$m$
components from
\eqref{eq:block-right-action},
(b) is block-matrix multiplication, and (c) uses
\eqref{eq:block-partition-reduced-vector}.
Because
$\widetilde{\mathbf v}\in\mathbb R^{2m}$
was arbitrary,
\eqref{eq:block-DR-action}
implies the matrix identity
\begin{align}
\mathbf D\mathbf R
=
\begin{pmatrix}
\mathbf E_m
&
\mathbf0
\\
\mathbf0
&
-\mathbf J_m\mathbf E_m
\end{pmatrix}.
\label{eq:block-DR-matrix}
\end{align}

The reversal matrix is orthogonal. Indeed, for every
$r,s\in\{0,\ldots,m-1\}$,
\begin{align}
(\mathbf J_m^{\top}\mathbf J_m)_{rs}
&\stackrel{(a)}{=}
\sum_{q=0}^{m-1}
J_{qr}J_{qs}
\stackrel{(b)}{=}
\begin{cases}
1,
&
 r=s,
\\
0,
&
 r\neq s,
\end{cases}
\stackrel{(c)}{=}
(\mathbf I_m)_{rs}.
\label{eq:block-J-orthogonal-entrywise}
\end{align}
Here (a) is the componentwise formula for a matrix product. For (b), the
factor
$J_{qr}$
is nonzero only for
$q=m-1-r$,
whereas
$J_{qs}$
is nonzero only for
$q=m-1-s$;
these two row indices coincide exactly when
$r=s$.
Finally, (c) is the entrywise definition of the identity matrix. Since the
entries agree for every
$r,s$,
\begin{align}
\mathbf J_m^{\top}\mathbf J_m
=
\mathbf I_m.
\label{eq:block-J-orthogonal}
\end{align}

We now apply the stiffness factorization from
Lemma~\ref{lem:stiffness-matrix}. The anchored stiffness matrix satisfies
\begin{align}
\mathbf K_0
&\stackrel{(a)}{=}
\mathbf R^{\top}\mathbf K\mathbf R
\stackrel{(b)}{=}
\frac1h
\mathbf R^{\top}\mathbf D^{\top}\mathbf D\mathbf R
\stackrel{(c)}{=}
\frac1h
(\mathbf D\mathbf R)^{\top}
(\mathbf D\mathbf R)
\nonumber\\
&\stackrel{(d)}{=}
\frac1h
\begin{pmatrix}
\mathbf E_m^{\top}
&
\mathbf0
\\
\mathbf0
&
-\mathbf E_m^{\top}\mathbf J_m^{\top}
\end{pmatrix}
\begin{pmatrix}
\mathbf E_m
&
\mathbf0
\\
\mathbf0
&
-\mathbf J_m\mathbf E_m
\end{pmatrix}
\nonumber\\
&\stackrel{(e)}{=}
\frac1h
\begin{pmatrix}
\mathbf E_m^{\top}\mathbf E_m
&
\mathbf0
\\
\mathbf0
&
\mathbf E_m^{\top}
\mathbf J_m^{\top}\mathbf J_m
\mathbf E_m
\end{pmatrix}
\stackrel{(f)}{=}
\frac1h
\begin{pmatrix}
\mathbf E_m^{\top}\mathbf E_m
&
\mathbf0
\\
\mathbf0
&
\mathbf E_m^{\top}\mathbf E_m
\end{pmatrix}.
\label{eq:block-K0-factorized}
\end{align}
Here (a) is the definition of
$\mathbf K_0$,
(b) substitutes
$\mathbf K=\frac1h\mathbf D^{\top}\mathbf D$
from
Lemma~\ref{lem:stiffness-matrix},
(c) uses
$(\mathbf D\mathbf R)^{\top}
=
\mathbf R^{\top}\mathbf D^{\top}$,
(d) substitutes
\eqref{eq:block-DR-matrix}
and uses
$(-\mathbf J_m\mathbf E_m)^{\top}
=-\mathbf E_m^{\top}\mathbf J_m^{\top}$,
(e) performs the block-matrix multiplication, and (f) uses
\eqref{eq:block-J-orthogonal}.

It remains to identify
$\mathbf E_m^{\top}\mathbf E_m$.
For
$r,s\in\{0,\ldots,m-1\}$,
\begin{align}
(\mathbf E_m^{\top}\mathbf E_m)_{rs}
&\stackrel{(a)}{=}
\sum_{q=0}^{m-1}
E_{qr}E_{qs}
\stackrel{(b)}{=}
\begin{cases}
1,
&
 r=s=0,
\\
2,
&
 r=s\in\{1,\ldots,m-1\},
\\
-1,
&
 |r-s|=1,
\\
0,
&
 \text{otherwise},
\end{cases}
\stackrel{(c)}{=}
\tau_{rs}.
\label{eq:block-EtE-entrywise}
\end{align}
Here (a) is the componentwise formula for matrix multiplication. To justify
(b), observe from
\eqref{eq:block-half-difference-matrix}
that column zero has the single nonzero entry
$E_{00}=-1$,
so its squared Euclidean norm is one. Every column
$r\in\{1,\ldots,m-1\}$
has exactly the two nonzero entries
$E_{r-1,r}=1$
and
$E_{rr}=-1$,
so its squared Euclidean norm is
$1^2+(-1)^2=2$.
Two consecutive columns have exactly one common nonzero row: if
$s=r+1$, their product in row
$r$
is
$E_{rr}E_{r,r+1}=(-1)(1)=-1$,
and the symmetric case
$r=s+1$
is identical. Columns whose indices differ by at least two have no common
nonzero row, so their inner product is zero. Finally, (c) uses the definition
of
$\tau_{rs}$
in
\eqref{eq:Tm-entrywise}.
Since the entries agree for all
$r,s$,
\begin{align}
\mathbf E_m^{\top}\mathbf E_m
=
\mathbf T_m.
\label{eq:block-EtE-T}
\end{align}
Substituting
\eqref{eq:block-EtE-T}
into
\eqref{eq:block-K0-factorized}
gives
\begin{align}
\mathbf K_0
=
\frac1h
\begin{pmatrix}
\mathbf T_m
&
\mathbf0
\\
\mathbf0
&
\mathbf T_m
\end{pmatrix},
\end{align}
which is exactly
\eqref{eq:block-decomposition-claim}
and proves the lemma.

\end{proof}

\begin{lemma}[Brownian energy in increment coordinates]
\label{lem:increment-energy}

Let
$f\in\mathcal H_h^0$,
and let
$\widetilde{\mathbf v}\in\mathbb R^G$
be its reduced anchored nodal coefficient vector. Define the corresponding full
anchored nodal vector and increment vector by
\begin{align}
\mathbf v
&:=
\mathbf R\widetilde{\mathbf v}
\in
\mathbb R^{G+1},
\label{eq:increment-energy-full-vector}
\\
\mathbf{w}
&:=
\mathbf D\mathbf v
=
\mathbf D\mathbf R\widetilde{\mathbf v}
=
(w_0,\ldots,w_{G-1})^{\top}
\in
\mathbb R^G,
\label{eq:increment-energy-increment-vector}
\end{align}
so that
$f=\mathcal R(\mathbf v)
=\mathcal R\left(\mathbf R\widetilde{\mathbf v}\right)$.
Then
\begin{align}
\left\|f\right\|_{B,h}^{2}
=
\frac{1}{h}
\left\|\mathbf{w}\right\|_2^{2}
=
\frac{1}{h}
\sum_{i=0}^{G-1}
w_i^{2}.
\label{eq:increment-energy-claim}
\end{align}

\end{lemma}

\begin{proof}

By
\eqref{eq:increment-energy-full-vector},
$\mathbf v\in\mathbb R^{G+1}$,
and by
\eqref{eq:increment-energy-increment-vector},
$\mathbf{w}\in\mathbb R^G$.
For every
$i\in\{0,\ldots,G-1\}$,
the $i$th increment satisfies
\begin{align}
w_i
&\stackrel{(a)}{=}
(\mathbf D\mathbf v)_i
\stackrel{(b)}{=}
v_{i+1}-v_i.
\label{eq:increment-energy-component}
\end{align}
Here (a) takes the $i$th component of the defining identity
$\mathbf{w}=\mathbf D\mathbf v$
in
\eqref{eq:increment-energy-increment-vector},
and (b) uses the definition of the first-difference operator
$\mathbf D$.

Since
$f=\mathcal R(\mathbf v)$,
Proposition~\ref{prop:finite-brownian-energy} gives both the quadratic-form
representation of the Brownian energy and the factorization of the stiffness
matrix. Therefore,
\begin{align}
\left\|f\right\|_{B,h}^{2}
&\stackrel{(a)}{=}
\mathbf v^{\top}\mathbf K\mathbf v
\stackrel{(b)}{=}
\mathbf v^{\top}
\left(
\frac{1}{h}
\mathbf D^{\top}\mathbf D
\right)
\mathbf v
\stackrel{(c)}{=}
\frac{1}{h}
\mathbf v^{\top}
\mathbf D^{\top}\mathbf D
\mathbf v
\stackrel{(d)}{=}
\frac{1}{h}
\left(
\mathbf D\mathbf v
\right)^{\top}
\left(
\mathbf D\mathbf v
\right)
\nonumber\\
&\stackrel{(e)}{=}
\frac{1}{h}
\mathbf{w}^{\top}
\mathbf{w}
\stackrel{(f)}{=}
\frac{1}{h}
\left\|\mathbf{w}\right\|_2^{2}
\stackrel{(g)}{=}
\frac{1}{h}
\sum_{i=0}^{G-1}
w_i^{2}.
\label{eq:increment-energy-proof-chain}
\end{align}
Here (a) is the Brownian-energy quadratic-form identity established in
Proposition~\ref{prop:finite-brownian-energy};
(b) substitutes the stiffness factorization
$\mathbf K=\frac{1}{h}\mathbf D^{\top}\mathbf D$
from the same proposition;
(c) moves the scalar factor
$1/h$
outside the matrix product;
(d) uses
$(\mathbf D\mathbf v)^{\top}=\mathbf v^{\top}\mathbf D^{\top}$
and the associativity of matrix multiplication;
(e) substitutes
$\mathbf{w}=\mathbf D\mathbf v$
from
\eqref{eq:increment-energy-increment-vector};
(f) uses the definition
$\|\mathbf x\|_2^2=\mathbf x^{\top}\mathbf x$
for vectors in Euclidean space; and
(g) expands that norm componentwise using
$\mathbf{w}=(w_0,\ldots,w_{G-1})^{\top}$.
Thus
\eqref{eq:increment-energy-proof-chain}
is exactly
\eqref{eq:increment-energy-claim},
which proves the lemma.

\end{proof}
\begin{lemma}[Brownian energy in spectral coordinates]
\label{lem:spectral-energy}

Let
$f\in\mathcal H_h^0$,
and let
$\widetilde{\mathbf v}\in\mathbb R^G$
be its reduced anchored nodal coefficient vector, so that
\begin{align}
f
=
\mathcal R\left(\mathbf R\widetilde{\mathbf v}\right).
\label{eq:spectral-energy-reconstruction}
\end{align}
Then the anchored stiffness matrix
\begin{align}
\mathbf K_0
=
\mathbf R^{\top}\mathbf K\mathbf R
\end{align}
is symmetric positive definite. Consequently, it admits an orthogonal
eigendecomposition
\begin{align}
\mathbf K_0
=
\mathbf Q\boldsymbol{\Lambda}\mathbf Q^{\top},
\qquad
\boldsymbol{\Lambda}
=
\operatorname{diag}\left(\lambda_1,\ldots,\lambda_G\right),
\qquad
\lambda_i>0,
\quad
 i=1,\ldots,G.
\end{align}
where
$\mathbf Q\in\mathbb R^{G\times G}$
is orthogonal. Define
\begin{align}
\mathbf c
:=
\mathbf Q^{\top}\widetilde{\mathbf v}
=
\left(c_1,\ldots,c_G\right)^{\top}.
\label{eq:spectral-energy-coordinate-definition}
\end{align}
Then
\begin{align}
\left\|f\right\|_{B,h}^{2}
=
\mathbf c^{\top}\boldsymbol{\Lambda}\mathbf c
=
\sum_{i=1}^{G}\lambda_i c_i^2.
\label{eq:spectral-energy-claim}
\end{align}

\end{lemma}

\subsection{Proof of \texorpdfstring{\Cref{lem:spectral-energy}}{Lemma C5}}
\label{app:proof-spectral-energy}

\begin{proof}

Set
\begin{align}
\mathbf v
:=
\mathbf R\widetilde{\mathbf v}
\in
\mathbb R^{G+1}.
\label{eq:spectral-energy-full-vector}
\end{align}
By
\eqref{eq:spectral-energy-reconstruction}
and
\eqref{eq:spectral-energy-full-vector},
\begin{align}
f
=
\mathcal R(\mathbf v).
\label{eq:spectral-energy-f-from-v}
\end{align}

We first verify that
$\mathbf K_0$
is symmetric positive definite. Write the anchoring reconstruction matrix as
\begin{align}
\mathbf R
=
\left[
\mathbf r_1,
\ldots,
\mathbf r_G
\right],
\end{align}
where
$\mathbf r_1,\ldots,\mathbf r_G\in\mathbb R^{G+1}$
are its columns. By the explicit construction in
\eqref{eq:anchoring-reconstruction-matrix},
these columns are pairwise distinct canonical basis vectors of
$\mathbb R^{G+1}$.
Therefore, for every
$r,s\in[G]$,
\begin{align}
\left(
\mathbf R^{\top}\mathbf R
\right)_{rs}
&\stackrel{(a)}{=}
\mathbf r_r^{\top}\mathbf r_s
\stackrel{(b)}{=}
\delta_{rs}
\stackrel{(c)}{=}
\left(\mathbf I_G\right)_{rs}.
\label{eq:spectral-energy-RtR-entrywise}
\end{align}
Here (a) is the entrywise rule for the product
$\mathbf R^{\top}\mathbf R$,
(b) uses the orthonormality of pairwise distinct canonical basis vectors, and
(c) is the entrywise definition of the identity matrix. Since
\eqref{eq:spectral-energy-RtR-entrywise}
holds for every
$r,s\in[G]$,
\begin{align}
\mathbf R^{\top}\mathbf R
=
\mathbf I_G.
\label{eq:spectral-energy-R-isometry}
\end{align}

By Proposition~\ref{prop:finite-brownian-energy},
\begin{align}
\mathbf K
=
\frac1h\mathbf D^{\top}\mathbf D.
\end{align}
Hence
\begin{align}
\mathbf K^{\top}
&\stackrel{(a)}{=}
\left(
\frac1h\mathbf D^{\top}\mathbf D
\right)^{\top}
\stackrel{(b)}{=}
\frac1h
\mathbf D^{\top}
\left(\mathbf D^{\top}\right)^{\top}
\stackrel{(c)}{=}
\frac1h\mathbf D^{\top}\mathbf D
\stackrel{(d)}{=}
\mathbf K.
\label{eq:spectral-energy-K-symmetric}
\end{align}
Here (a) substitutes the stiffness factorization,
(b) uses
$(\alpha\mathbf A)^{\top}=\alpha\mathbf A^{\top}$
for real
$\alpha$
and
$(\mathbf A\mathbf B)^{\top}=\mathbf B^{\top}\mathbf A^{\top}$,
(c) uses
$(\mathbf D^{\top})^{\top}=\mathbf D$,
and (d) substitutes the stiffness factorization again. It follows that
\begin{align}
\mathbf K_0^{\top}
&\stackrel{(a)}{=}
\left(
\mathbf R^{\top}\mathbf K\mathbf R
\right)^{\top}
\stackrel{(b)}{=}
\mathbf R^{\top}\mathbf K^{\top}\mathbf R
\stackrel{(c)}{=}
\mathbf R^{\top}\mathbf K\mathbf R
\stackrel{(d)}{=}
\mathbf K_0.
\label{eq:spectral-energy-K0-symmetric}
\end{align}
Here (a) uses the definition of
$\mathbf K_0$,
(b) applies the transpose-of-a-product rule,
(c) uses
\eqref{eq:spectral-energy-K-symmetric},
and (d) again uses the definition of
$\mathbf K_0$.
Thus
$\mathbf K_0$
is symmetric.

To prove positive definiteness, let
$\mathbf z\in\mathbb R^G\setminus\{\mathbf0\}$
be arbitrary and define
\begin{align}
\mathbf y
:=
\mathbf R\mathbf z
\in
\mathbb R^{G+1}.
\label{eq:spectral-energy-y-definition}
\end{align}
Using
\eqref{eq:spectral-energy-R-isometry},
we obtain
\begin{align}
\left\|\mathbf y\right\|_2^2
&\stackrel{(a)}{=}
\left(
\mathbf R\mathbf z
\right)^{\top}
\left(
\mathbf R\mathbf z
\right)
\stackrel{(b)}{=}
\mathbf z^{\top}\mathbf R^{\top}\mathbf R\mathbf z
\nonumber\\
&\stackrel{(c)}{=}
\mathbf z^{\top}\mathbf I_G\mathbf z
\stackrel{(d)}{=}
\left\|\mathbf z\right\|_2^2
\stackrel{(e)}{>}
0.
\label{eq:spectral-energy-y-nonzero}
\end{align}
Here (a) uses
$\|\mathbf y\|_2^2=\mathbf y^{\top}\mathbf y$
and
\eqref{eq:spectral-energy-y-definition},
(b) uses
$(\mathbf R\mathbf z)^{\top}=\mathbf z^{\top}\mathbf R^{\top}$
and associativity,
(c) substitutes
\eqref{eq:spectral-energy-R-isometry},
(d) uses
$\mathbf I_G\mathbf z=\mathbf z$
and the definition of the Euclidean norm, and
(e) follows from
$\mathbf z\neq\mathbf0$.
Consequently,
\begin{align}
\mathbf y
\neq
\mathbf0.
\label{eq:spectral-energy-y-is-nonzero}
\end{align}
Moreover, the construction of
$\mathbf R$
inserts a zero at the anchor index, and therefore
\begin{align}
y_{i_0}
=
0.
\label{eq:spectral-energy-y-anchored}
\end{align}

The quadratic form of
$\mathbf K_0$
at
$\mathbf z$
satisfies
\begin{align}
\mathbf z^{\top}\mathbf K_0\mathbf z
&\stackrel{(a)}{=}
\mathbf z^{\top}\mathbf R^{\top}\mathbf K\mathbf R\mathbf z
\stackrel{(b)}{=}
\left(
\mathbf R\mathbf z
\right)^{\top}
\mathbf K
\left(
\mathbf R\mathbf z
\right)
\nonumber\\
&\stackrel{(c)}{=}
\mathbf y^{\top}\mathbf K\mathbf y
\stackrel{(d)}{=}
\frac1h
\left\|
\mathbf D\mathbf y
\right\|_2^2
\stackrel{(e)}{\geq}
0.
\label{eq:spectral-energy-K0-nonnegative}
\end{align}
Here (a) substitutes
$\mathbf K_0=\mathbf R^{\top}\mathbf K\mathbf R$,
(b) uses
$\mathbf z^{\top}\mathbf R^{\top}=(\mathbf R\mathbf z)^{\top}$
and associativity,
(c) substitutes
\eqref{eq:spectral-energy-y-definition},
(d) applies Proposition~\ref{prop:finite-brownian-energy} to
$\mathbf y$, and
(e) uses
$h>0$
and the nonnegativity of a squared Euclidean norm.

Suppose, for contradiction, that
\begin{align}
\mathbf z^{\top}\mathbf K_0\mathbf z
=
0.
\label{eq:spectral-energy-zero-assumption}
\end{align}
Combining
\eqref{eq:spectral-energy-K0-nonnegative}
and
\eqref{eq:spectral-energy-zero-assumption}
gives
\begin{align}
0
&\stackrel{(a)}{=}
\frac1h
\left\|
\mathbf D\mathbf y
\right\|_2^2
\stackrel{(b)}{\Longrightarrow}
\left\|
\mathbf D\mathbf y
\right\|_2^2
=
0
\stackrel{(c)}{\Longrightarrow}
\mathbf D\mathbf y
=
\mathbf0.
\label{eq:spectral-energy-Dy-zero}
\end{align}
Here (a) uses
\eqref{eq:spectral-energy-K0-nonnegative}
with equality as imposed by
\eqref{eq:spectral-energy-zero-assumption},
(b) multiplies by the positive scalar
$h$, and
(c) uses the fact that a vector has zero squared Euclidean norm if and only if
it is the zero vector. By Lemma~\ref{lem:kernel-stiffness},
\eqref{eq:spectral-energy-Dy-zero}
implies that there exists
$a\in\mathbb R$
such that
\begin{align}
\mathbf y
=
a\mathbf1.
\label{eq:spectral-energy-y-constant}
\end{align}
Taking the anchor component in
\eqref{eq:spectral-energy-y-constant}
and using
\eqref{eq:spectral-energy-y-anchored}
yields
\begin{align}
0
&\stackrel{(a)}{=}
y_{i_0}
\stackrel{(b)}{=}
a,
\end{align}
where (a) is
\eqref{eq:spectral-energy-y-anchored}
and (b) takes the
$i_0$th component of
\eqref{eq:spectral-energy-y-constant}.
Thus
$a=0$, and
\eqref{eq:spectral-energy-y-constant}
gives
$\mathbf y=\mathbf0$, contradicting
\eqref{eq:spectral-energy-y-is-nonzero}.
Therefore
\begin{align}
\mathbf z^{\top}\mathbf K_0\mathbf z
>
0
\qquad
\text{for every }
\mathbf z\in\mathbb R^G\setminus\{\mathbf0\}.
\label{eq:spectral-energy-K0-positive}
\end{align}
Together with
\eqref{eq:spectral-energy-K0-symmetric},
this proves that
$\mathbf K_0$
is symmetric positive definite.

By the spectral theorem for real symmetric matrices, there exist an orthogonal
matrix
$\mathbf Q\in\mathbb R^{G\times G}$
and a real diagonal matrix
$\boldsymbol{\Lambda}$
such that
\begin{align}
\mathbf K_0
=
\mathbf Q\boldsymbol{\Lambda}\mathbf Q^{\top}.
\end{align}
The positive definiteness in
\eqref{eq:spectral-energy-K0-positive}
implies that every diagonal entry
$\lambda_i$
of
$\boldsymbol{\Lambda}$
is strictly positive. Indeed, if
$\mathbf q_i$
denotes the
$i$th column of
$\mathbf Q$, then
$\|\mathbf q_i\|_2=1$
and
$\mathbf K_0\mathbf q_i=\lambda_i\mathbf q_i$, so
\begin{align}
\lambda_i
&\stackrel{(a)}{=}
\lambda_i
\mathbf q_i^{\top}\mathbf q_i
\stackrel{(b)}{=}
\mathbf q_i^{\top}\mathbf K_0\mathbf q_i
\stackrel{(c)}{>}
0.
\label{eq:spectral-energy-eigenvalues-positive}
\end{align}
Here (a) uses
$\mathbf q_i^{\top}\mathbf q_i=1$,
(b) uses
$\mathbf K_0\mathbf q_i=\lambda_i\mathbf q_i$, and
(c) applies
\eqref{eq:spectral-energy-K0-positive}
to the nonzero vector
$\mathbf q_i$.

We now derive the Brownian energy identity. By
\eqref{eq:stiffness-energy-matrix-form},
\eqref{eq:spectral-energy-full-vector},
and the definition of
$\mathbf K_0$,
\begin{align}
\left\|f\right\|_{B,h}^{2}
&\stackrel{(a)}{=}
\mathbf v^{\top}\mathbf K\mathbf v
\stackrel{(b)}{=}
\left(
\mathbf R\widetilde{\mathbf v}
\right)^{\top}
\mathbf K
\left(
\mathbf R\widetilde{\mathbf v}
\right)
\nonumber\\
&\stackrel{(c)}{=}
\widetilde{\mathbf v}^{\top}
\mathbf R^{\top}\mathbf K\mathbf R
\widetilde{\mathbf v}
\stackrel{(d)}{=}
\widetilde{\mathbf v}^{\top}
\mathbf K_0
\widetilde{\mathbf v}
\nonumber\\
&\stackrel{(e)}{=}
\widetilde{\mathbf v}^{\top}
\mathbf Q\boldsymbol{\Lambda}\mathbf Q^{\top}
\widetilde{\mathbf v}
\nonumber\\
&\stackrel{(f)}{=}
\left(
\mathbf Q^{\top}\widetilde{\mathbf v}
\right)^{\top}
\boldsymbol{\Lambda}
\left(
\mathbf Q^{\top}\widetilde{\mathbf v}
\right)
\stackrel{(g)}{=}
\mathbf c^{\top}\boldsymbol{\Lambda}\mathbf c.
\label{eq:spectral-energy-quadratic}
\end{align}
Here (a) applies
\eqref{eq:stiffness-energy-matrix-form}
to
$f=\mathcal R(\mathbf v)$,
(b) substitutes
\eqref{eq:spectral-energy-full-vector},
(c) uses
$(\mathbf R\widetilde{\mathbf v})^{\top}
=\widetilde{\mathbf v}^{\top}\mathbf R^{\top}$
and associativity,
(d) uses
$\mathbf K_0=\mathbf R^{\top}\mathbf K\mathbf R$,
(e) substitutes the orthogonal eigendecomposition of
$\mathbf K_0$,
(f) uses
$\widetilde{\mathbf v}^{\top}\mathbf Q
=(\mathbf Q^{\top}\widetilde{\mathbf v})^{\top}$,
and
(g) substitutes
\eqref{eq:spectral-energy-coordinate-definition}.

Because
$\boldsymbol{\Lambda}$
is diagonal,
\begin{align}
\Lambda_{ij}
=
\lambda_i\delta_{ij},
\qquad
 i,j\in[G].
\label{eq:spectral-energy-Lambda-entrywise}
\end{align}
Therefore
\begin{align}
\mathbf c^{\top}\boldsymbol{\Lambda}\mathbf c
&\stackrel{(a)}{=}
\sum_{i=1}^{G}
\sum_{j=1}^{G}
 c_i\Lambda_{ij}c_j
\stackrel{(b)}{=}
\sum_{i=1}^{G}
\sum_{j=1}^{G}
 c_i\lambda_i\delta_{ij}c_j
\nonumber\\
&\stackrel{(c)}{=}
\sum_{i=1}^{G}
 c_i\lambda_i c_i
\stackrel{(d)}{=}
\sum_{i=1}^{G}
\lambda_i c_i^2.
\label{eq:spectral-energy-diagonal-expansion}
\end{align}
Here (a) expands the quadratic form componentwise,
(b) substitutes
\eqref{eq:spectral-energy-Lambda-entrywise},
(c) uses
$\delta_{ij}=0$
for
$i\neq j$
and
$\delta_{ii}=1$
to collapse the inner sum to its
$j=i$
term, and
(d) rearranges the scalar factors and uses
$c_i c_i=c_i^2$.
Combining
\eqref{eq:spectral-energy-quadratic}
and
\eqref{eq:spectral-energy-diagonal-expansion}
yields
\eqref{eq:spectral-energy-claim},
which proves the lemma.

\end{proof}

\begin{lemma}[Explicit spectral Brownian basis]
\label{lem:explicit-spectral-basis}

Assume that
$G=2m$,
and let
\begin{align}
\mathbf K_0
=
\mathbf Q
\boldsymbol{\Lambda}
\mathbf Q^{\top},
\qquad
\boldsymbol{\Lambda}
=
\operatorname{diag}
\left(
\lambda_1,
\ldots,
\lambda_G
\right)
\label{eq:explicit-spectral-basis-eigendecomposition}
\end{align}
be the orthogonal eigendecomposition obtained in
Theorem~\ref{thm:complete-dct-spectrum}.
Thus
$\mathbf Q\in\mathbb R^{G\times G}$
is orthogonal and
$\lambda_i>0$
for every
$i\in[G]$.
Let
$\mathbf e_i\in\mathbb R^G$
denote the
$i$th canonical basis vector.
For every
$i\in[G]$,
define
\begin{align}
\psi_i
&:=
\mathcal R
\left(
\mathbf R
\mathbf Q
\mathbf e_i
\right)
\nonumber\\
&=
\sum_{j=0}^{G}
\left(
\mathbf R
\mathbf Q
\mathbf e_i
\right)_j
\phi_j
=
\sum_{j=0}^{G}
\left(
\mathbf R
\mathbf Q
\right)_{ji}
\phi_j.
\label{eq:explicit-spectral-basis-definition}
\end{align}

Then the following statements hold.

\begin{enumerate}[(i)]

\item
\label{lem:explicit-spectral-basis-i}
Each function
$\psi_i$
belongs to
$\mathcal H_h^0$,
and the family
$\left\{\psi_1,\ldots,\psi_G\right\}$
forms a basis of
$\mathcal H_h^0$.

\item
\label{lem:explicit-spectral-basis-ii}
Let
$f\in\mathcal H_h^0$,
and let
$\widetilde{\mathbf v}\in\mathbb R^G$
be its reduced anchored nodal coefficient vector.
Then
$f$
admits the unique representation
\begin{align}
f
=
\sum_{i=1}^{G}
c_i
\psi_i,
\label{eq:explicit-spectral-basis-expansion}
\end{align}
where
\begin{align}
\mathbf c
:=
\left(
 c_1,
 \ldots,
 c_G
\right)^{\top}
=
\mathbf Q^{\top}
\widetilde{\mathbf v}
\in
\mathbb R^G
\label{eq:explicit-spectral-basis-coefficients}
\end{align}
is the spectral coefficient vector of
$f$.

\item
\label{lem:explicit-spectral-basis-iii}
The basis is orthogonal with respect to the Brownian inner product. More
precisely, for every
$i,j\in[G]$,
\begin{align}
\left\langle
\psi_i,
\psi_j
\right\rangle_{B,h}
=
\lambda_i
\delta_{ij}.
\label{eq:explicit-spectral-basis-orthogonality}
\end{align}

\end{enumerate}

\end{lemma}

\begin{proof}

We first verify that every function defined in
\eqref{eq:explicit-spectral-basis-definition}
belongs to the anchored space. Fix
$i\in[G]$.
Because
$\psi_i$
is a finite linear combination of
$\phi_0,\ldots,\phi_G$,
we have
$\psi_i\in\mathcal H_h$.
Moreover, the anchor is the grid point
$t_{i_0}=0$,
and the nodal basis satisfies
$\phi_j(0)=\delta_{j i_0}$.
Therefore,
\begin{align}
\psi_i(0)
&\stackrel{(a)}{=}
\sum_{j=0}^{G}
\left(
\mathbf R\mathbf Q
\right)_{ji}
\phi_j(0)
\stackrel{(b)}{=}
\sum_{j=0}^{G}
\left(
\mathbf R\mathbf Q
\right)_{ji}
\delta_{j i_0}
\nonumber\\
&\stackrel{(c)}{=}
\left(
\mathbf R\mathbf Q
\right)_{i_0 i}
\stackrel{(d)}{=}
\sum_{r=1}^{G}
R_{i_0 r}
Q_{ri}
\stackrel{(e)}{=}
0.
\label{eq:explicit-spectral-basis-anchored}
\end{align}
Here (a) uses the definition of
$\psi_i$
in
\eqref{eq:explicit-spectral-basis-definition};
(b) substitutes
$\phi_j(0)=\delta_{j i_0}$;
(c) uses the defining property of the Kronecker delta to retain only the term
$j=i_0$;
(d) expands the
$(i_0,i)$
entry of the matrix product
$\mathbf R\mathbf Q$;
and (e) uses the construction of the anchoring reconstruction matrix
$\mathbf R$,
whose
$i_0$th row is the zero row because the anchor coefficient is fixed at zero.
Thus
$\psi_i(0)=0$,
and hence
$\psi_i\in\mathcal H_h^0$.
Since
$i\in[G]$
was arbitrary, this conclusion holds for every member of the family.

We next prove existence and uniqueness of the spectral expansion. Let
$f\in\mathcal H_h^0$
be arbitrary. By the reduced anchored nodal representation, there exists a
unique vector
$\widetilde{\mathbf v}\in\mathbb R^G$
such that
\begin{align}
f
=
\mathcal R
\left(
\mathbf R
\widetilde{\mathbf v}
\right).
\label{eq:explicit-spectral-basis-reduced-reconstruction}
\end{align}
Define
$\mathbf c$
by
\eqref{eq:explicit-spectral-basis-coefficients}.
Because
$\mathbf Q$
is orthogonal,
\begin{align}
\mathbf Q^{\top}\mathbf Q
=
\mathbf I_G,
\qquad
\mathbf Q\mathbf Q^{\top}
=
\mathbf I_G.
\label{eq:explicit-spectral-basis-Q-orthogonal}
\end{align}
Consequently,
\begin{align}
\mathbf Q\mathbf c
&\stackrel{(a)}{=}
\mathbf Q
\mathbf Q^{\top}
\widetilde{\mathbf v}
\stackrel{(b)}{=}
\mathbf I_G
\widetilde{\mathbf v}
\stackrel{(c)}{=}
\widetilde{\mathbf v}.
\label{eq:explicit-spectral-basis-inverse-coordinate-map}
\end{align}
Here (a) substitutes the definition
$\mathbf c=\mathbf Q^{\top}\widetilde{\mathbf v}$;
(b) uses
$\mathbf Q\mathbf Q^{\top}=\mathbf I_G$
from
\eqref{eq:explicit-spectral-basis-Q-orthogonal};
and (c) uses the defining action of the identity matrix.

Writing
$\mathbf c=\sum_{i=1}^{G}c_i\mathbf e_i$
and using
\eqref{eq:explicit-spectral-basis-reduced-reconstruction}
and
\eqref{eq:explicit-spectral-basis-inverse-coordinate-map},
we obtain
\begin{align}
f
&\stackrel{(a)}{=}
\mathcal R
\left(
\mathbf R
\widetilde{\mathbf v}
\right)
\stackrel{(b)}{=}
\mathcal R
\left(
\mathbf R
\mathbf Q
\mathbf c
\right)
\nonumber\\
&\stackrel{(c)}{=}
\mathcal R
\left(
\mathbf R
\mathbf Q
\sum_{i=1}^{G}
c_i
\mathbf e_i
\right)
\stackrel{(d)}{=}
\mathcal R
\left(
\sum_{i=1}^{G}
c_i
\mathbf R
\mathbf Q
\mathbf e_i
\right)
\nonumber\\
&\stackrel{(e)}{=}
\sum_{i=1}^{G}
c_i
\mathcal R
\left(
\mathbf R
\mathbf Q
\mathbf e_i
\right)
\stackrel{(f)}{=}
\sum_{i=1}^{G}
c_i
\psi_i.
\label{eq:explicit-spectral-basis-existence-chain}
\end{align}
Here (a) is
\eqref{eq:explicit-spectral-basis-reduced-reconstruction};
(b) substitutes
$\widetilde{\mathbf v}=\mathbf Q\mathbf c$
from
\eqref{eq:explicit-spectral-basis-inverse-coordinate-map};
(c) expands
$\mathbf c$
in the canonical basis of
$\mathbb R^G$;
(d) uses the linearity of the matrix maps
$\mathbf Q$
and
$\mathbf R$;
(e) uses the linearity of the finite-element reconstruction operator
$\mathcal R$;
and (f) applies the definition of
$\psi_i$
in
\eqref{eq:explicit-spectral-basis-definition}.
This proves the existence of the representation
\eqref{eq:explicit-spectral-basis-expansion}.

To prove uniqueness, suppose that another vector
\begin{align}
\mathbf d
:=
\left(
 d_1,
 \ldots,
 d_G
\right)^{\top}
\in
\mathbb R^G
\end{align}
satisfies
\begin{align}
f
=
\sum_{i=1}^{G}
d_i
\psi_i.
\label{eq:explicit-spectral-basis-second-expansion}
\end{align}
Repeating the linearity calculation in
\eqref{eq:explicit-spectral-basis-existence-chain}
with
$\mathbf d$
in place of
$\mathbf c$
gives
\begin{align}
\sum_{i=1}^{G}
d_i
\psi_i
=
\mathcal R
\left(
\mathbf R
\mathbf Q
\mathbf d
\right).
\label{eq:explicit-spectral-basis-second-reconstruction}
\end{align}
Combining
\eqref{eq:explicit-spectral-basis-reduced-reconstruction},
\eqref{eq:explicit-spectral-basis-second-expansion},
and
\eqref{eq:explicit-spectral-basis-second-reconstruction}
yields
\begin{align}
\mathcal R
\left(
\mathbf R
\widetilde{\mathbf v}
\right)
=
\mathcal R
\left(
\mathbf R
\mathbf Q
\mathbf d
\right).
\label{eq:explicit-spectral-basis-equal-reconstructions}
\end{align}
The nodal representation in
$\mathcal H_h$
is unique. Therefore,
\begin{align}
\mathbf R
\widetilde{\mathbf v}
=
\mathbf R
\mathbf Q
\mathbf d.
\label{eq:explicit-spectral-basis-equal-full-vectors}
\end{align}
By
\eqref{eq:spectral-energy-R-isometry},
$\mathbf R^{\top}\mathbf R=\mathbf I_G$.
Left-multiplying
\eqref{eq:explicit-spectral-basis-equal-full-vectors}
by
$\mathbf R^{\top}$
gives
\begin{align}
\widetilde{\mathbf v}
&\stackrel{(a)}{=}
\mathbf I_G
\widetilde{\mathbf v}
\stackrel{(b)}{=}
\mathbf R^{\top}
\mathbf R
\widetilde{\mathbf v}
\stackrel{(c)}{=}
\mathbf R^{\top}
\mathbf R
\mathbf Q
\mathbf d
\nonumber\\
&\stackrel{(d)}{=}
\mathbf I_G
\mathbf Q
\mathbf d
\stackrel{(e)}{=}
\mathbf Q
\mathbf d.
\label{eq:explicit-spectral-basis-unique-reduced-vector}
\end{align}
Here (a) uses the action of the identity matrix;
(b) substitutes
$\mathbf R^{\top}\mathbf R=\mathbf I_G$;
(c) left-multiplies
\eqref{eq:explicit-spectral-basis-equal-full-vectors}
by
$\mathbf R^{\top}$;
(d) again uses
$\mathbf R^{\top}\mathbf R=\mathbf I_G$;
and (e) uses
$\mathbf I_G\mathbf Q\mathbf d=\mathbf Q\mathbf d$.
It follows that
\begin{align}
\mathbf d
&\stackrel{(a)}{=}
\mathbf I_G
\mathbf d
\stackrel{(b)}{=}
\mathbf Q^{\top}
\mathbf Q
\mathbf d
\stackrel{(c)}{=}
\mathbf Q^{\top}
\widetilde{\mathbf v}
\stackrel{(d)}{=}
\mathbf c.
\label{eq:explicit-spectral-basis-unique-coefficients}
\end{align}
Here (a) uses the action of the identity matrix;
(b) uses
$\mathbf Q^{\top}\mathbf Q=\mathbf I_G$;
(c) substitutes
\eqref{eq:explicit-spectral-basis-unique-reduced-vector};
and (d) uses the definition of
$\mathbf c$
in
\eqref{eq:explicit-spectral-basis-coefficients}.
Thus the coefficients in
\eqref{eq:explicit-spectral-basis-expansion}
are unique, proving
(\ref{lem:explicit-spectral-basis-ii}).

The existence of the expansion for every
$f\in\mathcal H_h^0$
shows that
$\{\psi_1,\ldots,\psi_G\}$
spans
$\mathcal H_h^0$.
To verify linear independence, suppose
$\sum_{i=1}^{G}d_i\psi_i=0$.
The zero function has reduced anchored nodal coefficient vector
$\widetilde{\mathbf v}=\mathbf0$
and hence, by
\eqref{eq:explicit-spectral-basis-coefficients},
spectral coefficient vector
$\mathbf Q^{\top}\mathbf0=\mathbf0$.
The uniqueness established in
\eqref{eq:explicit-spectral-basis-unique-coefficients}
therefore gives
$\mathbf d=\mathbf0$.
Thus the family is linearly independent. Since it both spans
$\mathcal H_h^0$
and is linearly independent, it is a basis, proving
(\ref{lem:explicit-spectral-basis-i}).

It remains to prove the Brownian orthogonality relation. For each
$i\in[G]$,
set
\begin{align}
\mathbf a_i
:=
\mathbf R
\mathbf Q
\mathbf e_i
=
\left(
 a_{0i},
 \ldots,
 a_{Gi}
\right)^{\top}
\in
\mathbb R^{G+1}.
\label{eq:explicit-spectral-basis-ai-definition}
\end{align}
Then
$\psi_i=\mathcal R(\mathbf a_i)$.
Fix arbitrary
$i,j\in[G]$.
Using the definition of the Brownian inner product and the entrywise definition
of the stiffness matrix in
Proposition~\ref{prop:finite-brownian-energy},
we obtain
\begin{align}
\left\langle
\psi_i,
\psi_j
\right\rangle_{B,h}
&\stackrel{(a)}{=}
\int_{-A}^{A}
\psi_i'(t)
\psi_j'(t)
\,\mathrm dt
\stackrel{(b)}{=}
\int_{-A}^{A}
\left(
\sum_{r=0}^{G}
 a_{ri}
 \phi_r'(t)
\right)
\left(
\sum_{s=0}^{G}
 a_{sj}
 \phi_s'(t)
\right)
\,\mathrm dt
\nonumber\\
&\stackrel{(c)}{=}
\sum_{r=0}^{G}
\sum_{s=0}^{G}
 a_{ri}
 a_{sj}
\int_{-A}^{A}
\phi_r'(t)
\phi_s'(t)
\,\mathrm dt
\stackrel{(d)}{=}
\sum_{r=0}^{G}
\sum_{s=0}^{G}
 a_{ri}
 K_{rs}
 a_{sj}
\stackrel{(e)}{=}
\mathbf a_i^{\top}
\mathbf K
\mathbf a_j
\nonumber\\
&\stackrel{(f)}{=}
\left(
\mathbf R\mathbf Q\mathbf e_i
\right)^{\top}
\mathbf K
\left(
\mathbf R\mathbf Q\mathbf e_j
\right)
\stackrel{(g)}{=}
\mathbf e_i^{\top}
\mathbf Q^{\top}
\mathbf R^{\top}
\mathbf K
\mathbf R
\mathbf Q
\mathbf e_j
\stackrel{(h)}{=}
\mathbf e_i^{\top}
\mathbf Q^{\top}
\mathbf K_0
\mathbf Q
\mathbf e_j
\nonumber\\
&\stackrel{(i)}{=}
\mathbf e_i^{\top}
\mathbf Q^{\top}
\left(
\mathbf Q
\boldsymbol{\Lambda}
\mathbf Q^{\top}
\right)
\mathbf Q
\mathbf e_j
\stackrel{(j)}{=}
\mathbf e_i^{\top}
\left(
\mathbf Q^{\top}\mathbf Q
\right)
\boldsymbol{\Lambda}
\left(
\mathbf Q^{\top}\mathbf Q
\right)
\mathbf e_j
\stackrel{(k)}{=}
\mathbf e_i^{\top}
\boldsymbol{\Lambda}
\mathbf e_j
\nonumber\\
&\stackrel{(l)}{=}
\Lambda_{ij}
\stackrel{(m)}{=}
\lambda_i
\delta_{ij}.
\label{eq:explicit-spectral-basis-inner-product-chain}
\end{align}
Here (a) is the definition of the Brownian inner product;
(b) uses
$\psi_i=\sum_{r=0}^{G}a_{ri}\phi_r$
and
$\psi_j=\sum_{s=0}^{G}a_{sj}\phi_s$,
and differentiates these finite sums on each open mesh interval;
(c) expands the product and uses the linearity of the integral for finite sums;
(d) uses
$K_{rs}=\int_{-A}^{A}\phi_r'(t)\phi_s'(t)\,\mathrm dt$;
(e) is the componentwise formula for the bilinear matrix product
$\mathbf a_i^{\top}\mathbf K\mathbf a_j$;
(f) substitutes
\eqref{eq:explicit-spectral-basis-ai-definition};
(g) uses
$(\mathbf R\mathbf Q\mathbf e_i)^{\top}
=\mathbf e_i^{\top}\mathbf Q^{\top}\mathbf R^{\top}$
and associativity;
(h) substitutes
$\mathbf K_0=\mathbf R^{\top}\mathbf K\mathbf R$;
(i) substitutes the eigendecomposition
\eqref{eq:explicit-spectral-basis-eigendecomposition};
(j) uses associativity of matrix multiplication;
(k) uses
$\mathbf Q^{\top}\mathbf Q=\mathbf I_G$
and the action of the identity matrix;
(l) uses the defining property of canonical basis vectors,
$\mathbf e_i^{\top}\boldsymbol{\Lambda}\mathbf e_j=\Lambda_{ij}$;
and (m) uses the diagonal structure
$\Lambda_{ij}=\lambda_i\delta_{ij}$.
This proves
\eqref{eq:explicit-spectral-basis-orthogonality}
and hence
(\ref{lem:explicit-spectral-basis-iii}).
In particular, if
$i\neq j$,
then
$\langle\psi_i,\psi_j\rangle_{B,h}=0$,
whereas if
$i=j$,
then
\begin{align}
\left\|\psi_i\right\|_{B,h}^{2}
=
\left\langle\psi_i,\psi_i\right\rangle_{B,h}
=
\lambda_i
>
0,
\end{align}
because
$\lambda_i>0$.
This completes the proof.

\end{proof}

\section{Implications of the Equivalent Coordinate Representations}
\label{app:equivalence-implications}

\paragraph{Function-space invariance.}
Theorem~\ref{thm:equivalence-of-implementations} establishes that the nodal, increment, and spectral parameterizations are mathematically equivalent realizations of the same finite Brownian RKHS. Consequently, the represented function space, the Brownian norm, and the associated approximation properties are intrinsic to the underlying RKHS and remain invariant under changes of coordinate realization.

\paragraph{Controlled optimization.}
Because all three realizations represent exactly the same hypothesis class, differences cannot be
attributed to representational capacity, approximation error, or intrinsic Brownian regularization. The
sharper distinction is optimizer-specific: \Cref{prop:gd-equivariance} forces mapped nodal and spectral
GD/SGD trajectories to agree, \Cref{prop:increment-natural-gradient} identifies increment GD with
Brownian preconditioning, and \Cref{prop:adam-nonequivariance} permits standard Adam to distinguish the
DCT-VIII basis. The direct one-layer and recursive tests in \Cref{sec:experiments} numerically verify these
predictions under mapped initialization and shared update protocols.

\paragraph{Connection to the spectral realization.}
The equivalence theorem is independent of the particular choice of coordinates. The explicit spectral realization presented in Section~\ref{sec:explicit-spectrum} constructs a block-structured orthogonal basis of the anchored finite Brownian RKHS through the complete eigendecomposition of the anchored stiffness matrix, providing the spectral implementation used in the optimization study.

\section{Implications of the Explicit Spectral Representation}
\label{app:spectral-implications}
\paragraph{Explicit block-structured spectral coordinates.}
Unlike the nodal and increment parameterizations, the spectral coordinates diagonalize the Brownian energy. The energy determines the eigenspaces, but the twofold multiplicity of every eigenvalue does not determine a unique eigenbasis inside each left--right eigenspace. The displayed block-diagonal eigendecomposition fixes an explicit orthogonal coordinate system by respecting the reduced-coordinate ordering and the two half-grid blocks.

\paragraph{Diagonalization of the Brownian energy.}
The explicit spectral basis diagonalizes the Brownian inner product, reducing the Brownian norm to a weighted Euclidean norm of the spectral coefficients. Consequently, each spectral mode contributes independently to the total Brownian energy, providing a natural decomposition of the finite Brownian RKHS into orthogonal Brownian spectral modes.

\paragraph{Implications for optimization.}
The spectral transformation is orthogonal, so mapped nodal--spectral GD/SGD must agree after coordinates,
initialization, scalar steps, and minibatches are mapped consistently. Standard Adam is different: its
universal orthogonal equivariance group is only the signed permutations, and the DCT-VIII matrix is not
one. The controlled tests in \Cref{sec:experiments} numerically verify both the exact GD/SGD identity and
the permitted Adam separation.

\section{Supplementary Experimental Evidence}
\label{app:experiments}

This appendix separates direct theorem-verification protocols from secondary predictive evidence. The old
synthetic tables and legacy coordinate curves are omitted because their raw outputs mixed incompatible
sample sizes or lacked the mapped-initialization and optimizer metadata required for a valid equivariance
test. The retained EuroSAT and spatially separated Salinas studies were regenerated from frozen pipelines and
independently audited from saved predictions and probabilities.

\subsection{Theory-Validation Protocols and Reproducibility Audit}
\label{app:theory-validation-reproducibility}

\paragraph{Evidence classes.}
The direct tests E0, E1A--E1C, and E2 numerically verify algebraic identities, mapped optimizer trajectories, conditioning statements, and the theorem-specific standard-Adam counterexample. They are not tuned predictive comparisons. The EuroSAT and spatially separated Salinas studies instead compare different model classes and are retained only as secondary predictive and robustness evidence; they are not used to establish coordinate equivariance.

\begin{table}[H]
\centering
\caption{Frozen verification and audit protocols. Identity-test arms share one mapped represented initialization, objective normalization, scalar step sequence, and, for SGD, minibatch order; they are never tuned separately.}
\label{tab:theory-validation-protocols}
\scriptsize
\setlength{\tabcolsep}{3pt}
\resizebox{\textwidth}{!}{%
\begin{tabular}{llllll}
\toprule
Stage & Purpose & Data/runs & Mesh & Budget & Precision \\
\midrule
E0 & Finite Brownian RKHS, spectrum, and coordinate identities & 125 random vectors & $G\in\{8,16,32,64,128\}$ & 4097-point dense grid & float64 CPU \\
E1A & One-layer mapped-trajectory identities for GD and SGD & 256 samples, 5 seeds & $G\in\{8,16,32,64,128\}$ & 40 GD; 80 SGD; batch 32 & float64 CPU \\
E1B & Conditioning and fixed-step convergence under mesh refinement & 256 samples, 5 seeds & $G\in\{8,16,32,64,128\}$ & gap $10^{-8}$; analytic steps & float64 CPU \\
E1C & Basis dependence of standard Adam & 256 samples, 5 seeds & $G\in\{8,16,32,64,128\}$ & 200 steps; full batch & float64 CPU \\
E2 & Recursive blockwise mapped-trajectory identities & 128 samples, 5 seeds & $(8,16,32)$ and $(32,64,128)$ & 8 GD; 16 SGD; batch 32 & float64 CPU \\
Real-data audit & Independent metric and statistical reconstruction & 390 prediction artifacts & EuroSAT and Salinas & 5720 published-value checks & NumPy reconstruction \\
\bottomrule
\end{tabular}%
}
\end{table}

\paragraph{E0: finite Brownian RKHS identities.}
We use $A=2$, $G\in\{8,16,32,64,128\}$, seeds $\{0,1,2,3,4\}$, and five random vectors per seed. All matrix, spectrum, Brownian-Gram, reconstruction, and energy diagnostics are evaluated in float64 on CPU; reconstruction is checked on $4097$ points.

\paragraph{E1A: one-layer mapped-trajectory identities for GD and SGD.}
The Brownian-regularized least-squares objective uses $n=256$, $\rho=0.1$, and noise standard deviation $0.03$. Every coordinate arm begins from one reduced-nodal function and is mapped exactly. Full-batch GD uses $40$ updates; SGD uses $80$ updates with batch size $32$ and an identical minibatch sequence. No coordinate arm is tuned separately.

\paragraph{E1B: conditioning and convergence under mesh refinement.}
The pure Brownian Hessians are constructed directly in nodal, spectral, and increment coordinates. For least squares, $n=256$ and $\rho\in\{0.05,0.1,0.2,0.5\}$. The convergence diagnostic uses $\rho=0.1$, the analytic optimal scalar step for each SPD quadratic, a relative objective-gap tolerance of $10^{-8}$, and at most $200{,}000$ iterations. These steps are analytic, not validation-selected.

\paragraph{E1C: basis dependence of standard Adam.}
We use scalar $\beta_1=0.9$, $\beta_2=0.999$, $\varepsilon=10^{-8}$, zero initial moment states, bias correction, learning rate $0.003$ for the regularized objective, and $0.01$ for the deterministic linear counterexample. The experiment uses $200$ full-batch updates, no AdamW, weight decay, clipping, or AMSGrad, and no coordinate-specific tuning.

\paragraph{E2: recursive blockwise mapped-trajectory identities.}
The recursive model contains three trainable Brownian profile blocks with grid sets $(8,16,32)$ and $(32,64,128)$, trainable cross-layer couplings, and identically initialized non-profile parameters. The objective uses $n=128$ and $\rho=0.05$. We run eight full-batch steps at $\eta=5\times10^{-5}$ and sixteen shared-minibatch steps at $\eta=2\times10^{-5}$ with batch size $32$. Each increment block is compared with the nodal Brownian-gradient update scaled by its own $1/h_b$.
\paragraph{Independent real-data audit.}
For the retained EuroSAT and spatially separated Salinas studies, accuracy, multiclass F$_1$, calibration, and likelihood metrics were reconstructed directly from 390 saved prediction artifacts, including 5,820 per-class records and 88,920 confusion-matrix cells. A separate implementation reproduced 5,720 reported AUC, exact-test, crossed-bootstrap, and variance-component entries across 11 tables with zero mismatch and maximum discrepancy $8.66\times10^{-15}$.

\subsection{EuroSAT Under Test-Time Spectral-Band Dropout}
\label{app:eurosat-robustness}

\paragraph{Protocol.}
Patch MLP and Spatial Path-Atomic BKL receive the same Spatial65 representation of all $13$ EuroSAT
bands. For each of five paired seeds, both models are trained once on the same clean split. At test time
only, one to six bands are zeroed before feature extraction; masks are nested within each seed and shared
between models. For a metric $M$, retention is $M(d)/M(0)$, and the normalized area under the retention
curve summarizes the complete degradation path. We additionally record ECE, NLL, per-class F$_1$,
inference time, and parameter count.

\begin{table}[H]
\centering
\small
\caption{EuroSAT robustness under nested test-time spectral-band dropout. Values are means over five paired seeds. Retention AUC is normalized over zero to six zeroed bands; higher is better. Maximum ECE is lower-is-better.}
\label{tab:eurosat-spectral-dropout-robustness}
\resizebox{\linewidth}{!}{%
\begin{tabular}{lccccc}
\toprule
Model & \shortstack{Accuracy\\Retention AUC (\%)} & \shortstack{Macro-F$_1$\\Retention AUC (\%)} & \shortstack{Worst Accuracy\\Retention (\%)} & \shortstack{Worst Macro-F$_1$\\Retention (\%)} & \shortstack{Maximum\\ECE} \\
\midrule
Patch MLP & $37.78$ & $32.16$ & $21.80$ & $13.88$ & $0.774$ \\
Spatial Path-Atomic BKL & $\mathbf{46.81}$ & $\mathbf{42.10}$ & $\mathbf{23.38}$ & $\mathbf{17.68}$ & $\mathbf{0.629}$ \\
\bottomrule
\end{tabular}%
}
\end{table}

\begin{figure}[H]
\centering
\includegraphics[width=0.94\linewidth]{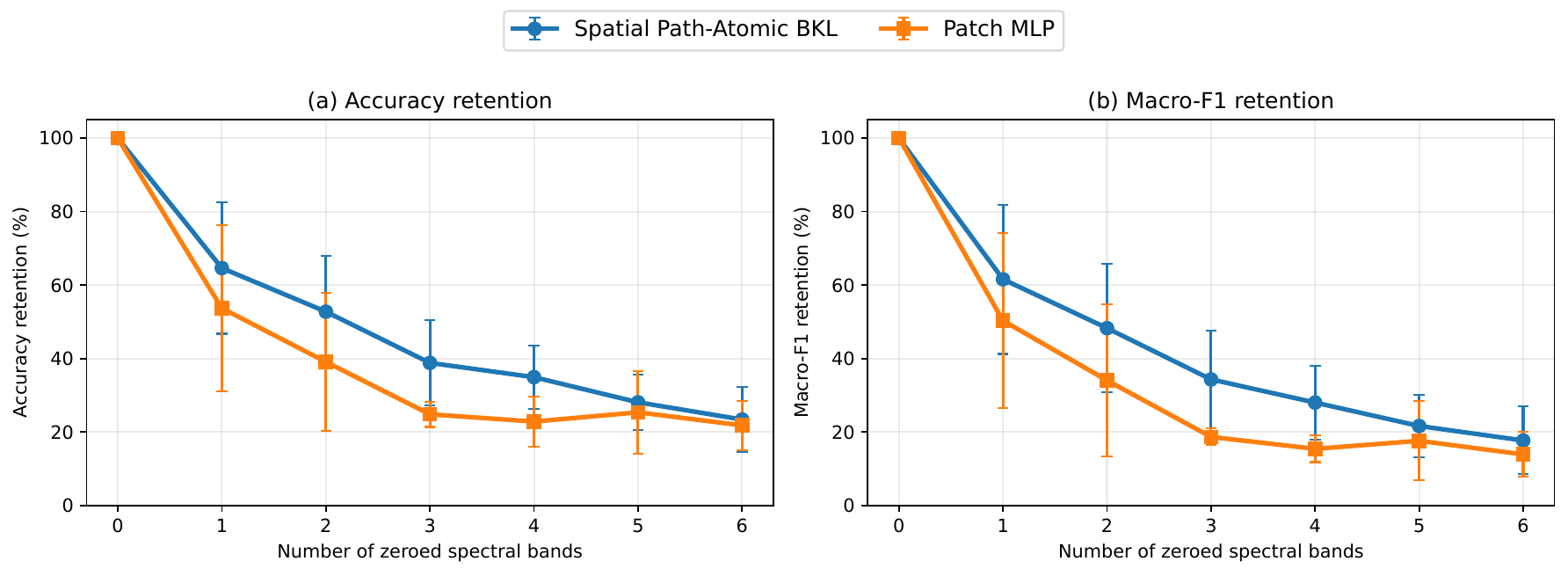}
\caption{EuroSAT accuracy- and macro-F$_1$-retention trajectories under nested test-time spectral-band dropout. Lines show paired-seed means and shading one standard deviation.}
\label{fig:eurosat-spectral-dropout-retention}
\end{figure}

\begin{figure}[H]
\centering
\includegraphics[width=0.64\linewidth]{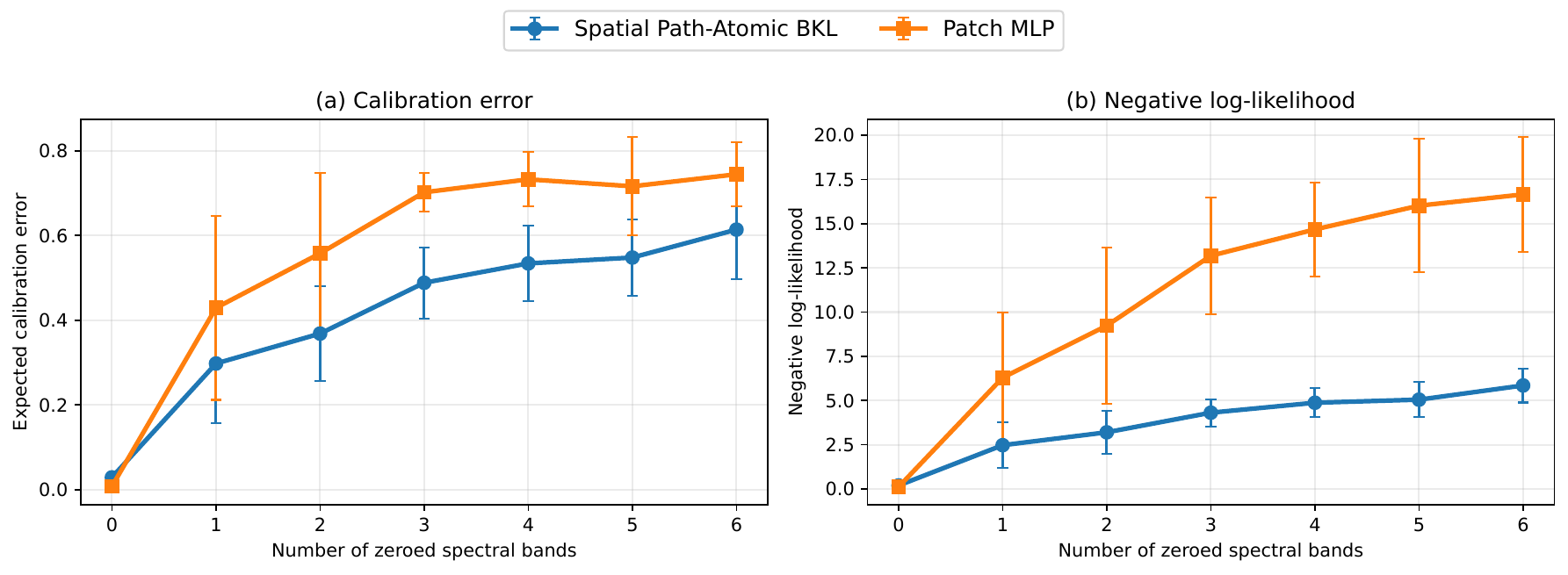}
\caption{EuroSAT calibration and likelihood degradation under nested spectral-band dropout; lower is better.}
\label{fig:eurosat-spectral-dropout-calibration}
\end{figure}

\paragraph{Results and scope.}
Spatial Path-Atomic BKL has normalized accuracy-retention AUC $46.81\%$ versus $37.78\%$ for Patch MLP
and macro-F$_1$-retention AUC $42.10\%$ versus $32.16\%$. All five pairs favor BKL for both AUCs and for
maximum-ECE reduction; each exact two-sided sign-flip test gives $p=0.0625$, the minimum nonzero value for
five pairs. The result concerns the full degradation trajectory rather than uniform superiority at each
severity. This image-level multispectral experiment is a BKL-versus-MLP comparison and is not used to
establish coordinate equivariance.

\begin{table}[H]
\centering
\small
\caption{Exact paired sign-flip analysis over five seeds. Improvements are oriented so that positive values favor Spatial Path-Atomic BKL.}
\label{tab:eurosat-robustness-tests}
\resizebox{\linewidth}{!}{%
\begin{tabular}{lcccc}
\toprule
Metric & Mean improvement & Two-sided $p$ & Paired $d_z$ & BKL-favored seeds \\
\midrule
Accuracy retention AUC & $0.0903$ & $0.0625$ & $1.61$ & $5/5$ \\
Macro-F$_1$ retention AUC & $0.0994$ & $0.0625$ & $2.11$ & $5/5$ \\
Worst accuracy retention & $0.0158$ & $0.6875$ & $0.17$ & $3/5$ \\
Worst Macro-F$_1$ retention & $0.0379$ & $0.4375$ & $0.46$ & $3/5$ \\
Maximum ECE reduction & $0.1449$ & $0.0625$ & $1.53$ & $5/5$ \\
\bottomrule
\end{tabular}%
}
\end{table}

\begingroup
\setlength{\intextsep}{5pt plus 1pt minus 1pt}
\setlength{\textfloatsep}{5pt plus 1pt minus 1pt}
\setlength{\abovecaptionskip}{3pt}
\setlength{\belowcaptionskip}{1pt}
\subsection{Salinas Evaluation Under Spatial Separation and Missing Spectral Bands}
\label{app:salinas-corrected-robustness}

\paragraph{Protocol.}
We evaluate the mathematically constrained finite VBKL implementation on the 204-band Salinas scene using a spatially separated protocol fixed before the final model evaluations. Labeled pixels are grouped into $24\times24$ spatial blocks, and a four-pixel cross-split buffer removes 19,152 boundary pixels. The retained split contains 23,998 training, 5,424 validation, and 5,555 test pixels, with all 16 classes represented in every split. Standardization is fitted on the training split only. The standardized spectra are divided by the training 0.99-quantile of their Euclidean norms and are radially projected, when necessary, to the radius-two ball. This projection affected 2 training spectra and no validation or test spectra (0 and 0, respectively).

The constrained VBKL uses 32 paths, depth four under the convention of one normalized linear projection followed by three Brownian profiles, grid size 32, and interval radius $A=2$. Every effective direction has unit Euclidean norm; every profile is anchored at zero and radially normalized to Brownian energy at most one. Hyperparameters were chosen using validation data only: seed 0 supplied the initial search, seeds 1--3 supplied independent validation confirmation, and the final clean models use seeds 4, 5, 6, 7, 8. The fixed Patch MLP has 15,728 parameters, whereas constrained VBKL has 10,128.

\begin{table}[H]
\centering
\caption{Clean Salinas performance (mean $\pm$ standard deviation over five final evaluation seeds).}
\label{tab:salinas-corrected-clean}
\begingroup
\scriptsize
\setlength{\tabcolsep}{3.5pt}
\renewcommand{\arraystretch}{1.08}
\begin{tabular*}{\linewidth}{@{\extracolsep{\fill}}lrrrrrr@{}}
\toprule
Model &
\multicolumn{1}{c}{Acc. (\%)} &
\multicolumn{1}{c}{Macro-F1 (\%)} &
\multicolumn{1}{c}{Bal. Acc. (\%)} &
\multicolumn{1}{c}{Worst F1 (\%)} &
\multicolumn{1}{c}{ECE} &
\multicolumn{1}{c}{NLL} \\
\midrule
Patch MLP & $91.87\pm0.93$ & $94.68\pm1.14$ & $94.83\pm1.28$ & $73.93\pm2.55$ & $0.018\pm0.002$ & $0.245\pm0.013$ \\
Constrained VBKL & $92.18\pm0.20$ & $95.49\pm0.40$ & $95.23\pm0.60$ & $73.83\pm0.63$ & $0.015\pm0.004$ & $0.218\pm0.008$ \\
\bottomrule
\end{tabular*}
\vspace{2pt}
\begin{tabular*}{\linewidth}{@{\extracolsep{\fill}}lrrr@{}}
\toprule
Model &
\multicolumn{1}{c}{Parameters} &
\multicolumn{1}{c}{Training time (s)} &
\multicolumn{1}{c}{Inference (ms/sample)} \\
\midrule
Patch MLP & 15,728 & $5.7\pm0.6$ & $0.001\pm0.000$ \\
Constrained VBKL & 10,128 & $166.2\pm18.9$ & $0.024\pm0.000$ \\
\bottomrule
\end{tabular*}
\endgroup
\end{table}

\Needspace{6\baselineskip}
\paragraph{Clean performance.}
Across the five final seeds, constrained VBKL obtains 92.18\% accuracy and 95.49\% Macro-F1, compared with 91.87\% and 94.68\% for Patch MLP. The corresponding balanced accuracies are 95.23\% and 94.83\%. Constrained VBKL also has lower mean ECE (0.015 versus 0.018) and lower mean NLL. The worst-class F1 values are essentially tied.

\paragraph{Test-time spectral masking.}
Starting from each clean checkpoint, we mask nested random subsets of 16, 31, 47, 63, 78, and 94 bands. A masked raw coordinate is replaced by its training-set band mean, so it becomes exactly zero after training standardization. We cross five final model seeds with five independent mask seeds, giving 25 cells per corrupted level and 300 model evaluations in total. Confidence intervals resample the model-seed and mask-seed axes separately; exact sign-flip analyses of their marginal means are reported as sensitivity analyses.

\begin{figure}[!t]
\centering
\includegraphics[width=0.88\linewidth]{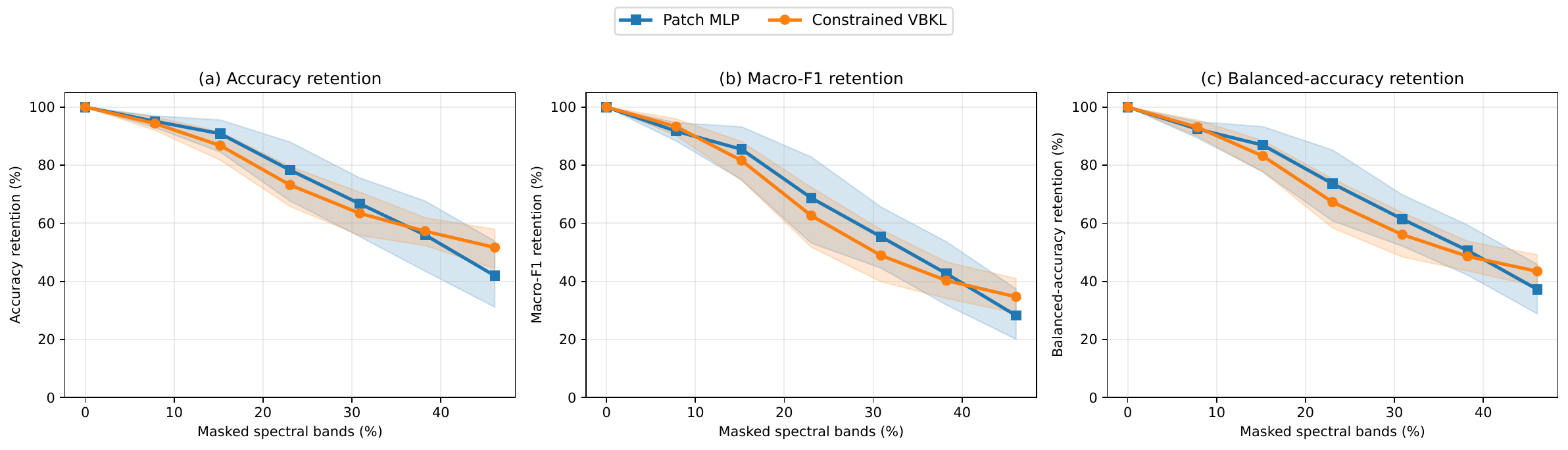}
\caption{Salinas accuracy, Macro-F1, and balanced-accuracy retention under nested spectral-band masking. Lines are crossed-design means and shaded regions are 95\% crossed-bootstrap intervals.}
\label{fig:salinas-corrected-retention}
\end{figure}

\begin{table}[H]
\centering
\caption{Aggregate Salinas robustness under nested spectral-band masking. Entries are means with 95\% crossed-bootstrap intervals. Retention is higher-is-better; ECE and NLL increases are lower-is-better.}
\label{tab:salinas-corrected-robustness}
\begingroup
\scriptsize
\setlength{\tabcolsep}{3.5pt}
\renewcommand{\arraystretch}{1.08}
\begin{tabular*}{\linewidth}{@{\extracolsep{\fill}}lrrr@{}}
\toprule
Model &
\multicolumn{1}{c}{Accuracy-retention AUC (\%)} &
\multicolumn{1}{c}{Macro-F1-retention AUC (\%)} &
\multicolumn{1}{c}{Balanced-retention AUC (\%)} \\
\midrule
Patch MLP & $76.30\,[69.82,\,82.13]$ & $67.98\,[60.57,\,74.78]$ & $72.26\,[66.21,\,77.99]$ \\
Constrained VBKL & $75.12\,[71.53,\,78.39]$ & $65.66\,[60.29,\,70.70]$ & $69.97\,[65.71,\,74.09]$ \\
\bottomrule
\end{tabular*}
\vspace{2pt}
\begin{tabular*}{\linewidth}{@{\extracolsep{\fill}}lrrrr@{}}
\toprule
Model &
\multicolumn{1}{c}{Level-6 Acc. ret. (\%)} &
\multicolumn{1}{c}{Level-6 F1 ret. (\%)} &
\multicolumn{1}{c}{ECE-increase AUC} &
\multicolumn{1}{c}{NLL-increase AUC} \\
\midrule
Patch MLP & $41.92\,[31.22,\,53.95]$ & $28.22\,[19.89,\,37.42]$ & $0.184\,[0.135,\,0.235]$ & $1.568\,[0.866,\,2.367]$ \\
Constrained VBKL & $51.66\,[44.13,\,57.85]$ & $34.67\,[29.02,\,40.85]$ & $0.162\,[0.134,\,0.195]$ & $1.098\,[0.850,\,1.378]$ \\
\bottomrule
\end{tabular*}
\endgroup
\end{table}

\paragraph{Crossover rather than uniform dominance.}
Across the complete corruption path, Patch MLP has slightly higher mean retention AUC: the BKL-minus-MLP differences are -1.17, -2.32, and -2.29 percentage points for accuracy, Macro-F1, and balanced accuracy, respectively. All corresponding crossed-bootstrap intervals include zero. At the most severe level, however, constrained VBKL retains 51.66\% of its clean accuracy and 34.67\% of its clean Macro-F1, compared with 41.92\% and 28.22\% for Patch MLP. The level-six improvements are +9.74 and +6.45 percentage points. All five model-seed marginal means favor constrained VBKL for level-six accuracy retention, yielding an exact two-sided sign-flip value of 0.0625; the crossed-bootstrap interval nevertheless still includes zero because mask realization and model--mask interaction remain substantial.

\paragraph{Calibration and interpretation.}
Constrained VBKL exhibits a smaller average corruption-induced increase in both ECE and NLL. The crossed mean improvements are +0.021 for ECE-increase AUC and +0.469 for NLL-increase AUC. These effects are directionally favorable but their crossed-bootstrap intervals include zero. Thus this experiment supports a qualified conclusion: Patch MLP has slightly higher average retention AUC over the full masking path, whereas constrained VBKL is more competitive on clean data, has higher mean retention at the most severe masking level, and shows smaller likelihood and calibration deterioration. The effect is heterogeneous across both model initialization and mask realization, so the result does not justify a claim of uniform robustness dominance.

\begin{figure}[!htbp]
\centering
\includegraphics[width=0.90\linewidth]{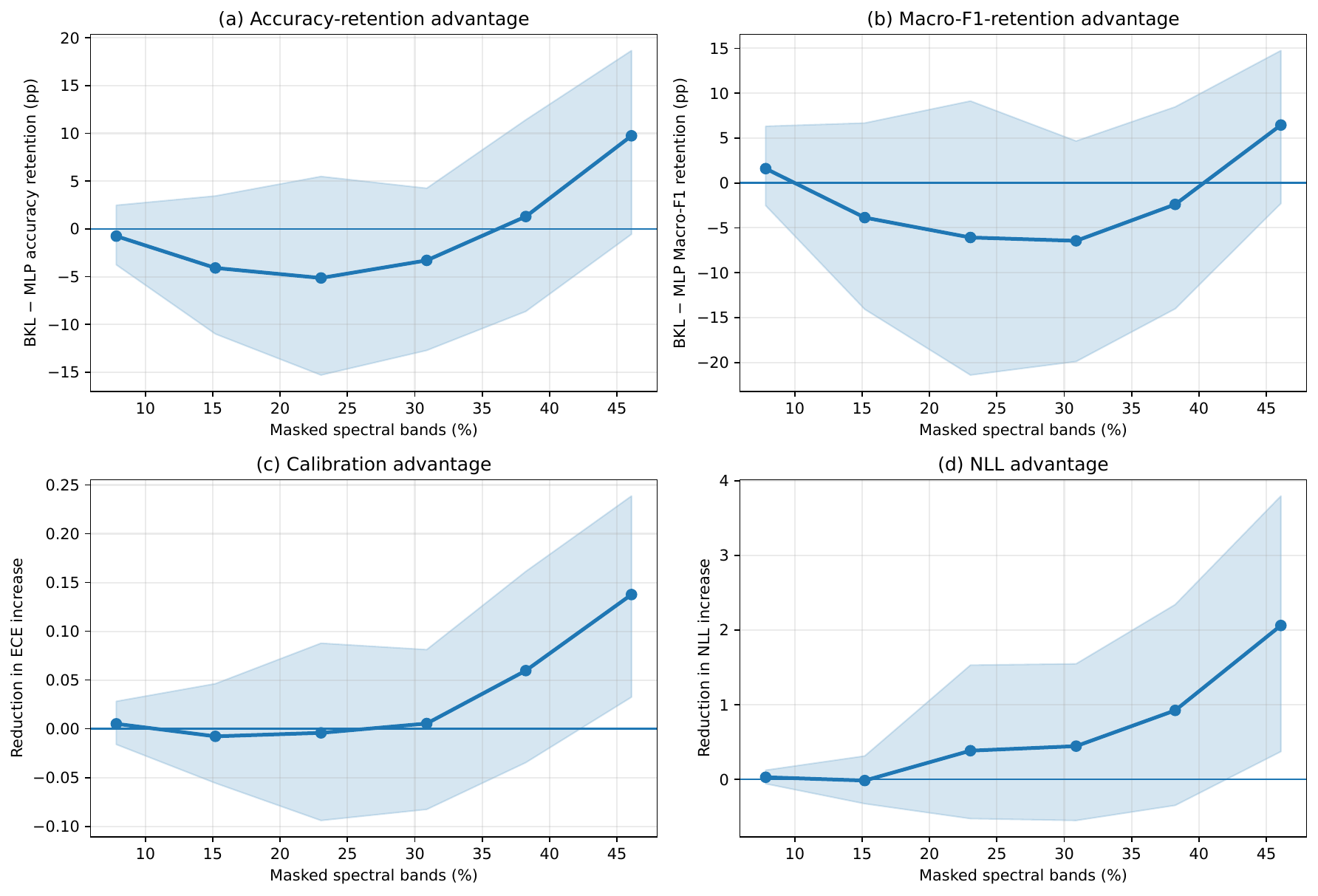}
\vspace{2pt}
\includegraphics[width=0.90\linewidth]{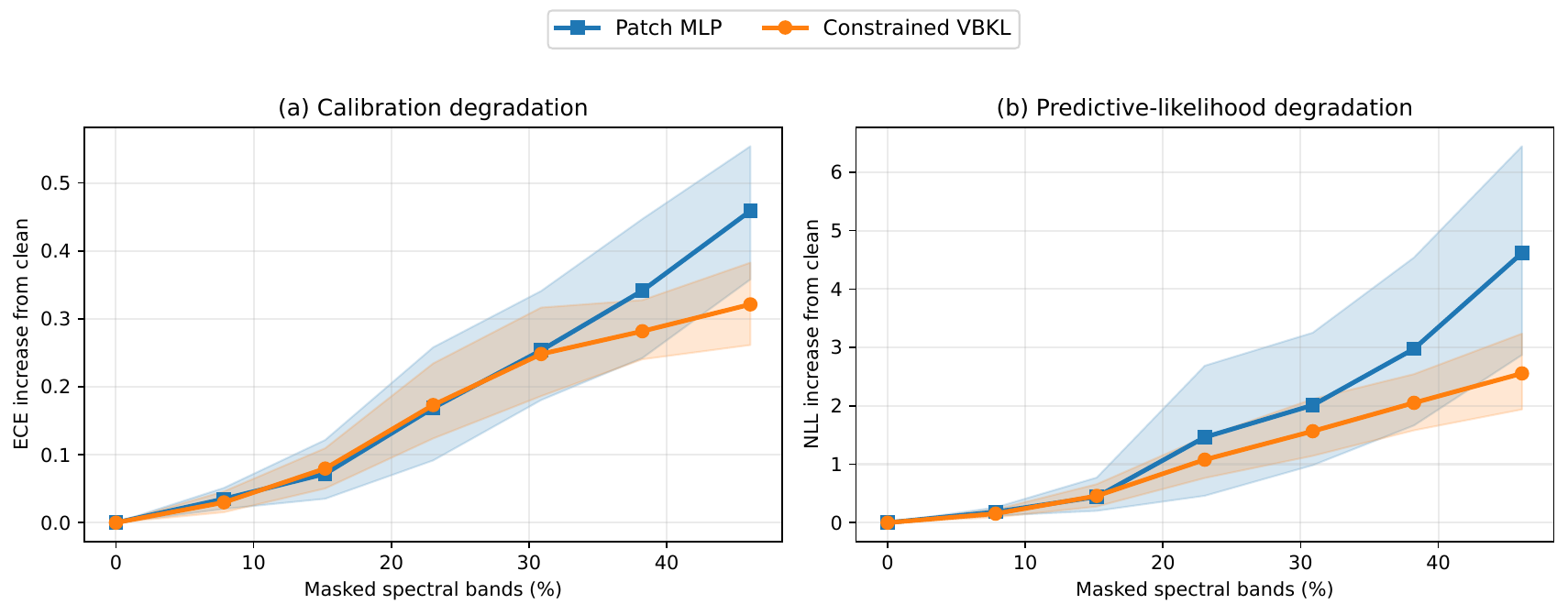}
\caption{Salinas corruption-path comparisons. The upper panels show BKL-minus-MLP differences, with positive values favoring constrained VBKL; for ECE and NLL, positive values denote a smaller corruption-induced increase. The lower panels show the corresponding calibration and negative-log-likelihood degradation.}
\label{fig:salinas-corrected-crossover}
\label{fig:salinas-corrected-calibration}
\end{figure}

\begin{table}[H]
\centering
\caption{Crossed Salinas robustness comparisons. Improvements are oriented so that positive values favor constrained VBKL. The confidence interval is obtained by crossed bootstrap. The model-seed and mask-seed columns report BKL-favored/MLP-favored marginal means and exact two-sided sign-flip $p$-values.}
\label{tab:salinas-corrected-crossed-tests}
\fontsize{6.3}{7.1}\selectfont
\setlength{\tabcolsep}{1.35pt}
\renewcommand{\arraystretch}{1.02}
\begin{tabular*}{\linewidth}{@{\extracolsep{\fill}}llcccccc@{}}
\toprule
Metric &
Favored model &
\shortstack{Mean improvement\\[.2ex]\relax[95\% CI]} &
\shortstack{Cells\\B/M} &
\shortstack{Model seeds\\B/M} &
$p_{\rm model}$ &
\shortstack{Mask seeds\\B/M} &
$p_{\rm mask}$ \\
\midrule
Accuracy-retention AUC & Patch MLP & $-1.17\,[-6.86,\,+4.33]$ & 14/11 & 2/3 & 0.6875 & 3/2 & 0.6250 \\
Macro-F1-retention AUC & Patch MLP & $-2.32\,[-10.28,\,+5.39]$ & 12/13 & 2/3 & 0.6875 & 2/3 & 0.4375 \\
Balanced-accuracy-retention AUC & Patch MLP & $-2.29\,[-8.88,\,+4.27]$ & 12/13 & 2/3 & 0.6875 & 1/4 & 0.3125 \\
Level-6 accuracy retention & Constrained VBKL & $+9.74\,[-0.54,\,+18.58]$ & 19/6 & 5/0 & 0.0625 & 3/2 & 0.2500 \\
Level-6 Macro-F1 retention & Constrained VBKL & $+6.45\,[-2.64,\,+14.72]$ & 18/7 & 4/1 & 0.1875 & 4/1 & 0.1875 \\
ECE-increase AUC & Constrained VBKL & $+0.021\,[-0.032,\,+0.074]$ & 16/9 & 4/1 & 0.4375 & 3/2 & 0.3125 \\
NLL-increase AUC & Constrained VBKL & $+0.469\,[-0.207,\,+1.194]$ & 15/10 & 4/1 & 0.1250 & 3/2 & 0.3125 \\
\bottomrule
\end{tabular*}
\end{table}

\begin{table}[H]
\centering
\caption{Severity-by-severity Salinas retention. Means are taken over the full $5\times5$ crossed design at corrupted levels; the clean level uses five model seeds. Differences are BKL minus MLP in percentage points.}
\label{tab:salinas-corrected-severity}
\scriptsize
\setlength{\tabcolsep}{3.0pt}
\renewcommand{\arraystretch}{1.06}
\begin{tabular*}{\linewidth}{@{\extracolsep{\fill}}rrrrrrrr@{}}
\toprule
Level &
Bands &
MLP Acc. ret. &
BKL Acc. ret. &
$\Delta$ Acc. &
MLP F1 ret. &
BKL F1 ret. &
$\Delta$ F1 \\
\midrule
0 & 0 & 100.00 & 100.00 & +0.00 & 100.00 & 100.00 & +0.00 \\
1 & 16 & 95.09 & 94.34 & -0.75 & 91.68 & 93.27 & +1.60 \\
2 & 31 & 90.83 & 86.75 & -4.09 & 85.46 & 81.60 & -3.86 \\
3 & 47 & 78.30 & 73.17 & -5.13 & 68.70 & 62.62 & -6.08 \\
4 & 63 & 66.72 & 63.43 & -3.29 & 55.36 & 48.91 & -6.45 \\
5 & 78 & 55.97 & 57.27 & +1.29 & 42.67 & 40.28 & -2.39 \\
6 & 94 & 41.92 & 51.66 & +9.74 & 28.22 & 34.67 & +6.45 \\
\bottomrule
\end{tabular*}
\end{table}

\endgroup

\end{document}